\documentclass{article}
\usepackage{iclr2026_conference,times} 

\usepackage[utf8]{inputenc}  
\usepackage[T1]{fontenc}
\usepackage{microtype}

\usepackage{amsmath,amssymb,amsfonts,mathtools}
\usepackage{nicefrac}
\usepackage{amsmath,amsfonts,bm}

\def\eqref#1{equation~\ref{#1}}
\def\1{\bm{1}}

\DeclareMathAlphabet{\mathsfit}{\encodingdefault}{\sfdefault}{m}{sl}
\SetMathAlphabet{\mathsfit}{bold}{\encodingdefault}{\sfdefault}{bx}{n}

\DeclareMathOperator*{\argmin}{arg\,min}

\input{mysymbol.sty}
\input{mypackage.sty}
\input{myacronym.sty}

\usepackage{booktabs}
\usepackage{makecell}
\usepackage{wrapfig}
\usepackage{placeins}

\usepackage{algorithm}
\usepackage{algorithmic}

\usepackage{xcolor}
\usepackage[most]{tcolorbox}
\usepackage{pifont}

\usepackage{tikz}
\usetikzlibrary{matrix}
\usepackage{pgfplots}
\pgfplotsset{compat=1.18}
\usepackage{standalone}

\usepackage{etoc}
\usepackage{etoolbox}

\makeatletter
\patchcmd{\@setfontsize}
  {\baselineskip\baselinestretch\baselineskip}%
  {\baselineskip 1.0\baselineskip}{}{}
\g@addto@macro\normalsize{%
  \setlength{\abovedisplayskip}{4pt plus 1pt minus 1pt}%
  \setlength{\belowdisplayskip}{4pt plus 1pt minus 1pt}%
  \setlength{\abovedisplayshortskip}{2pt plus 1pt minus 1pt}%
  \setlength{\belowdisplayshortskip}{2pt plus 1pt minus 1pt}%
}
\makeatother

\title{In-memory Training on Analog Devices with Limited \\  Conductance States via  Multi-tile Residual Learning}

\author{
Jindan Li\textsuperscript{1}, 
Zhaoxian Wu\textsuperscript{1}, 
Gaowen Liu\textsuperscript{3}, 
Tayfun Gokmen\textsuperscript{4},
Tianyi Chen\textsuperscript{1, 2}\\[10pt]
\textsuperscript{1} Cornell University, New York, NY 10044 \\
\textsuperscript{2} Rensselaer Polytechnic Institute, Troy, NY 12180 \\
\textsuperscript{3} Cisco Research, Naperville, IL 60540 \\
\textsuperscript{4} IBM T. J. Watson Research Center, Yorktown Heights, NY 10598 \\
\texttt{\{jl4767, zw868, tianyi.chen\}@cornell.edu}
}

\iclrfinalcopy 
\begin{document}

\maketitle

\begin{abstract}
Analog in-memory computing (AIMC) accelerators enable efficient deep neural network computation directly within memory using resistive crossbar arrays, where model parameters are represented by the conductance states of memristive devices. 
However, effective in-memory training typically requires at least 8-bit conductance states to match digital baselines. Realizing such fine-grained states is costly and often requires complex noise mitigation techniques that increase circuit complexity and energy consumption. In practice, many promising memristive devices such as ReRAM offer only about 4-bit resolution due to fabrication constraints, and this limited update precision substantially degrades training accuracy.
To enable on-chip training with these limited-state devices, this paper proposes a \emph{residual learning} framework that sequentially learns on multiple crossbar tiles to compensate the residual errors from low-precision weight updates. 
Our theoretical analysis shows that the optimality gap shrinks with the number of tiles and achieves a linear convergence rate. Experiments on standard image classification benchmarks demonstrate that our method consistently outperforms state-of-the-art in-memory analog training strategies under limited-state settings, while incurring only moderate hardware overhead as confirmed by our cost analysis.
\end{abstract}

\section{Introduction}
With the growing adoption of AI across various fields, the demand for \emph{accurate and energy-efficient} training hardware is increasing. 
In this context, \emph{analog in-memory computing} (AIMC) is an emerging solution that performs matrix vector multiplication (MVM) operations directly on weights stored in memory, offering significant efficiency improvements over conventional von Neumann systems. 
In AIMC hardware, the parameters (matrices) of deep neural networks (DNN) are represented by the conductance states of  \emph{memristive devices} in analog crossbar arrays, while the inputs (vectors) are programmed as voltage signals. Using Kirchhoff's and Ohm’s laws, MVM operations between a $D\times D$ matrix and a vector can be completed in $\mathcal{O}(1)$ time on AIMC hardware, while it requires at least $\mathcal{O}(D)$ time in digital MVM \citep{hu2016dot}.
A full MVM on analog hardware can be executed with energy in the range of tens of femtojoules ($10^{-15}$ joules), whereas accessing a 1 kB SRAM block in digital systems typically costs 1 picojoule ($10^{-12}$ joules) per byte \citep{murmann2020mixed}. 
This advantage translates into higher energy efficiency.
A typical commercial digital accelerator has plateaued around 10 tera-operations per second per watt (TOPS/W) \citep{reuther2022ai}, which can be significantly surpassed by AIMC accelerators.
For example, a monolithic 3D AIMC chip achieves more than 210 TOPS/W~\citep{chen202264}, and a 4$\times$4 core array reaches 30 TOPS/W~\citep{jia2021scalable}.
However, due to the inherent difficulty in precisely and reliably changing the conductance of the memory elements, in-memory analog  training presents significant challenges.

This paper focuses on gradient-based in-memory training on AIMC hardware. The objective of training is to solve the standard model optimization problem, formally defined as:
\begin{equation}
    \label{problem}
    W^* := \arg\min_{W \in \mathbb{R}^{D \times D}} f(W)
\end{equation}
\begin{wrapfigure}{r}{0.4\linewidth}
    \centering
         \vspace{-.5em}
\includegraphics[width=0.9\linewidth]{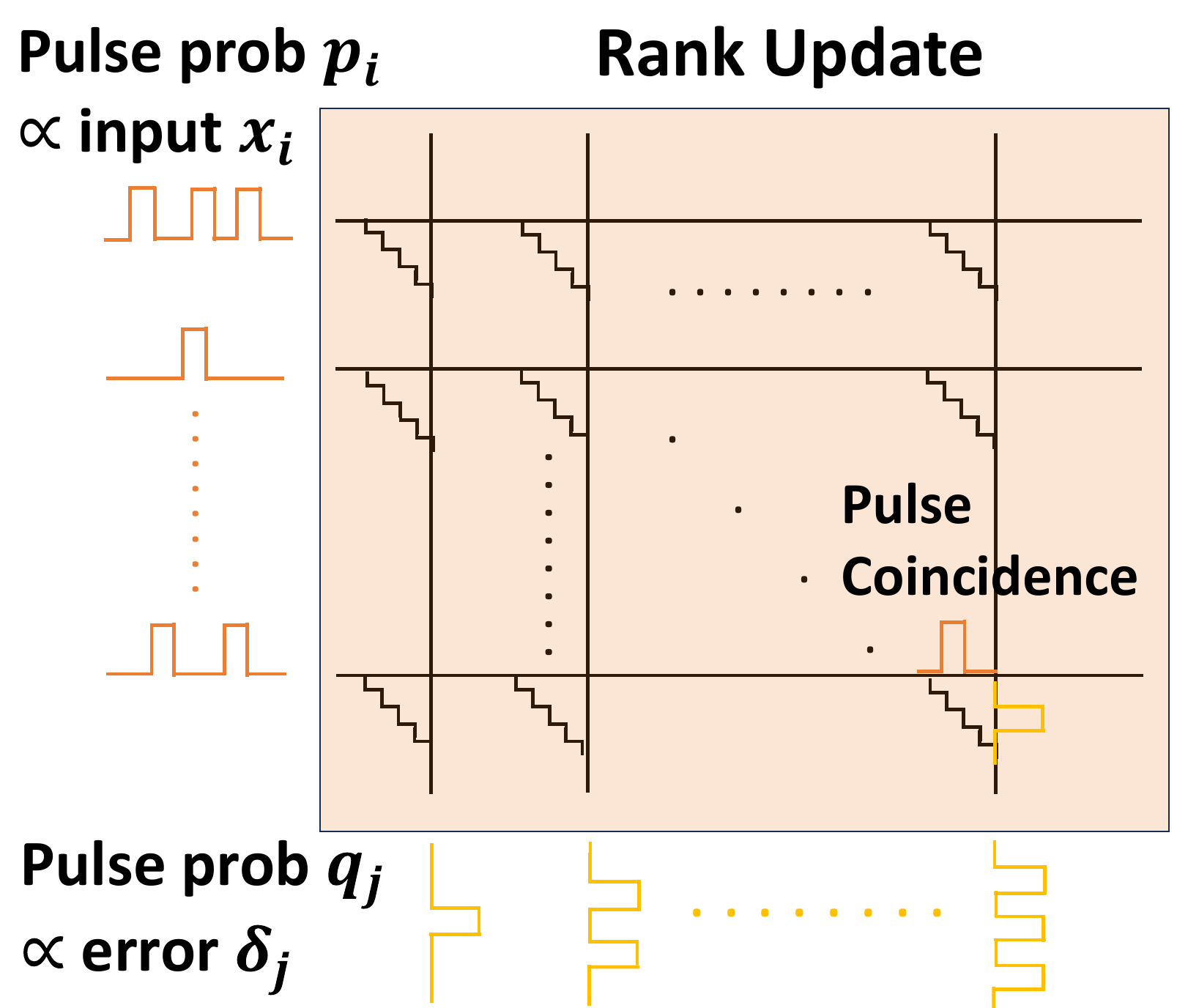} 
% \vspace{-2em}
\caption{Illustration for rank update via stochastic pulse streams.
}    \label{fig:demo_of_pulse_update}
     \vspace{-1em}
\end{wrapfigure}
where \( f(\cdot) : \mathbb{R}^{D \times D} \to \mathbb{R} \) is the objective and \( W \) is a trainable matrix stored in analog crossbar arrays. In digital accelerators,   \eqref{problem} can be solved by stochastic gradient descent (SGD), whose recursion is given by $W_{t+1} = W_t -\alpha\nabla f(W_t; \xi_t)$. Here, $\alpha$ is the learning rate and \( \xi_t \) denotes a sample randomly drawn in iteration \( t \). To implement SGD on AIMC hardware, one needs to update the weights stored in the crossbar array using the \emph{rank update} method \citep{gokmen2016acceleration}. This approach leverages two $\mathcal{O}(D)$-dimensional vectors, the backpropagation error $\delta$ and the input $x$, to perform in-memory updates directly on the analog array via stochastic pulse streams, as illustrated in Figure \ref{fig:demo_of_pulse_update}. 
 Ideally, each pulse adjusts a weight element $w \in W$ by the minimal increment $\pm \Delta w_{\min}$, with the sign determined by the pulse polarity.
The resulting weight evolution is illustrated in Figure~\ref{fig:weight_increment}.

 \begin{wrapfigure}{r}{0.45\linewidth}
    \centering
    \vspace{-1em}  
\includegraphics[width=0.95\linewidth]{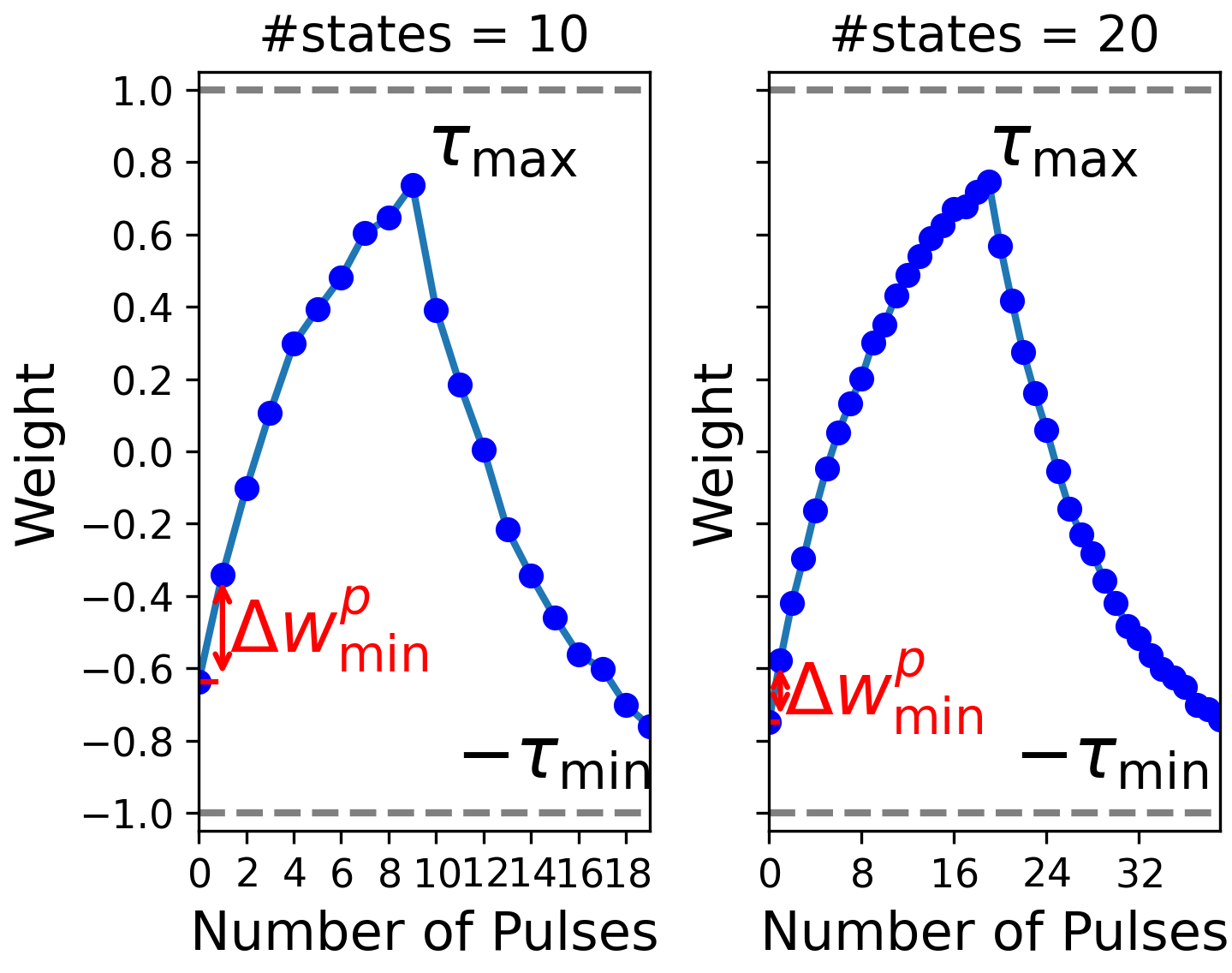}
    % \vspace{-1em}
    \caption{Illustration of pulsed weight updates on 10 and 20 states softbound devices. Due to asymmetric update, the actual weight increment follows $\Delta w^p_{\min} = \Delta w_{\min} \cdot q_+(w)$, where $q_+(\cdot)$ represents the device positive response factor.
    % where $q_+(g)$ models the nonlinearity in the programming direction.
    }
\label{fig:weight_increment}
    \vspace{-1em} 
\end{wrapfigure}

 We define the number of device states by dividing the total weight range \( w \in [\tau_{\min}, \tau_{\max}] \) by this minimal weight change:
 $n_{\text{states}}: = (\tau_{\max} - \tau_{\min})/\Delta w_{\min}$, where  \(n_{\text{states}}\) determines how many distinct values the weight can stably represent.
 A smaller  \(n_{\text{states}}\) (larger  \(\Delta w_{\min}\)) amplifying quantization noise \(\zeta\) whose variance scales with \(\Delta w_{\min}\). This noise captures the gap between ideal and actual updates and fundamentally limits training accuracy.
Previous studies have shown that successful training on crossbar-based architectures typically requires at least 8-bit distinct conductance levels to achieve competitive accuracy~\citep{li2018capacitor, chen2017neurosim+}. However, some devices struggle to provide this level of granularity within a single memory cell. MRAM devices are typically limited to two stable states per cell, whereas ReRAM is usually constrained to 4-bit per cell in practice (detailed survey in Table~\ref{table:device_granularity} and Appendix \ref{sec:Memeristor_Devices}), which makes it difficult to achieve the multi-bit precision required for effective training.  As illustrated in Figure~\ref{fig:demo_of_training_with_states}, reducing the number of states to 20 or fewer results in a convergence failure. While ECRAM can support thousands of states, it remains hindered by practical challenges, including complex three-terminal design, CMOS incompatibility, and material instability~\citep{kim2023three, kwak2025unveiling}, which lack a scalable fabrication pipeline~\citep{kwak2024electrochemical}.
In contrast, ReRAM remains one of the most manufacturable and scalable options~\citep{stecconi2024analog}. In practice, its bi-directional update behavior typically involves limited conductance states together with asymmetric non-idealities \cite{xi2020memory}, which form the primary focus of this paper.
 {\em Rather than} pushing for increasingly precise devices, our work advocates {\em algorithm innovations} to mitigate the limitations of low-state memristive devices, which better align with current fabrication capabilities and offer energy and area efficiency for near-term deployment. Importantly, our goal is not to dismiss high-state devices, but to emphasize the practical and architectural benefits of training with low-state memristive technologies.

\subsection{Main results}
This work addresses the fundamental challenges of limited precision in gradient-based training on AIMC hardware, which stem from the limited number of conductance states and the asymmetric update. We address these challenges by designing {\em composite weight representations} that integrate multiple low-precision tiles to represent high-precision weights, and by developing {\em multi-timescale residual learning} algorithms that enable each tile to dynamically track the residual training error left by preceding tiles. Together, these techniques ensure stable convergence and high training accuracy under low-precision constraints. This motivates our first question: 

\textbf{Q1)} \emph{ How can high-precision weights be represented using limited-conductance states AIMC devices?}

To construct a high precision weight, we define the composite weight as  $\overline{W} = \sum_{n=0}^{N} \gamma^n W^{(n)}$, where $W^{(n)}$ denotes a low precision weight on an AIMC tile $n$, and $\gamma \in (0, 1)$ controls its scaling.
This structure increases the total number of representable values exponentially with the number of tiles, thus significantly enhancing the effective numeric precision. The composite weight $\overline{W}$ is then used in both forward and backward passes; see the details of circuit implementation in Section \ref{sec:Analog_Circuit_Implementation}. 
Given the composite
weight $\overline{W}$, it raises another critical question:

\textbf{Q2)} \emph{How to ensure that the composite weight $\overline{W}$ converges effectively under gradient-based training?}

 \begin{wrapfigure}{r}{0.4\linewidth}
    \centering
        \vspace{-1.5em}
\includegraphics[width=0.9\linewidth]{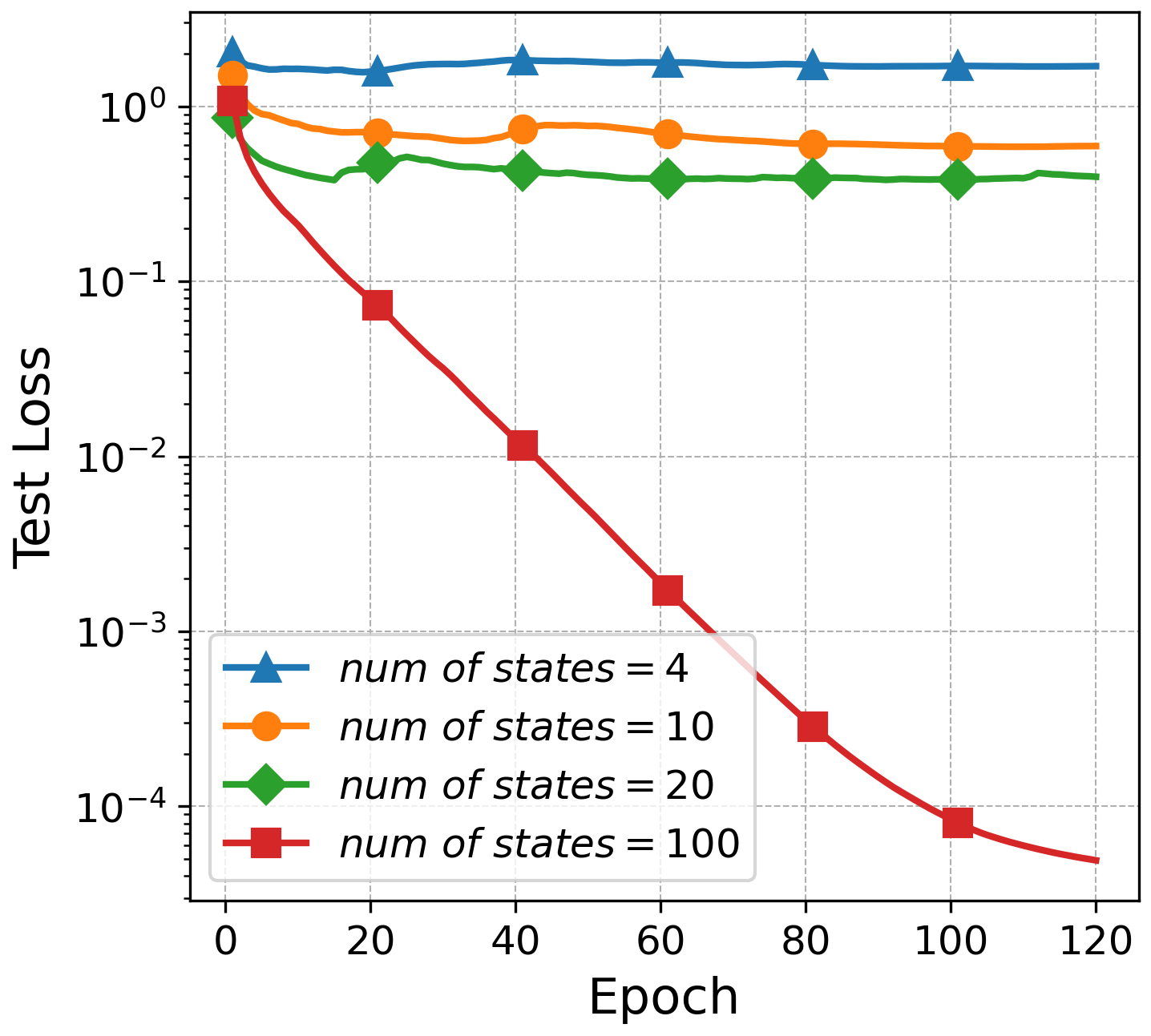}    \vspace{-1em}\caption{Training fails to converge under 4-bit conductance states.
The experiment is conducted on LeNet-5 (MNIST) using Tiki-Taka (TT-v1). 
}    \label{fig:demo_of_training_with_states}   \vspace{-1.5em}
\end{wrapfigure}
To ensure that the composite weight \( \overline{W} \) converges under gradient-based training, we propose a multi-timescale residual learning strategy inspired by the recent advances in \emph{single-timescale stochastic approximation} (STSA) \citep{shen2022single}. However, different from STSA, which tracks a drifting optimum at a single timescale, our method employs a multi-timescale scheme in which each analog tile learns on a progressively slower timescale to approximate the residual left by lower-resolution tiles. This recursive refinement ensures that the composite weight \( \overline{W} \) approaches the global optimum $W^*$ with an exponentially vanishing residual error.
Since our algorithms enable in-memory analog training, the multi-tile strategy incurs minimal hardware overhead in terms of digital storage and latency, thereby achieving a better accuracy–efficiency trade-off than methods that rely on digital processors for high-precision gradient computation such as Mixed-Precision (MP) (see Table~\ref{tab:weight-update-comparison} and Section~\ref{sec:Analog_Circuit_Implementation}).

\textbf{Our contributions.} This work makes the following key contributions:
\begin{enumerate}   
\item[\bf C1)] We propose a high-precision in-memory analog training framework termed multi-timescale residual learning, which overcomes the precision bottleneck of limited conductance states, without requiring reset operations and relying solely on an open-loop transfer process between tiles, thus simplifying hardware implementation.
    \item[\bf C2)] We theoretically analyze the non-convergence of single-tile Analog SGD under realistic device constraints and establish both an upper bound and a matching lower bound. Furthermore, we analyze the convergence of our multi-timescale residual learning and show that the error can be exponentially reduced by increasing the number of tiles.
\item[\bf C3)] We evaluate the proposed algorithm using IBM AIHWKIT~\citep{rasch2021flexible} on CIFAR-100, Fashion-MNIST, and other datasets, demonstrating consistent improvements over existing in-memory analog training methods under limited conductance states.  
 We analyze hardware costs (storage, energy, latency and area) on real datasets, showing that our method achieves an accuracy–efficiency trade-off compared to baseline methods.  
\end{enumerate}

\subsection{Related works}
\label{sec:related works}

\textbf{Gradient-based training on AIMC hardware.} 
Gradient-based AIMC training was first explored by \emph{rank-update} methods such as Analog SGD~\citep{gokmen2016acceleration}. 
TT-v1 mitigates asymmetric updates and noise by accumulating gradients on an auxiliary array and periodically transferring them~\citep{gokmen2020algorithm}, while TT-v2 adds digital filtering for improved robustness~\citep{gokmen2021enabling}. 
However, these methods often fail to converge on larger models under limited conductance states. 
Hybrid 3T1C–PCM designs~\citep{ambrogio2018equivalent,cristiano2018perspective} improve precision through closed-loop tuning (CLT) but incur high latency and area overhead (detailed in Appendix \ref{sec:Memeristor_Devices}). 
Another line of \emph{hybrid training paradigms} programs high-precision digital gradients directly into analog weights (MP~\citep{le2018mixed}), with momentum-based extensions~\citep{wang2020ssm,huang2020overcoming}, though these incur substantial digital storage and compute cost, as compared in Figure \ref{fig:weight-update-comparison}.

\textbf{Multi-sequence Stochastic approximation.}
Stochastic approximation with multiple coupled sequences \cite{yang2019multilevel, shen2022single, huang2025single} has found broad applications in machine learning such as bilevel learning \cite{lu2023bilevel, jiang2024primal} and reinforcement learning \cite{zeng2024fast, zeng2024accelerated}. Our analog training on multiple tiles can be naturally viewed as a system of
\begin{wrapfigure}{r}{0.55\linewidth}
    \centering
 \vspace{-0.5em} % adjust top spacing relative to text
    \includegraphics[width=\linewidth]{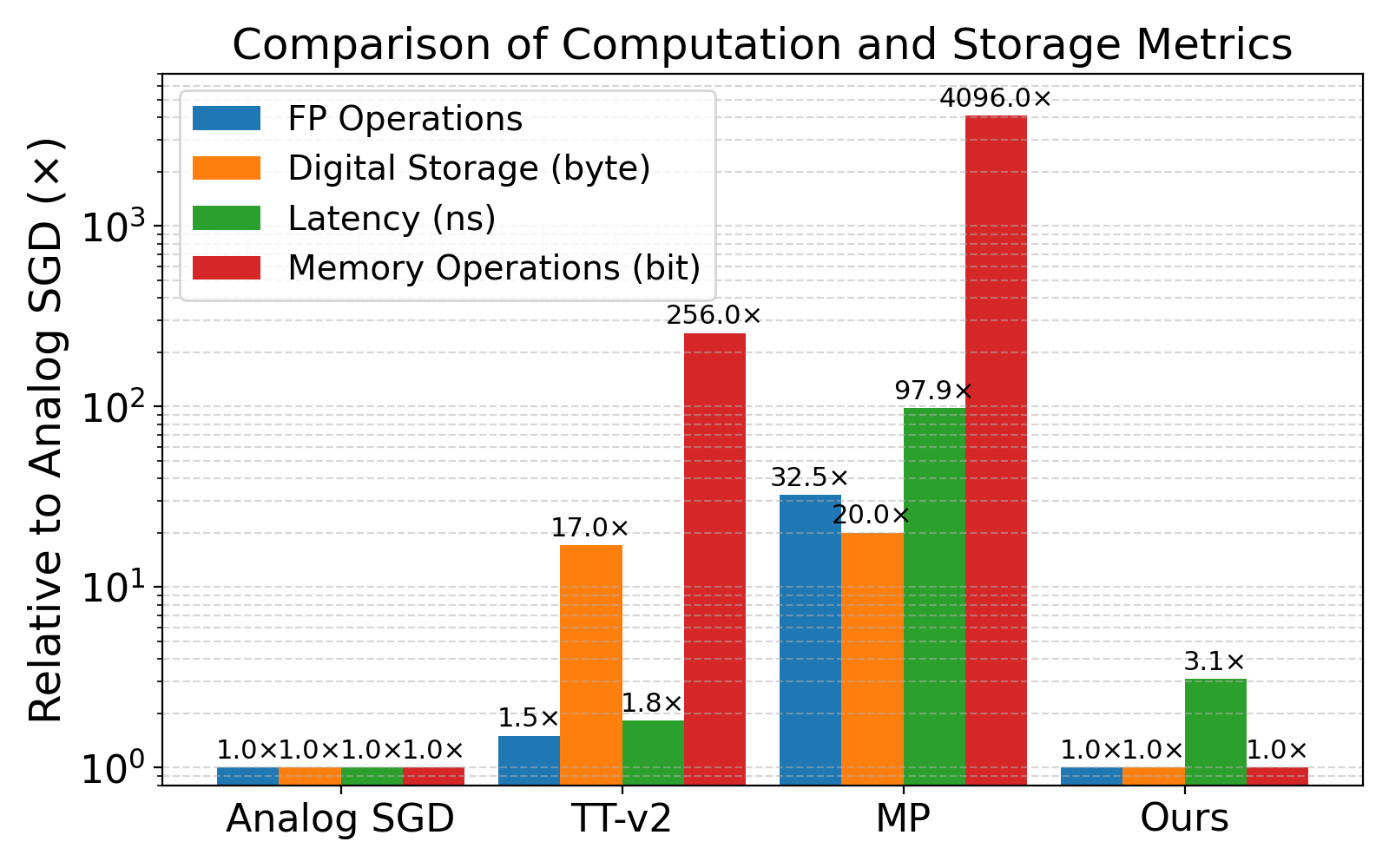}
    \vspace{-2em}
    \caption{
Algorithm comparison of computation and storage for per-sample update with model dimension $D=32$ and mini-batch size $B=4$ from the statistics in Table~\ref{tab:weight-update-comparison}. MP incurs substantially higher overhead, which grows more severe as $D$ and $B$ increase. 
    }
    \vspace{-2em}
    \label{fig:weight-update-comparison}
\end{wrapfigure}
coupled sequences, since the gradient is computed over all tiles jointly, and in turn, each tile’s update is driven by this gradient to ensure the component weight converges toward the global optimum.

\textbf{Low-precision computing.}  
Existing works have primarily shown how low-precision devices can be combined to achieve high-precision computation in \emph{static settings} such as scientific computing and DNN inference (e.g., bit-slicing schemes \citep{feinberg2018enabling, song2023refloat, song2024programming,le2022precision, pedretti2021conductance, boybat2018neuromorphic,mackin2022optimised}). However, extending this concept to high-precision training on multiple low-precision tiles is far more challenging, as it must maintain convergence in the presence of asymmetric updates while weights are continually changing under limited-precision gradients.
Alternative precision enhancement strategies fall under hardware-aware training, which incorporates quantization and other hardware noise into the digital training process to improve inference accuracy when weights are deployed on non-ideal analog devices~\citep{klachko2019improving,he2019noise,buchel2025analog}, in contrast to our in-memory analog training that updates weights directly on analog hardware.

\textbf{Contribution relative to prior works.}
Our work extends \cite{wu2024towards,wu2025analog} by introducing a new setting that accounts for limited conductance state non-idealities and by addressing it through a new algorithm called multi-timescale residual learning. Both prior studies focus on asymmetric non-idealities: they model the analog update dynamics and establish convergence analyzes for Analog SGD and the Tiki-Taka algorithm. Specifically, \cite{wu2024towards} considers only an asymmetric linear device, while \cite{wu2025analog} extends the analysis to general asymmetric devices to demonstrate the scalability of the approach.
Building on their foundation, we incorporate another widespread device non-ideality, limited conductance states, by modeling a quantization noise term that introduces a new non-vanishing error component in Analog SGD. To mitigate this error, we generalize the two-tile residual learning scheme in \cite{wu2025analog} to a multi-tile regime, where each tile approximates the residual left by lower-resolution tiles, leading to an exponentially reduced error floor as the number of tiles increases. The key challenge is that this residual keeps drifting as  lower-resolution tiles are updated, which we resolve by a multi-timescale learning strategy that freezes the  lower-resolution tiles while the current tile runs a sufficiently long inner loop to track its quasi-stationary residual. See detailed discussion in Appendix \ref{section:relation-with-wu2024}.

\section{Training Dynamics on Non-Ideal AIMC Hardware}
\label{sec:Training_Dynamics}
In this section, we analyze how two critical non-idealities in AIMC hardware, \emph{limited conductance states} and \emph{asymmetric updates}, affect the training dynamics. We derive an analog update rule that explicitly captures these effects, and show that under this rule, Analog SGD exhibits an asymptotic error determined by both the gradient noise and the quantization noise.
% the minimal weight change.
\label{subsection:Pulse update}

\textbf{Asymmetric pulse update.} 
Rank-update-based training updates the weights on the crossbar array by simultaneously applying two stochastic pulse streams to its rows and columns, with weight increments occurring at pulse coincidences. Ideally, each pulse coincidence induces a minimal weight change $\Delta w_{\min}$. However, practical updates depend nonlinearly on both the current weight and the pulse polarity, causing deviations from this idealized increment and resulting in asymmetric updates. 
Specifically, given the weight 
\( W_t \in \mathbb{R}^{D\times D}\) at iteration \( t \), the asymmetric pulse update for an element  $w_t$ is modeled as:
$w_{t+1} =
\begin{cases}
w_{t} + \Delta w_{\min} \cdot q_+(w_{t}), & \textit{for a positive pulse}, \\
w_{t} + \Delta w_{\min} \cdot q_-(w_{t}), & \textit{for a negative pulse},
\end{cases}$ 
 where \( q_+(w) \) and \( q_-(w) \) denote the device response factors to positive and negative pulses, respectively. 
Following the decomposition introduced in~\citep{gokmen2020algorithm}, we define the symmetric and asymmetric components as
$F(w) := \frac{q_{-}(w) + q_{+}(w)}{2}$ and $G(w) := \frac{q_{-}(w) - q_{+}(w)}{2}$, yielding a compact element-wise update form triggered by each pulse coincidence: $w_{t+1} = w_t + \Delta w_{\min} \odot F(w_t) -|\Delta w_{\min}| \odot G(w_t)$.

\textbf{Quantization noise from limited conductance states.}  
 During the rank update process (see Figure~\ref{fig:demo_of_pulse_update}), each weight element $w_{ij}$, located at column $i$ and row $j$ of the crossbar array, is updated by $\alpha x_i \delta_j$, where $x_i$ is the $i$-th entry of the input vector $x$, $\delta_j$ is the $j$-th entry of the backpropagated error vector $\delta$, and $\alpha$ is the learning rate.
We implement the update using stochastic pulse streams, where the amplitude of each pulse is generated from a Bernoulli distribution with parameters $p_i \propto x_i$ and $q_j \propto \delta_j$. This scheme guarantees that the expectation of the actual weight change $\Delta w_{ij}$ is equal to the ideal update $\alpha x_i \delta_j$.
However, due to the limited number of conductance states, each pulse induces only a discrete weight increment of magnitude $\Delta w_{\min}$. This discretization introduces a mismatch between the actual update and its ideal target. We capture this discrepancy by defining a stochastic noise term $\zeta_{ij}$, such that
$
\Delta w_{ij} = \alpha x_i \delta_j + \zeta_{ij},
$
and show its statistical properties in the following.
\begin{restatable}[Statistical properties of pulse update noise]{lemma}{LemmaTTConvergenceScvx}
    \label{lemma:pulse-update-error}
    Under the stochastic pulse update in \citep{gokmen2016acceleration}, the random variable \(\zeta_{ij}\) has the following properties:
\begin{align}
\mathbb{E}[\zeta_{ij}] = 0,\quad  \textit{ and }\quad
    \operatorname{Var}[\zeta_{ij}] = \Theta ( \alpha \cdot \Delta w_{\min}).\nonumber
\end{align}
\vspace{-2em}
\end{restatable}
The proof of Lemma \ref{lemma:pulse-update-error} is deferred to
Appendix \ref{Section:Proof_of_Lemma pulse-update-error}. 
Since analog crossbar arrays update all weight elements in parallel, the matrix update rule combining asymmetric pulse updates and quantization noise from limited conductance states can be succinctly represented as:
\begin{tcolorbox}[colback=gray!10, colframe=black, boxrule=0.5pt]
    \vspace{-0.5em}
\begin{align}
    \label{biased-update}
    W_{t+1} = W_t + \Delta W_t \odot F(W_t) -|\Delta W_t| \odot G(W_t)+ \zeta_t
 \end{align}
    \vspace{-1.1em}
\end{tcolorbox}
where the operations $|\cdot|$ and $\odot$ denote element-wise absolute value and multiplication, respectively. The specific form of $\Delta W_t$ depends on the chosen optimization algorithm.
By substituting  $\Delta W_t$ with the  gradient used in digital SGD, the update rule for Analog SGD under   \eqref{biased-update} becomes: \begin{align}
    \label{biased-update_sgd}
    W_{t+1} = W_t -\alpha\nabla f(W_t;\xi_t) \odot F(W_t) -|\alpha\nabla f(W_t;\xi_t)| \odot G(W_t)+ \zeta_t. 
 \end{align} 
 Based on~\eqref{biased-update_sgd}, we establish the upper and lower bounds on the convergence of Analog SGD on a single tile with limited conductance states. Our analysis shows that asymmetric pulse responses and the quantization noise term \(\zeta_t\) arising from limited conductance states pose fundamental challenges to convergence and lead to a non-negligible asymptotic error during training.  \begin{theorem}[Convergence of Analog SGD, short version]
\label{theorem_Convergence of_Analog_SGD_short}
Under a set of mild assumptions, with \( \sigma^2 \) denoting the variance bound of the  gradient noise, if the learning rate is set as 
$\alpha =\mathcal{O} \big(\sqrt{ \frac{ 2(f(W_0) - f^*)}{\sigma^2 T}}\big)$, then
it holds that:
\[
 \frac{1}{T} \sum_{t=0}^{T-1} \mathbb{E}[\|W^* - W_t\|^2]
\leq  \mathcal{O} \Big(R_T\sqrt{ \frac{ 2(f(W_0) - f^*)\sigma^2}{ T}}\Big)+4\sigma^2 S_T + R_T\Delta w_{\min}
\]
where \(
  S_T := \frac{1}{T} \sum_{t=0}^{T-1} \frac{\|W_t\|^2_\infty / \tau_{\max}^2}{1 - \|W_t\|^2_\infty / \tau_{\max}^2}\)
, \(
R_T := \frac{1}{T} \sum_{t=0}^{T-1} \frac{2L}{1 - \|W_t\|^2_\infty / \tau_{\max}^2}\).
\end{theorem}
Theorem \ref{theorem_Convergence of_Analog_SGD_short} suggests that the average squared Euclidean distance between \( W_t \) and \( W^* \) is upper bounded by the sum of three terms: the first term vanishes at a rate of 
\(
\mathcal{O}(\sqrt{\sigma^2 /T})
\),
which also appears in the digital SGD’s convergence bound; the second and third terms contribute to the \emph{asymptotic error} of Analog SGD, which does not vanish as the number of iterations \(T\) increases. 
Intuitively, the second term arises from the absolute gradient term $|\alpha \nabla f(W_t,\xi_t)|$ in   \eqref{biased-update_sgd}, which introduces variance scaling as $\alpha \sigma^2$. The third term originates from the quantization noise $\zeta_t$, which has variance $\Theta(\alpha \Delta w_{\min})$. In the convergence analysis, after normalizing by the descent coefficient $\alpha$, these two terms result in residual errors of order $\sigma^2$ and $\Delta w_{\min}$, respectively. In contrast, in digital SGD, the variance of sample noise scaling as  $\alpha^2 \sigma^2$ vanishes under diminishing learning rates. We next present a matching lower bound.
\begin{restatable}[Lower bound of the error of Analog SGD, short version]{theorem}{ThmOneTileAsymptoticError}
\label{theorem:one-tile-asymptotic-error} 
Under a set of mild assumptions, if the learning rate $\alpha = \frac{1}{2L}$, there exists an instance where \textit{Analog SGD} generates a sequence $\{W_t\}_{t=0}^{T-1}$ such that the iterates converge to a neighborhood of the optimal solution \( W^* \), satisfying:
\[
\frac{1}{T} \sum_{t=0}^{T-1} \mathbb{E}[\|W^* - W_t\|^2]\geq  \Omega( \sigma^2S_T+ R_T\Delta w_{\min}).
\]
\vspace{-1.2em}
\end{restatable}
The full versions of Theorem \ref{theorem_Convergence of_Analog_SGD_short} and \ref{theorem:one-tile-asymptotic-error} together with their proofs are deferred to Appendix \ref{section:Analog_Stochastic_Gradient_Descent_Convergence}.
 These theoretical insights underscore the importance of addressing quantization noise, which stands as a key obstacle to fully realizing the potential of analog neural network training.

\section{Multi-tile Residual Learning on  Non-Ideal Hardware}\label{sec.3}
\begin{wrapfigure}{r}{0.4\linewidth}
    \centering 
    \vspace{-1em}
\includegraphics[width=0.9\linewidth]{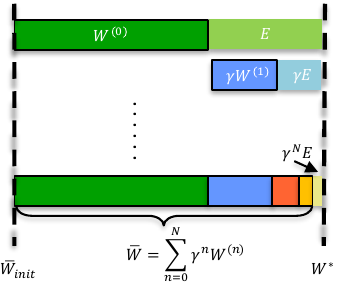}
\vspace{-0.5em}
\caption{Illustration of residual learning intuition. $\overline{W}_{\textit{init}}$ denotes the initial composite weight where $W^{(0)}$ is randomly initialized and all remaining tiles are 0.
}    \label{fig:demo_of_residual}
    \vspace{-4.5em}
\end{wrapfigure}
As discussed in Section~\ref{sec:Training_Dynamics}, single-tile training on non-ideal AIMC hardware inevitably results in a non-vanishing error. We propose a multi-timescale residual learning strategy for non-ideal AIMC hardware, where each additional scaled tile iteratively corrects the \emph{asymptotic residual error} left by the preceding  lower-resolution tiles due to limited conductance states and asymmetric updates.

\subsection{Multi-tile residual learning formulation}
Denote the weights stored on a single analog tile at iteration $t$ by $W^{(0)}_t$, the optimal weight by $W^*$, and the non-vanishing error as
\(
E := \lim_{t \to \infty}  W^* - W_t^{(0)}. 
\)
To mitigate this error, 
we introduce a second analog tile $W^{(1)}$, scaled by a factor $\gamma$, to iteratively compensate for it. As illustrated in Figure~\ref{fig:demo_of_residual}, rather than directly approximating $E$, the second tile approximates a scaled target $E/\gamma$. Although $W^{(1)}$ still suffers from similar device non-idealities and incurs a non-vanishing error when tracking its target (i.e., 
$
\lim_{t \to \infty}E/\gamma - W^{(1)}_{t} = E
$),
the combined output of the two tiles nonetheless converges to a smaller residual:
\begin{align}
   & \quad \lim_{t \to \infty}W^* - (W_t^{(0)} + \gamma W_t^{(1)}) = \lim_{t \to \infty}(W^* - W_t^{(0)}) - \gamma W_t^{(1)}= \lim_{t \to \infty}E - \gamma W_t^{(1)}
= \gamma E.  \nonumber
\end{align}
This shows that the use of an additional tile reduces the asymptotic residual by a factor of
$\gamma$. Extending this idea further, we introduce $N$ more analog tiles $W^{(1)}, \dots, W^{(N)}$, each tile $W^{(n)}$ is  scaled by a geometric factor $\gamma^n$. We define the geometric sum of the first $n$ tiles as
$
\overline{W}^{(n)} := \sum_{n'=0}^{n-1} \gamma^{n'} W^{(n')}, n \in [N],
$
so that the residual left by the first $n$ tiles is $W^* - \overline{W}^{(n)}$. We define the local optimal point for tile ${W}^{(n)}$ as
$
P_n^*(\overline{W}^{(n)}) := \arg\min_{P_n} f(\overline{W}^{(n)} + \gamma^n P_n).
$
Assuming that $f(\cdot)$ is strongly convex with a unique minimizer $W^*$, the optimal solution is
$
P_n^*(\overline{W}^{(n)}) = \gamma^{-n}(W^* - \overline{W}^{(n)}), \text{with } P_0^* := W^*.
$
To optimize each tile $W^{(n)}$, we minimize the objective
$
 \| W^{(n)} - P_n^*(\overline{W}^{(n)})\|^2,
$
so that $\gamma^n W^{(n)}$ approximates the residual left by the first $n$ tiles. Applying this process to all tiles finally yields an exponentially reduced error between the composite weight $\overline{W}$ and the optimal weight $W^*$.
Formally, we solve the multi-layer problem as:
\begin{subequations} \label{eq:nested_problem}
\begin{align}
W^{(0)} &:= \arg\min_{U_{0}}   \| U_{0}  - P_0^*  \|^2,\quad  P_0^* := W^*, \tag{4a} \\
W^{(1)} &:= \arg\min_{U_{1}}   \| U_{1}  - P_1^*(\overline{W}^{(1)})  \|^2, 
\quad \text{s.t. } P_1^*(\overline{W}^{(1)}) := \arg\min_{P_1} f(\overline{W}^{(1)} + \gamma P_1), \tag{4b}\\
&\ldots \notag\\
W^{(N)} &:= \arg\min_{U_{N}}    \| U_{N} - P_N^*(\overline{W}^{(N)}) \|^2, 
\quad \text{s.t. } P_N^*(\overline{W}^{(N)})  := \arg\min_{P_N} f(\overline{W}^{(N)} + \gamma^N P_N)\tag{4c}
\end{align}
\end{subequations}
where  $U_{n}, P_n \in \mathbb{R}^{D \times D}$ for $n \in \{0, \ldots, N\}$. 
\begin{figure}[t]
     \vspace{-0.5 em}
  \centering
\includegraphics[width=0.9\linewidth]{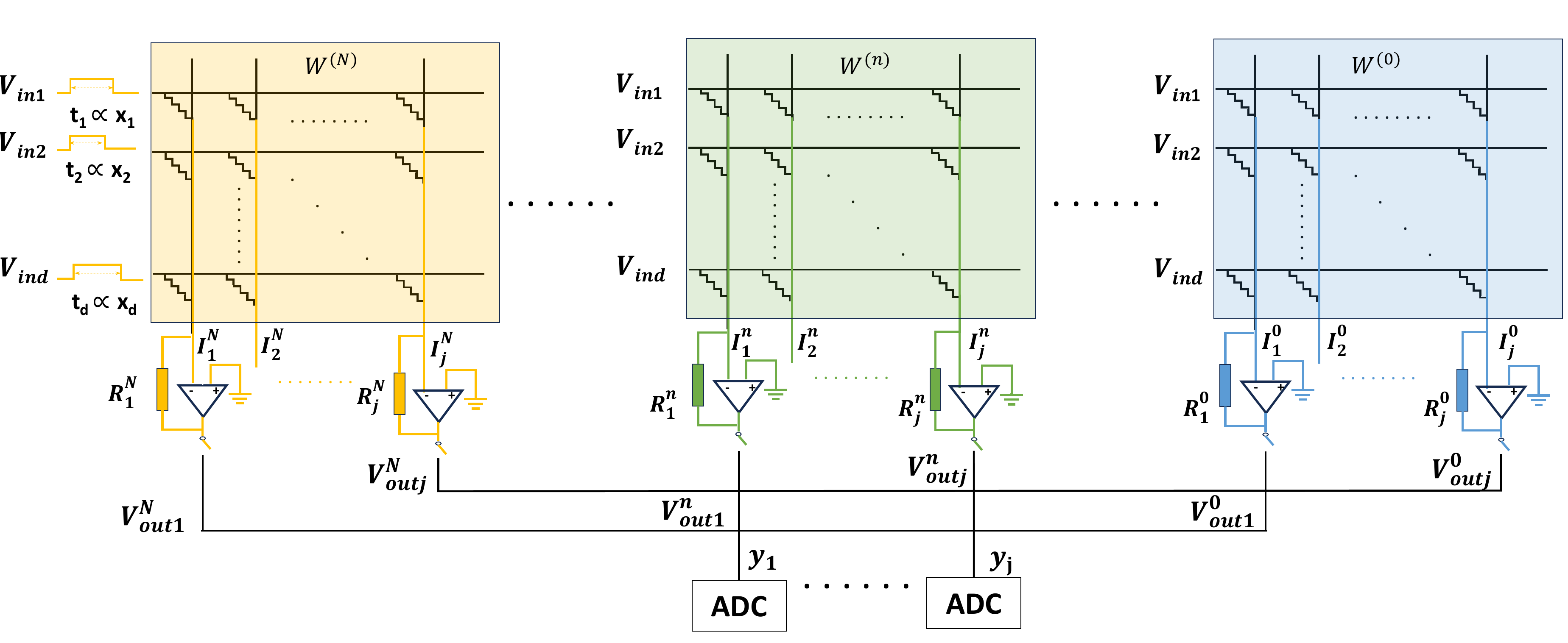}
      \vspace{-0.7 em}
    \caption{Analog circuit implementation of forward process using composite weight~$\overline{W}$.}
      \vspace{-1em} \label{fig:analog_forward_practice}
\end{figure}
\subsection{Multi-tile gradient-based update}
The optimization problem is challenging because the drifting optimum $P_n^*(\overline{W}^{(n)})$ is an implicit function of $\overline{W}^{(n)}$. To decouple this dependency when optimizing $W^{(n)}$, we freeze tiles $\{W^{(0)}, \ldots, W^{(n-1)}\}$ to ensure that $\overline{W}^{(n)}$ remains fixed.
To optimize each tile $W^{(n)}$ in the~\eqref{eq:nested_problem}, we aim to update via an approximate descent direction of its objective $\| W^{(n)} - P^*_n(\overline{W}^{(n)}) \|^2$ . The negative gradient of the objective is :
\begin{align}
    \label{eq:tile_gradient}
    -\nabla_{W^{(n)}} \| W^{(n)} - P^*_n(\overline{W}^{(n)}) \|^2 = 2 (P^*_n(\overline{W}^{(n)}) - W^{(n)})
\end{align}
which implies that $P^*_n(\overline{W}^{(n)}) - W^{(n)}$  is the descent direction to update $W^{(n)}$.  
 For $n = N$,
since
$$P^*_N(\overline{W}^{(N)}) - W^{(N)}= \gamma^{-N}(W^* - \overline{W}^{(N)})- W^{(N)} = \gamma^{-N}(W^* -\sum_{n=0}^{N} \gamma^{n} W^{(n)})=\gamma^{-N}(W^* - \overline{W})$$ and 
$\mathbb{E}_{\xi} [-\nabla f(\overline{W}; \xi)] = -\nabla f(\overline{W}) \propto W^* - \overline{W}$ under strong convexity assumptions, we directly use the stochastic gradient $-\nabla f(\overline{W}; \xi)$ as the descent direction to update $W^{(N)}$.
For $n \in [N-1]$, since the optimization on $W^{(n+1)}$ (see   \eqref{eq:nested_problem}) ensures that $W^{(n+1)} \approx P^*_{n+1}(\overline{W}^{(n+1)}) = \gamma^{-1}(P^*_n(\overline{W}^{(n)}) - W^{(n)})$, we use $W^{(n+1)}$ to update $W^{(n)}$.
Following the update rule in \eqref{biased-update}, for tile $W^{(N)}$, the update is given by:
\begin{align}
\label{eq:TT-update-n-I}
W_{t+1}^{(N)} &= W_{t}^{(N)} 
 - \alpha \nabla f(\overline{W}_{t}; \xi_{t}) \odot F\bigl(W_{t}^{(N)}\bigr) 
 - \bigl|\alpha \nabla f(\overline{W}_{t}; \xi_{t})\bigr| \odot G\bigl(W_{t}^{(N)}\bigr)
 + \zeta_{t}.
\end{align} 
For  $W^{(n)}$, $n \in [N-1]$, the update is given by:
\begin{align}
\label{eq:TT-update-n-II}
W_{t_n+1}^{(n)} &= W_{t_n}^{(n)} 
 + \beta  \tilde{W}^{(n+1)} \odot F\bigl(W_{t_n}^{(n)}\bigr) 
 - \bigl|\beta  \tilde{W}^{(n+1)}\bigr| \odot G\bigl(W_{t_n}^{(n)}\bigr)
 + \zeta_{t_n}
\end{align}
where $\beta$ is the learning rate, and the transferred weight is defined as $\tilde{W}^{(n+1)} := W^{(n+1)}_{t_{n+1} + T_{n+1} - 1}$. 
We show in Section \ref{Sec:Theorem} that each tile $W^{(n)}$ requires an inner loop of $T_n = \Theta(\gamma^{-1})$ steps to converge to its optimum $P_n^*(\overline{W}^{(n)})$. We thus adopt a multi-timescale training schedule to coordinate these updates,
with each tile $W^{(n)}$ maintains a local step counter $t_n = \lfloor (t + 1)/\prod_{n'=n+1}^{N} T_{n'} \rfloor$. A detailed algorithm is provided in Algorithm \ref{alg:bit-slice-tt-switch}.

\begin{remark}
The optimization problem in \eqref{eq:nested_problem} resembles the STSA framework \citep{shen2022single}, where each sequence tracks a drifting optimum that evolves with the updates of other sequences, denoted as $y^{n,*}(y^{n-1})$ in STSA and $P^*_n(\overline{W}^{(n)})$ in our setting.
However, directly applying STSA to our problem encounters two main difficulties:
\textbf{C1)} STSA relies on rapid convergence of each sequence to its drifting optimum via a single update step, but 
in our composite weight structure, a single-step  update on \(W^{(n-1)}\) causes $P_n^*(\overline{W}^{(n)})$ to drift approximately \(\Theta(\gamma^{-1})\) times faster than a single-step update on \(W^{(n)}\); 
\textbf{C2)} STSA considers that $y^{n,*}(y^{n-1})$ depends only on sequence $y^{n-1}$, while our scenario involves \(P^*_{n}(\overline{W}^{(n)})\) depending on multiple sequences \(\overline{W}^{(n)}\).
\end{remark}
 \begin{remark}
Our update dynamics for each tile naturally support an open-loop transfer process. As shown in \eqref{eq:tile_gradient}, what needs to be propagated is only the descent direction of each tile rather than its exact weight value. This eliminates the need for closed-loop tuning, thereby reducing control overhead and highlighting the hardware efficiency of our algorithm.
\end{remark}

\subsection{Analog circuit implementation}
\label{sec:Analog_Circuit_Implementation}
Figure~\ref{fig:analog_forward_practice} illustrates how the composite weight \( \overline{W} \) is formed in the analog domain by combining low-precision tiles \( W^{(n)} \in \mathbb{R}^{D \times D} \), \( n \in \{0,\ldots, N\} \).
 For the forward pass \( y = x^{\top} \overline{W} \), each input \( x_d \) is encoded by a voltage pulse on the \( d \)-th row of each tile \( W^{(n)} \), with duration proportional to \( x_d \).
 By Ohm’s and Kirchhoff’s laws, each tile produces a current of the $j$-th column as $
I_j^{n}  = \sum_{d=1}^D W_{d,j}^{(n)} x_d$. 
Each \( I_j^n \) is fed into an inverting op-amp with feedback resistor \( R_j^n \) to apply scaling \( \gamma^n \), yielding \( V_{\text{out}j}^n = -R_j^n I_j^n \). The voltage outputs are summed in hardware to produce the final result \( y_j \) as: 
\begin{align}
    y_j  \propto  \sum_{n=0}^{N} V_{\text{outj}}^{n}
= -\sum_{n=0}^{N} R_j^{n} I_j^{n}
= -\sum_{n=0}^{N} R_j^{n}  \sum_{d=1}^D W_{d,j}^{(n)} x_d
= -\sum_{d=1}^D \overline W_{d,j} x_d.
\nonumber \end{align}
A similar operation is used during the backward pass, where the output is given by
$\delta = \overline W^{\top} \delta'$, with $\delta'$ denoting the error signal propagated from the next layer. 
Table~\ref{tab:weight-update-comparison} compares the computational complexity and estimated update latency of MP, Analog SGD, TT-v2, and our method. Our algorithm achieves a low update latency, upper bounded by 95.9 ns even with an infinite number of tiles, which is over 30$\times$ faster than MP, while requiring only $\mathcal{O}(2D)$ digital storage for the input $x$ and error $\delta$ and $\mathcal{O}(1)$ digital memory operations for reading and writing, which is comparable to the cost of Analog SGD.
Please see Appendix~\ref{appendix:complexity_analysis}, where we discuss the feasibility of circuit-level implementation and present a detailed comparison of digital storage, runtime, energy, and area costs across algorithms.

\section{Stochastic Approximation Theory for Residual Learning}
\label{Sec:Theorem}
In this section, we present a proof sketch for the convergence of our proposed multi-timescale residual learning algorithm. Before analyzing the algorithm, we introduce four assumptions concerning the objective function, the gradient noise, and the device response characteristics.
\begin{assumption}[Unbiasness and bounded variance]
    \label{assumption:noise}
    The sample $\xi_t$ is independently sampled from a distribution $\ccalD$, $\forall t\in[T]$, and the stochastic gradient is unbiased with bounded variance, i.e., $\mbE_{\xi_t}[\nabla f(W_t;\xi_t)] = \nabla f(W_t)$ and $\mbE_{\xi_t}[\|\nabla f(W_t;\xi_t)-\nabla f(W_t)\|^2]\le\sigma^2$.  
\end{assumption}
\begin{assumption} [Smoothness and strong convexity]
\label{assumption:Lip}
  $f(W)$ is $L$-smooth and $\mu$-strongly convex. 
\end{assumption}
\begin{assumption}[Bounded weights]
\label{assumption:bounded-saturation}
The weights are bounded as $\|W_t\|_\infty \leq W_{\max} \leq \tau_{\max}$ for all $t$.
\end{assumption}
\begin{table}[t]
\small
  \centering
  \vspace{-1em}
  \setlength{\tabcolsep}{4pt}
\renewcommand{\arraystretch}{1.2}
  \begin{tabular}{c c c | c | c  c c}
    \toprule
    \textbf{Dataset} & \textbf{TT-v1} & \textbf{TT-v2} & \textbf{MP}
      & \textbf{Ours (3 tiles)} & \textbf{Ours (4 tiles)} & 
      \textbf{Ours (6 tiles)} \\
    \midrule
    Fashion-MNIST  (\#4)
          & 10.01{\tiny$\pm$0.07} 
      & 47.51{\tiny$\pm$0.91} 
      & 75.61{\tiny$\pm$0.69} 
      & 68.09{\tiny$\pm$0.49}  
      & 73.35{\tiny$\pm$0.13}
      & 75.11 {\tiny$\pm$0.07}
      \\
   MNIST  (\#10)
      & 78.65{\tiny$\pm$2.36} 
      & 95.43{\tiny$\pm$0.17}
      & 99.13{\tiny$\pm$0.02} 
      & 95.07{\tiny$\pm$0.35} 
      & 97.10{\tiny$\pm$0.17} 
      & 98.53{\tiny$\pm$0.09} \\
    \bottomrule
  \end{tabular}
   \vspace{-0.5em}
\caption{
Test accuracy on MNIST and Fashion-MNIST with analog LeNet-5 under 10 and 4 states. 
Compared methods include MP, TT-v1, TT-v2 and different versions of our algorithm.
}
\vspace{-1em}
\label{tab:num_states_comparison}
\end{table}
\begin{table}
\centering
\small
% \vspace{-0.5em}
\setlength{\tabcolsep}{6pt}
\renewcommand{\arraystretch}{1.2}
\begin{tabular}{ccc|c|ccccc}
\toprule
\textbf{Dataset} & \textbf{TT-v1} & \textbf{TT-v2} & \textbf{MP} & \textbf{Ours (4 tiles)} & \textbf{Ours (6 tiles)} & \textbf{Ours (8 tiles)} \\
\midrule
CIFAR-10 (\#4) & 11.65{\tiny$\pm$0.68} & 87.43{\tiny$\pm$0.10} & 93.31{\tiny$\pm$0.04} & 90.45{\tiny$\pm$0.28} & 92.02{\tiny$\pm$0.14} & 90.65{\tiny$\pm$0.09} \\
CIFAR-100 (\#4) & 
9.75{\tiny$\pm$2.47} & 11.28{\tiny$\pm$0.62} & 70.25{\tiny$\pm$0.61} & 63.26{\tiny$\pm$0.45} & 68.46{\tiny$\pm$0.23} & 
69.63{\tiny$\pm$0.40}\\
CIFAR-10 (\#16)& 58.00{\tiny$\pm$0.26} & 84.30{\tiny$\pm$0.09} & 95.04{\tiny$\pm$0.05} & 92.98{\tiny$\pm$0.62} & 93.29{\tiny$\pm$0.64} & 
94.36{\tiny$\pm$0.08}\\
CIFAR-100 (\#16) & 26.69{\tiny$\pm$0.17} & 65.17{\tiny$\pm$0.36} & 73.16{\tiny$\pm$0.07} & 68.11{\tiny$\pm$0.35} & 69.62{\tiny$\pm$0.53} &
72.20{\tiny$\pm$0.11} \\
\bottomrule
\end{tabular}
\vspace{-0.5 em}
\caption{
Test accuracy on CIFAR-10  and CIFAR-100 under 4 and 16 conductance states on ResNet-34. Compared methods
include MP, TT-v1, TT-v2 and different versions of our algorithm.
}
\label{tab:num_states_comparison_cifar_models}
% \vspace{-2em}
\end{table}
\begin{assumption}[Response factor and zero shifted symmetric point]
\label{assumption:pulse-response-symmetry} 
\textbf{(Continuity)} \( q_+(\cdot) \) and \( q_-(\cdot) \) are continuous; 
\textbf{(Saturation)} \( q_+(\tau_{\max}) = 0 \), \( q_-(\tau_{\min} ) = 0 \); 
\textbf{(Positive-definiteness)} \( q_+(w) > 0 \) for all \( w < \tau_{\max} \), and \( q_-(w) > 0 \) for all \( w > \tau_{\min} \); 
\textbf{(Symmetric point)} \( G(w) = 0 \) if and only if \( w = 0 \).
\end{assumption}
\vspace{-0.5em}
Assumptions \ref{assumption:noise}--\ref{assumption:Lip} are standard in convex optimization \citep{bottou2018optimization}.  Assumption \ref{assumption:bounded-saturation} assumes that $W_t$ remains within a small region, which is a mild condition that generally holds in practice. 
Assumption \ref{assumption:pulse-response-symmetry} defines the response function class observed in resistive devices \citep{wu2025analog} and adopts the widely used zero-shifted symmetric point in analog training \citep{gokmen2020algorithm}. 
We begin by presenting Lemmas \ref{lemma:2} and \ref{lemma:3}, which describe how each tile tracks its drifting optimum within an inner loop and serve as the basis for our convergence analysis. Their full proofs are given in Appendix \ref{section:proof-TT-convergence-scvx}. For notational convenience, here we write $P^*_{n}(\overline{W}^{(n)}_{t_{n-1}}) =: P^*_{n}$. 
\begin{lemma}[Descent lemma of the main sequence \(W^{(N)}\)]
   \label{lemma:2}
    Under Assumptions~\ref{assumption:noise}--\ref{assumption:pulse-response-symmetry},
 the update dynamics~\eqref{eq:TT-update-n-I} ensures that after a single inner loop of length \(T_N\), the expected distance between \(W^{(N)}\) and its optimum decreases as:
\begin{align}
\mathbb{E} \bigl[\|W^{(N)}_{t+T_N-1} - P^*_{N}\|^2\bigr]
&\le 
\left(1 - \Theta(\gamma^N)\right)^{T_N} \|W^{(N)}_{t} - P^*_{N}\|^2
+\Theta (\sigma^2 \gamma^{-N}+\gamma^{-\frac{4N}{3}} (\sigma \Delta w_{\min})^{\frac{2}{3}}). \nonumber
\end{align}
\end{lemma}
\begin{lemma}[Descent lemma of the sequences \(W^{(n)}\)]
   \label{lemma:3}
Under the same assumptions as Lemma \ref{lemma:2}, for \(n\in\{0,\ldots, N-1\}\), the update dynamics~\eqref{eq:TT-update-n-II} ensures that:
\begin{align}
\mathbb{E} \bigl[\|W^{(n)}_{t_n+T_n-1} - P^*_{n}\|^2\bigr]
&\le 
\left(1 - \Theta(\gamma)\right)^{T_n} \|W^{(n)}_{t_n} - P^*_{n}\|^2
+ \Theta(\gamma^2\|\Tilde W^{(n+1)} - P^*_{n+1}\|^2
+\Delta w_{\min}). \nonumber
\end{align}
\end{lemma}
\vspace{-0.5em}
Lemmas~\ref{lemma:2} and \ref{lemma:3} yield contraction terms $(1 - \Theta(\gamma^p))^{T_n}$ with $p=1$ for $W^{(n)}$ and $p=N$ for $W^{(N)}$. Using $(1-\lambda)^{T_n} \le e^{-\lambda T_n}$, we set $T_n=\Theta(\gamma^{-1})$ and $T_N=\Theta(\gamma^{-N})$ so that $(1-\Theta(\gamma^p))^{T_n}=\Theta(\rho)$, where $\rho$ is the contraction rate in the Lyapunov analysis below.
To this end, we now analyze the optimization problem in  \eqref{eq:nested_problem} by introducing a Lyapunov sequence
as 
$
 \mathbb{J}_{k} := \sum_{n=0}^N \|W^{(n)}_{t_n + kT_n - 1} - P^*_{n}(\overline{W}^{(n)}_{t_{n-1}+k})\|^2
$,
 which aggregates the squared distances between each tile and its drifting optimum after completing $k$ inner loops. In particular, when the slowest tile $W^{(0)}$ has completed $k$ inner loops, the fastest tile $W^{(N)}$ will have performed
$
t = k \prod_{n=0}^N T_n = \mathcal{O}(\gamma^{-2N}k)
$
gradient updates, which we take as the reference measure for evaluating the overall convergence rate. By balancing the learning rates and inner-loop length across all tiles, we establish a linear convergence rate for the Lyapunov sequence with a non-vanishing asymptotic error induced by device-level imperfections. This result is formalized in the following theorem.
\begin{restatable}[Convergence of residual learning]{theorem}{ThmTTConvergenceScvx}
\label{theorem:TT-convergence-scvx_long}
Suppose Assumptions~\ref{assumption:noise}--\ref{assumption:pulse-response-symmetry} hold. Let the scaling parameter satisfy \(\gamma \in (0, 1/
\sqrt{6}]\).
For all \(n\in\{0,\ldots, N-1\}\), set the learning rate \(\beta = \Theta(\gamma^2)\), the inner loop length
\(
T_{n} \geq \Theta \bigl(\gamma^{-1}\bigr),
\) except for $T_0= \Theta(1)$.
For \(n = N\), set the learning rate \(\alpha = \Theta(1)\) and 
\(
T_{N} \geq \Theta \bigl(\gamma^{-N}\bigr).
\)
Given \(t=\mathcal{O}(\gamma^{-2N}k) \), $\rho \in (0,1)$,
the Lyapunov function $\mathbb{J}_k$ is bounded as:
\begin{align}
   \mathbb{E}[\mathbb{J}_k] \le \mathcal{O}((1 - \rho)^{\gamma^{2N} t} )\mathbb{E}[\mathbb{J}_0] + \Theta(\gamma^{-\frac{4N}{3}} (\sigma \Delta w_{\min})^{\frac{2}{3}}).  \nonumber 
\end{align}
\end{restatable}
\vspace{-0.5em}
Theorem \ref{theorem:TT-convergence-scvx_long} shows the Lyapunov function converges linearly up to an asymptotic upper bound \(\Theta(\gamma^{-\frac{4N}{3}} (\sigma\Delta w_{\min})^{2/3})\) determined solely by quantization and sample noise. 
Corollary~\ref{corollary:Convergence_rate} follows from Theorem~\ref{theorem:TT-convergence-scvx_long} 
by applying the Lyapunov function bound as an upper bound on the component 
\(\|W^{(N)}_{t_N+kT_N-1} - P^*_{N}(\overline{W}^{(N)}_{t_{N-1}+k})\|^2\). 
Then using the definition of \(P^*_N(\overline{W}^{(N)})\) and multiplying both sides by \(\gamma^{2N}\) yields the optimality gap of the composite weight.

\begin{restatable}
[Optimality gap of residual learning]{corollary}{ThmTTConvergencecorollary}
\label{corollary:Convergence_rate}
Under the same conditions as in Theorem~\ref{theorem:TT-convergence-scvx_long}, 
 the limit of the composited weight \(\overline{W}_{t}\) satisfies:
\begin{align}
\limsup_{t \to \infty} \mathbb{E}\big[\|W^* - \overline{W}_{t}\|^2\big]
\leq 
\Theta(\gamma^{\frac{2N}{3}} (\sigma \Delta w_{\min})^{\frac{2}{3}}).\nonumber
\end{align}
\end{restatable}
\vspace{-1em}
The proofs of Theorem \ref{theorem:TT-convergence-scvx_long} and Corollary \ref{corollary:Convergence_rate} can be found in Appendix \ref{section:proof-TT-convergence-scvx}. 
Intuitively, increasing the value of $N$ reduces the upper bound $\Theta\big(\gamma^{\frac{2N}{3}} (\sigma \Delta w_{\min})^{\frac{2}{3}}\big)$, demonstrating the effectiveness of multi-timescale residual learning under a limited number of conductance states.
\section{Numerical Simulations}
\label{section: Numerical Simulations}
In this section, we evaluate our method using the AIHWKIT toolkit on the SoftBounds device class, which models bi-directional memristive devices such as ReRAM by capturing saturation effects and limited update precision. We compare against the MP, TT-v1, and TT-v2 baselines. 
Detailed algorithmic procedures are provided in Appendix~\ref{Sec:Pseudocode}. 
\subsection{Analog training performance on real dataset}
We train an analog LeNet-5 model on the MNIST and Fashion-MNIST dataset for 100 epochs with 4 and 10 conductance states, respectively.
We also train the ResNet-34 model on the CIFAR-10 and CIFAR-100 datasets with 4 or 16 conductance states, covering both an extreme low-precision case and the widely adopted industrial setting. Training is performed for 200 epochs on CIFAR-10 and 400 epochs on CIFAR-100, with \texttt{layer3}, \texttt{layer4}, and the  fully connected layer mapped to analog.  The parameter configurations for TT-v1, TT-v2, MP, and our method are provided in Appendix~\ref{sec:Simulation_Details}.
As shown in Tables~\ref{tab:num_states_comparison} and \ref{tab:num_states_comparison_cifar_models}, our residual learning method steadily improves accuracy as the number of tiles increases, surpassing both TT-v1 and TT-v2  with only 3 to 4 tiles, and  reaching accuracy comparable to MP while incurring far lower storage and runtime overhead, demonstrating both scalability and robustness across larger networks.
To further illustrate the scalability of our method, we also provide results in Appendix~\ref{sec:Supplement_simulations} on 
models with larger analog deployments and higher conductance states. 
\subsection{Ablation studies}
\label{sec:Ablation_studies}
\begin{figure}[t]  
    \centering
    % (a)
    \begin{minipage}{0.3\linewidth}
        \centering
        \includegraphics[width=\linewidth]{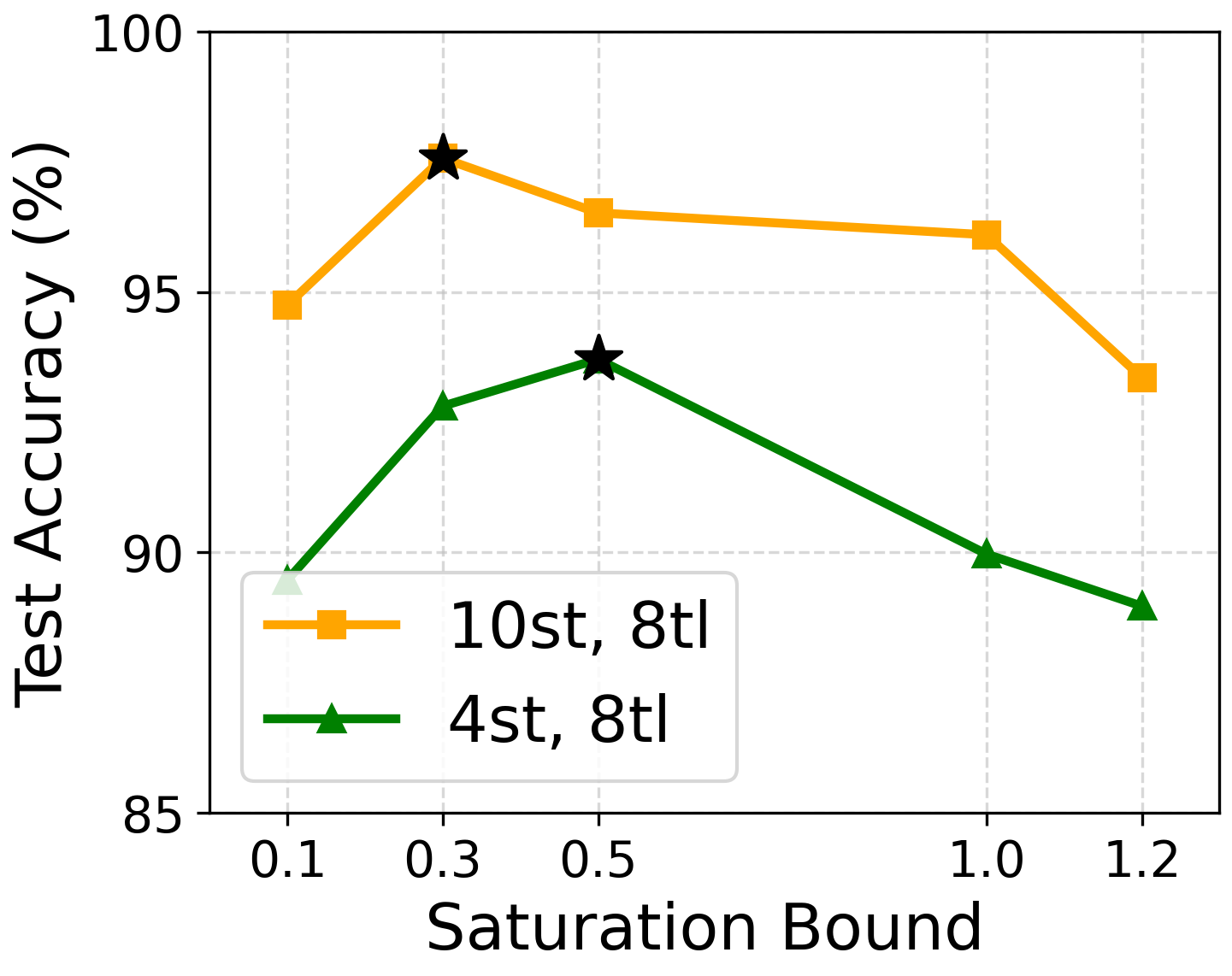}
        % (a)
    \end{minipage}
    \begin{minipage}{0.3\linewidth}
        \centering
        \includegraphics[width=\linewidth]{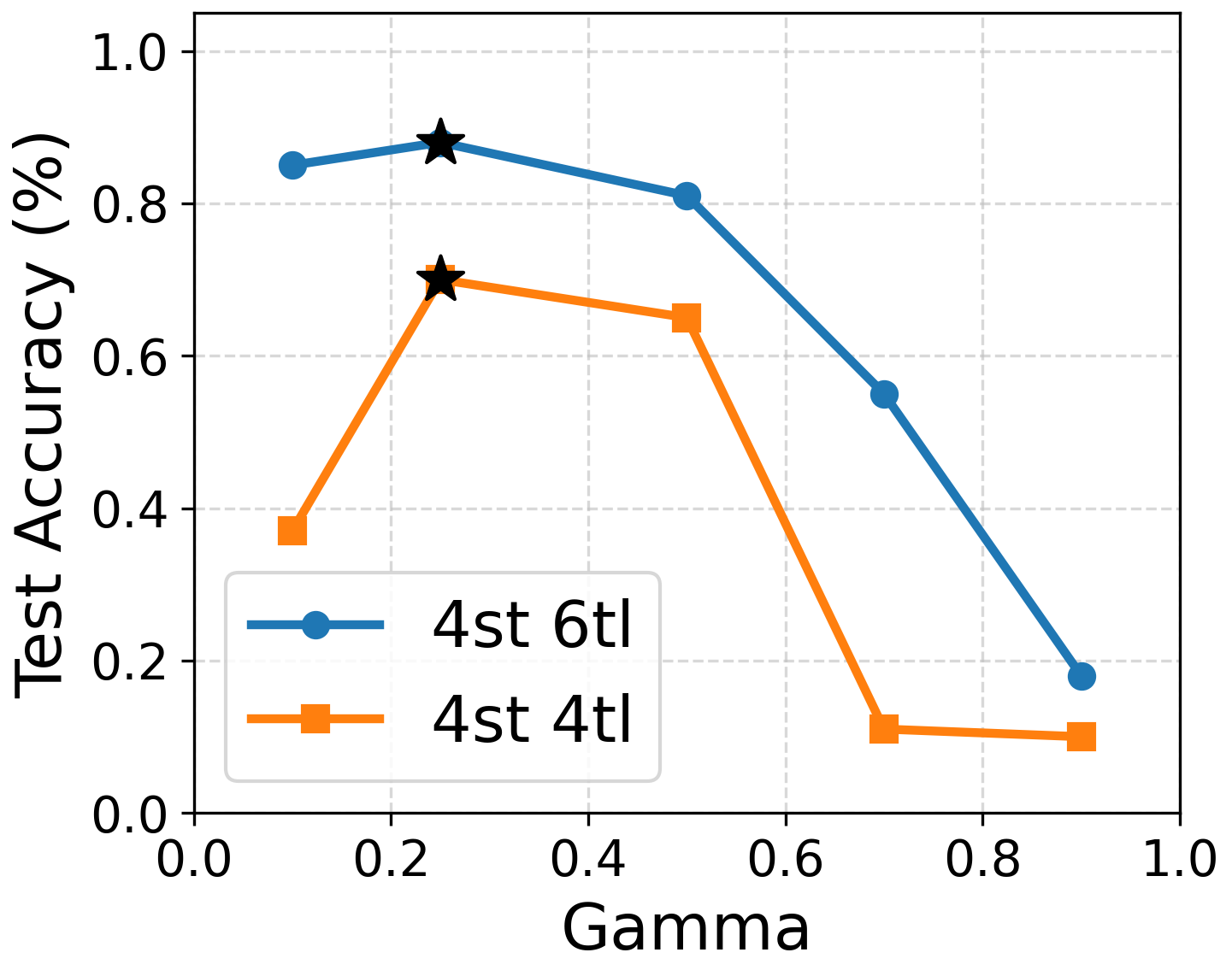}
        % (b)
    \end{minipage}
     \begin{minipage}{0.3\linewidth}
        \centering
        \includegraphics[width=\linewidth]{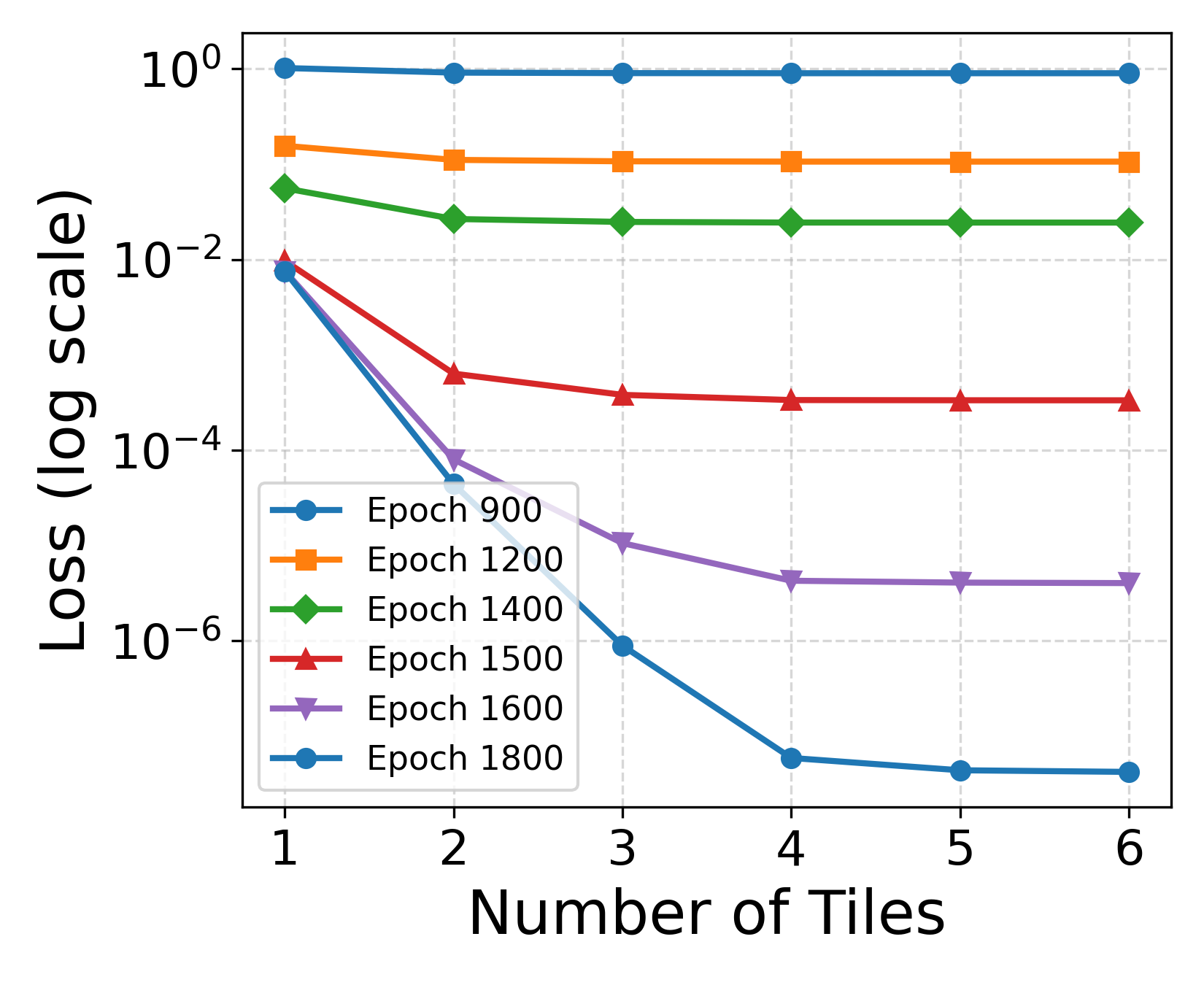}
        % (c)
    \end{minipage}
        \vspace{-0.2cm}
     \caption{({\sf left}) Effect of asymmetry. st: \#states, tl: \#tiles. 
    ({\sf middle}) Effect of geometric scaling factor $\gamma$. ({\sf right}) A toy example illustrating that training loss decreases along both the epoch and tile-count dimensions. The results demonstrate that each tile converges successfully, meanwhile, incorporating more tiles also improves accuracy.}
    \label{fig:asymmetry_gamma}
        \vspace{-0.4cm}
\end{figure}

\textbf{Effect of asymmetry.}  
For a given number of conductance states, the degree of asymmetry in an analog device is determined by the saturation bound $\tau_{\max}$ (with $\tau_{\min} = -\tau_{\max}$). 
We evaluate training accuracy with $\tau_{\max}$ varying in $(0,1)$ on MNIST and show that our algorithm consistently maintains high accuracy across different levels of asymmetry in Figure \ref{fig:asymmetry_gamma}~({\sf left}).

\textbf{Effect of geometric scaling factor.} As show in Figure \ref{fig:asymmetry_gamma} ({\sf middle}), we find that the optimal scaling factor $\gamma$ generally lies near $\frac{1}{n_{\text{states}}}$, which ensures that the weight range of each tile is properly matched to the resolution of its previous tile, whereas overly large $\gamma$ severely degrades accuracy. 

\textbf{Effect of number of tiles.}
To validate the convergence behavior of our residual learning mechanism, we construct a toy example based on a simple least-squares problem of the form $(w - b)^2$. The target output $b$ is quantized to 16-bit resolution, and each tile has 2-bit update granularity. The parameter configurations are provided in Appendix \ref{sec:Simulation_Details}. 
Figure \ref{fig:asymmetry_gamma} ({\sf right}) presents the log-scaled training loss. 
This visualization highlights two aspects of the learning dynamics. First, the composite weight converges reliably toward the target $w^* = b$. Second, the loss decreases consistently with more tiles, confirming the effectiveness of the multi-tile strategy.

\section{Conclusions and Limitations} 
\label{sec:limitations}
\vspace{-0.5em}
We propose multi-timescale residual learning, an in-memory analog training framework that enables reliable DNN training under limited conductance states by modeling device non-idealities and proving convergence, achieving strong results on standard image classification tasks.
We acknowledge the lack of real hardware evaluation and plan to validate our method through future chip-level experiments.

\bibliography{ref}
\bibliographystyle{plainnat}

\newpage
\appendix
\onecolumn
\begin{center}
\Large \textbf{Supplementary Materials} 

\end{center}

\addcontentsline{toc}{section}{} 
\part{} % Start the appendix part
\parttoc
\section{Literature Review }
 This section briefly reviews literature that is related to this paper, as complementary to Section \ref{sec:related works}.
\label{sec:Memeristor_Devices}

\textbf{Gradient-based training on AIMC hardware.}
Since the introduction of \emph{rank-update-based} Analog SGD~\citep{gokmen2016acceleration}, various techniques have been proposed to improve its robustness under non-ideal device behavior. 
To mitigate asymmetric updates and noise accumulation in Analog SGD, TT-v1 uses an auxiliary array to accumulate a moving average of gradients, periodically transferring them to the main array~\citep{gokmen2020algorithm}.
TT-v2 enhances this approach with digital filtering to improve robustness against noise and low conductance ~\citep{gokmen2021enabling}. However, despite these refinements, rank-update methods still struggle to converge reliably when scaling to larger models under limited conductance states. Some in-memory analog training approaches employ hybrid 3T1C–PCM synapses, which rely on closed-loop tuning (CLT) with about 20 write–verify retries to transfer gradients from the linear 3T1C capacitor to the non-ideal PCM device~\citep{ambrogio2018equivalent,cristiano2018perspective}. While precise, CLT introduces substantial latency, energy, and control overhead, and the bulky 3T1C CMOS cells also limit array density compared to compact memristive devices, which has motivated a shift toward more compact, modestly linear NVMs.
Moreover, the volatile nature of CMOS capacitors means that the effective number of valid states is ultimately determined by the non-volatile PCM, preventing extension to arbitrary precision beyond the intrinsic limits of memristive devices.
Our approach avoids the use of auxiliary CMOS components and iterative CLT, and instead employs fully asymmetric and limited-precision memristor-based synapses, which are trained in-memory through a simple \emph{open-loop transfer} mechanism between tiles and can achieve high-precision training.
Meanwhile, another class of in-memory training approaches known as \emph{hybrid training paradigms}  has emerged.The MP approach~\citep{le2018mixed}, for example, computes gradients digitally and directly programs these high-precision gradients into low-precision analog weights, enabling high-accuracy training even with very few conductance states. Subsequent extensions~\citep{wang2020ssm, huang2020overcoming} incorporate momentum to further suppress gradient noise. However, these methods often incur higher storage, memory access, and computational costs.

\setlength{\tabcolsep}{2pt} 
\begin{wraptable}{r}{0.65\linewidth}
  \vspace{-1.5 em}
\centering
\begin{tabular}{@{}lcc@{}}
\toprule
\textbf{Device Name} & \textbf{\# States} & \textbf{Mature} \\ \midrule
Capacitor~\citep{li2018capacitor}    & 400  & \ding{51} \\
ECRAM~\citep{tang2018ecram}        & 1000 & \ding{55}\\
ECRAM (MO)~\citep{kim2019metal}        & 7100 & \ding{55} \\
PCM~\citep{nandakumar2020mixed}       & 200  & \ding{51}\\
RERAM (OM)~\citep{gong2022deep}    & 21   & \ding{51}\\
RERAM (HfO$_2$)~\citep{gong2022deep}   & 4    & \ding{51}  \\ 
RERAM (AlO$_x$/HfO$_2$)~\citep{woo2016improved}   & 40    & \ding{51}  \\ 
RERAM (PCMO 
)~\citep{park2013neuromorphic}   & 50    & \ding{51}  \\ 
RERAM (HfO$_2$)~\citep{jiang2016sub}   & 26    & \ding{51}  \\ \bottomrule
\end{tabular}
\caption{
Comparison of representative analog memory devices used for DNN training. 
Here, \textit{Mature} denotes whether the corresponding class of memristor devices has been demonstrated with stable and reproducible fabrication processes \citep{joshi2020accurate}. 
}
\label{table:device_granularity}
\vspace{-1 em}
\end{wraptable}
\textbf{Low-precision computing.}  
Existing works have mainly demonstrated how low-precision devices can support high-precision computing in scientific computing and DNN inference. In scientific computing, algorithmic techniques have been proposed to achieve high-precision arithmetic on binary devices \cite{feinberg2018enabling, song2023refloat}, and residual programming strategies sequentially approximate the residuals of previously programmed tiles \cite{song2024programming}. 
The residual concept has also inspired precision-enhancing strategies in DNN inference, where multiple cells are used to encode high-precision weights \cite{le2022precision, pedretti2021conductance, boybat2018neuromorphic}. Varying the significance of these devices has been shown to yield further accuracy improvements \cite{mackin2022optimised}.
Unlike these approaches that operate on static weights, our method tackles the more challenging setting of training, where weights dynamically evolve under asymmetric and low-precision updates across multiple coupled tiles, and must still converge jointly toward the optimal solution.
Alternative precision enhancement strategies include incorporating hardware specific noises into the training process to improve inference accuracy \cite{klachko2019improving,he2019noise}.
Other approaches include logarithmic weight-to-conductance mappings that bias encoding toward more stable device states~\cite{vasilopoulos2023exploiting}, and inference~\cite{zhang2017memory} through weighted majority voting after offline learning of binary weights
and column-specific scaling factors.

\textbf{Memristor devices.}
To provide an intuitive overview of the characteristics of current analog memristor devices, particularly their conductance states, we survey recent works on ReRAM~\citep{stathopoulos2017multibit,chang2011short}, PCM~\citep{burr2015experimental,burr2016recent}, MRAM~\citep{apalkov2013spin,rzeszut2022multi}, and ECRAM~\citep{fuller2019parallel}, and summarize their reported conductance state information in Table~\ref{table:device_granularity}. In this work, our experiments and analysis primarily focus on bi-directionally updated devices such as ReRAM, which represent one of the most practically relevant device classes for analog training. To the best of our knowledge, most reported ReRAM devices provide fewer than 100 distinct conductance levels, with practical demonstrations typically limited to around 4-bit precision, which poses a fundamental limitation for scaling analog training to larger models.
As illustrated in Table~\ref{tab:80states}, where we conducted experiments on ResNet-34 with a larger portion of the model converted into analog and using 80 states, accuracy collapses for TT-v1 and remains far below our method even when TT-v2 is given the same device with only three tiles. 
These results suggest that our residual multi-tile mechanism is necessary for scaling analog training to larger models, rather than relying on extremely high-precision device assumptions.

\section{Related Work \texorpdfstring{\cite{wu2024towards,wu2025analog}}{}}
\label{section:relation-with-wu2024}
This paper can be viewed as an extension of~\cite{wu2024towards, wu2025analog} to the multi-tile setting, specifically designed to address the quantization noise arising from the limited conductance states in analog in-memory training. We adopt the analog training dynamics model: 
\begin{align}
    W_{t+1} = W_t + \Delta W_t \odot F(W_t) - |\Delta W_t| \odot G(W_t)
\end{align}
and build upon the asymptotic error analysis of Analog SGD presented in Theorem~\ref{theorem_Convergence of_Analog_SGD}, which originates from the asymmetric non-idealities studied in~\cite{wu2024towards, wu2025analog}. In addition, we explicitly incorporate the \emph{limited states non-ideality}, introducing the quantization noise term \(\zeta_t\) into both the model \eqref{biased-update} and the convergence proof, thereby capturing the impact of discrete conductance levels on analog training.

\textbf{Challenge of analyzing the convergence with generic response functions.}
\cite{wu2024towards} investigates analog training under the \emph{asymmetric linear device (ALD)} model, where the pulse response factors are defined as:
\begin{align}
    q_+(w) = 1 - \frac{w - w^\diamond}{\tau_{\max}}, \quad
    q_-(w) = 1 + \frac{w - w^\diamond}{\tau_{\max}}
\end{align}
with \( w^\diamond \) denoting the symmetric conductance point such that \(q_+(w^\diamond) = q_-(w^\diamond)\)~\citep{rasch2023fast}. By Assumption~\ref{assumption:pulse-response-symmetry}, this point can be shifted to zero, i.e., \(w^\diamond=0\). In this case, the symmetric and asymmetric components reduce to \(F(w) = 1\) and \(G(w) = |w|/\tau_{\max}\), which simplifies the Analog SGD recursion \eqref{biased-update_sgd} into the structured form:
\begin{align} \label{eq:analog_lowerbound}
    W_{t+1} = W_{t} - \alpha_t\nabla f(W_{t}; \xi_{t})
    - \frac{\alpha_t}{\tau_{\max}} \left| \nabla f(W_{t}; \xi_{t}) \right| \odot W_{t}.
\end{align}
This recursion admits a special structure where the first term and the bias term can be combined, leading to:
\begin{align}
    W_{t+1} = \Bigl(1 - \frac{\alpha_t}{\tau_{\max}} \left| \nabla f(W_{t}; \xi_{t}) \right|\Bigr)\odot W_{t}
    -\alpha_t\nabla f(W_{t}; \xi_{t})
\end{align}
which is a weighted average of \(W_{t}\) and \(\nabla f(W_{t}; \xi_{t})\). From this perspective, the transfer operation can be interpreted as \emph{biased gradient descent}, with the linear response enabling a tractable analysis.
In contrast, \cite{wu2025analog} considers the more challenging general case where the response functions are not necessarily linear:
\begin{align}
     W_{t+1} = W_t -\alpha\nabla f(W_t;\xi_t) \odot F(W_t) -|\alpha\nabla f(W_t;\xi_t)| \odot G(W_t).
\end{align}
In this setting, the above decomposition is no longer viable, complicating the analysis. To address this difficulty, the authors leverage the Lipschitz continuity of the response functions and establish bounds that relate \(G(W_t)\) to \(W_t\), thereby recovering a weighted average interpretation of \(W_t\) even under general nonlinear responses. In our paper, we adopt the generic-response to establish the scalability of our algorithm.

\textbf{Challenge of analyzing the convergence with quantization noise.}
The quantization noise has statistical properties with zero mean and variance proportional to $\Theta(\alpha \Delta w_{\min})$, which after normalizing by the step size $\alpha$ leads to a residual error of order $\Delta w_{\min}$ that does not vanish as $T$ increases, thereby introducing an additional asymptotic error term in Analog SGD. As we will prove in Section~\ref{section:Analog_Stochastic_Gradient_Descent_Convergence}, the asymptotic error is bounded by
\begin{align}
     \frac{1}{T} \sum_{t=0}^{T-1} \mathbb{E}[\|W^* - W_t\|^2]
\leq  \mathcal{O}\!\Big(R_T\sqrt{ \tfrac{ 2(f(W_0) - f^*)\sigma^2}{ T}}\Big)
     + 4\sigma^2 S_T + R_T\Delta w_{\min}
\end{align}
while the work in~\cite{wu2025analog} only considers the error from asymmetric updates, which originates from the absolute stochastic gradient term $|\alpha \nabla f(W_t;\xi_t)|$ in~\eqref{biased-update_sgd} and leads to an asymptotic error of order $\sigma^2$:
\begin{align}
     \frac{1}{T} \sum_{t=0}^{T-1} \mathbb{E}[\|W^* - W_t\|^2]
\leq  \mathcal{O}\!\Big(\sqrt{ \tfrac{ 2(f(W_0) - f^*)\sigma^2}{ T}}\Big)
     + 4\sigma^2 S_T.
\end{align}

\textbf{Challenge of analyzing the convergence with multiple-tiles.}
From the bilevel optimization perspective, the convergence analysis in~\cite{wu2025analog} for the two-tile case amounts to solving the following two-sequence nested problem:
\begin{align}
    \label{problem:residual-learning}
    \argmin_{W^{(0)}\in\reals^D} \;\|P_1^*(W^{(0)}) \|^2, 
    \quad \text{s.t. }~
    P_1^*(W^{(0)})  \in \argmin_{P\in\reals^D} f(W^{(0)} + \gamma P).
\end{align}
In contrast, our multi-tile residual learning algorithm generalizes this formulation to a multi-sequence bilevel optimization problem as described in Section \ref{sec.3}:
\begin{align}
W^{(0)} &:= \arg\min_{U_{0}}   \| U_{0}  - P_0^*  \|^2,\quad  P_0^* := W^*, \notag \\
W^{(1)} &:= \arg\min_{U_{1}}   \| U_{1}  - P_1^*(\overline{W}^{(1)})  \|^2, 
\quad \text{s.t. } P_1^*(\overline{W}^{(1)}) := \arg\min_{P_1} f(\overline{W}^{(1)} + \gamma P_1), \notag\\
&\ldots \notag\\
W^{(N)} &:= \arg\min_{U_{N}}    \| U_{N} - P_N^*(\overline{W}^{(N)}) \|^2, 
\quad \text{s.t. } P_N^*(\overline{W}^{(N)})  := \arg\min_{P_N} f(\overline{W}^{(N)} + \gamma^N P_N).
\end{align}
The analysis of multi-sequence residual learning is substantially more challenging than the two-tile case. First, each tile must track a drifting optimum that recursively depends on the composite weights of all preceding tiles, creating a deeply nested dependency that complicates the convergence analysis. Second, the quantization error from different tiles are not isolated but accumulate and couple across layers, requiring global measures such as Lyapunov functions to capture their joint dynamics, meanwhile ensuring that the overall error decays as more tiles are introduced. 
As detailed in Section 
\ref{sec.3} and \ref{section:proof-TT-convergence-scvx}, we address these difficulties using a multi-timescale learning strategy, where each tile is updated in its own inner loop with sufficiently long iterations, while the preceding tiles are frozen until the current tile converges close enough to its quasi-stationary optimum. 

\section{Mapping Coefficient Setting}
In AIMC, each \emph{logical weight} \( W \) is physically represented by mapping the difference between two \emph{physical conductance} values: a main conductance \( C_{\mathrm{main}} \) and a reference conductance \( C_{\mathrm{ref}} \). Specifically, the mapping takes the form:
\begin{align}
\label{eq:mapping}
    W = \kappa C = \kappa (C_{\mathrm{main}} - C_{\mathrm{ref}})
\end{align}
where \( \kappa \) is a fixed scaling constant that determines the logical weight range based on the physical conductance range of the device.

This representation scheme allows the hardware to represent both positive and negative weights using non-negative conductance values, which are physically realizable.
Before proceeding with the analysis, we clarify a slight abuse of notation used in the main text. In our notation, we do not explicitly distinguish between the physical conductance \( C \) of the memristive crossbar array and the corresponding logical weight \( W \), which are related by a fixed mapping constant. However, it is important to note that the device-level response functions \( q_\pm(\cdot) \) as well as symmetric and asymmetric components  \( F(\cdot) \) and  \( G(\cdot) \) are defined over the conductance domain. In the following theoretical analysis, we will reintroduce this mapping explicitly when necessary to ensure mathematical correctness.
In fact, the conductance-update rule is derived using the same approach as the update rule presented in ~\eqref{biased-update} of the main text:
\begin{equation}
\label{eq: conductance_update}
  C_{t+1}
   = 
  C_t
  + \Delta C_t \odot F(C_t)
  - |\Delta C_t| \odot G(C_t)
  + \frac{\zeta_t}{\kappa}
\end{equation}
where $\zeta_t$ is a stochastic quantization noise term introduced by the finite weight resolution $\Delta w_{\min}$, with $\mathbb{E}[\zeta_{t}] = 0, 
    \operatorname{Var}[\zeta_{t}] = \Theta ( \alpha \cdot \Delta w_{\min})$. Here $\alpha$ is the learning rate. See details in Lemma \ref{lemma:pulse-update-error}. We rewrite \eqref{biased-update} in its exact form by multiply $\kappa$ on both sides of \eqref{eq: conductance_update}:
\begin{equation}
  W_{t+1}
   = 
  W_t
  + \Delta W_t \odot F \bigl(\frac{W_t}{\kappa}\bigr)
  - |\Delta W_t| \odot G \bigl(\frac{W_t}{\kappa}\bigr)
  +\zeta_t.
  \label{eq:biased-update-conductance}
\end{equation}
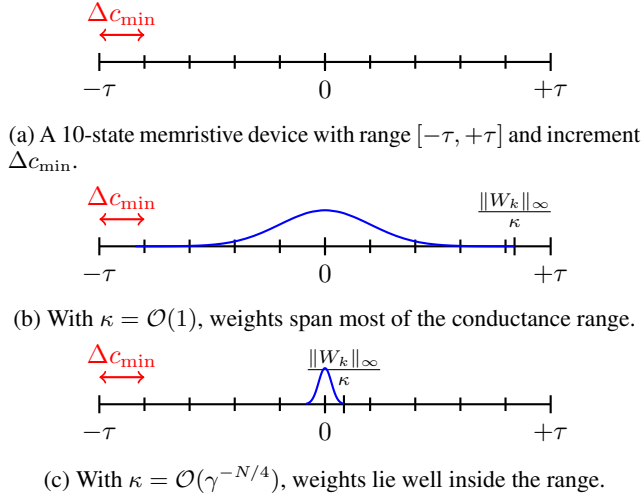
\begin{figure}[t]
    \centering
    \vspace{-2em}

    % ---------- (a) ----------
    \begin{subfigure}[b]{0.6\textwidth}
        \centering
        \begin{tikzpicture}[scale=0.6, thick]
            \draw (-5,0) -- (5,0);
            \draw (-5,0.2) -- (-5,-0.2);
            \draw (5,0.2) -- (5,-0.2);
            \draw (0,0.2) -- (0,-0.2);
            \foreach \i in {1,...,9} {
                \draw (-5 + \i*1,0.15) -- (-5 + \i*1,-0.15);
            }
            \draw[<->, red, thick] (-5,0.6) -- (-4,0.6);
            \node[above, red] at (-4.5,0.6) {$\Delta c_{\min}$};
            \node[below] at (-5,-0.2) {$-\tau$};
            \node[below] at (5,-0.2) {$+\tau$};
            \node[below] at (0,-0.2) {$0$};
        \end{tikzpicture}
        \caption{A 10-state memristive device with range $[-\tau,+\tau]$ and increment $\Delta c_{\min}$.}
    \end{subfigure}

    % ---------- (b) ----------
    \begin{subfigure}[b]{0.6\textwidth}
        \centering
        \begin{tikzpicture}[scale=0.6, thick]
            \draw (-5,0) -- (5,0);
            \draw (-5,0.2) -- (-5,-0.2);
            \draw (5,0.2) -- (5,-0.2);
            \draw (0,0.2) -- (0,-0.2);
            \draw (4.2,0.2) -- (4.2,-0.2);
            \foreach \i in {1,...,9} {
                \draw (-5 + \i*1,0.15) -- (-5 + \i*1,-0.15);
            }
            \draw[<->, red, thick] (-5,0.6) -- (-4,0.6);
            \node[above, red] at (-4.5,0.6) {$\Delta c_{\min}$};
            \node[below] at (-5,-0.2) {$-\tau$};
            \node[below] at (5,-0.2) {$+\tau$};
            \node[below] at (0,-0.2) {$0$};
            \node[above] at (4.2,0.2) {$\frac{\|W_{k}\|_{\infty}}{\kappa}$};
            \draw[domain=-4.2:4.2, smooth, variable=\x, blue, thick]
                plot ({\x}, {0.8*exp(-((\x)^2)/2)});
        \end{tikzpicture}
        \caption{With $\kappa=\mathcal{O}(1)$, weights span most of the conductance range.}
    \end{subfigure}

    % ---------- (c) ----------
    \begin{subfigure}[b]{0.6\textwidth}
        \centering
        \begin{tikzpicture}[scale=0.6, thick]
            \draw (-5,0) -- (5,0);
            \draw (-5,0.2) -- (-5,-0.2);
            \draw (5,0.2) -- (5,-0.2);
            \draw (0,0.2) -- (0,-0.2);
            \draw (0.42,0.2) -- (0.42,-0.2);
            \foreach \i in {1,...,9} {
                \draw (-5 + \i*1,0.15) -- (-5 + \i*1,-0.15);
            }
            \draw[<->, red, thick] (-5,0.6) -- (-4,0.6);
            \node[above, red] at (-4.5,0.6) {$\Delta c_{\min}$};
            \node[below] at (-5,-0.2) {$-\tau$};
            \node[below] at (5,-0.2) {$+\tau$};
            \node[below] at (0,-0.2) {$0$};
            \node[above] at (0.42,0.2) {$\frac{\|W_{k}\|_{\infty}}{\kappa}$};
            \draw[domain=-0.42:0.42, smooth, variable=\x, blue, thick]
                plot ({\x}, {0.8*exp(-((\x)^2)/0.04)});
        \end{tikzpicture}
        \caption{With $\kappa=\mathcal{O}(\gamma^{-N/4})$, weights lie well inside the range.}
    \end{subfigure}

    \vspace{-0.5em}
    \caption{Comparison of dynamic ranges and weight distributions under different $\kappa$.}
    \label{figure:mapping-illustration}
    \vspace{-2em}
\end{figure}
All theoretical guarantees in Section \ref{Sec:one-tile-asymptotic-error} and \ref{section:proof-TT-convergence-scvx} are proved using
\eqref{eq:biased-update-conductance}; the shorthand
\eqref{biased-update} is retained elsewhere to keep the notation compact.
Furthermore, although the logical weight \( W\in \mathbb{R}^{D\times D} \) in the main text, the proof part focuses on a vector-valued  $W \in \mathbb{R}^D$ while retaining the same uppercase notation to distinguish it from scalar elements \( w \). This simplification is justified because, in analog updates—whether during gradient updates or transfer updates—each column of the weight matrix behaves identically. Specifically, during gradient updates, all elements in a column are updated in parallel, and during transfer updates, updates are applied column-wise from one tile to another. Therefore, the vector setting captures the essential behavior without loss of generality.
Table~\ref{tab:notation_conductance_weight} summarizes the notations that appear in both the main text and the proofs, especially those where conductance and weight representations may be easily confused. 

In section \ref{section:proof-TT-convergence-scvx}, the mapping constant $\kappa$ is set as \(\mathcal{O}(\gamma^{-\frac{N}{4}}) \) to make Theorem~\ref{theorem:TT-convergence-scvx_long} hold. This setting is feasible because increasing the mapping coefficient 
$\kappa$ narrows the physical conductance range used to represent the same logical weight values. Figure \ref{figure:mapping-illustration} illustrates that  when the logical weights are concentrated well within the conductance boundaries, the impact of non-ideal device behaviors such as saturation, nonlinearities, or non-monotonicity is significantly reduced. This makes the training process more robust. 
\begin{table}[ht]
\centering
\begin{tabular}{ll}
\toprule
\textbf{Symbol} & \textbf{Meaning} \\
\midrule
\multicolumn{2}{l}{\textbf{Conductance domain}} \\
\midrule
$[-\tau, \tau]$ & Physical conductance range \\
$C_t$ & Matrix of physical conductance at step $t$\\
$c_t$ & Scalar conductance value of a single element at step $t$\\
$\Delta c_{\min}$ & Minimum conductance increment from one pulse at $c = 0$ \\
$F(c), G(c)$ & Symmetric and asymmetric components of conductance \\
$q_+(c), q_-(c)$ & Upward and downward pulse response factors at conductance $c$ \\
$U(c, \Delta c)$ & Approximate analog update over conductance \\
$U_p^{\text{BL}}(c, s)$ & Pulse-based update using bit-length pulses in conductance domain \\
\midrule
\multicolumn{2}{l}{\textbf{Weight domain}} \\
\midrule
$[\tau_{\min}, \tau_{\max}]$ & Logical weight range used in training \\
$W_t$ & Matrix of logical weights at step $t$\\
$w_t$ & Scalar logical weight of a single element at step $t$\\
$\|W_t\|_{\infty}$ & Maximum absolute value among all elements in $W_t$ \\
$W_{\max}$ & Upper bound of $\|W_{t}\|_\infty$ for all $t$, lies in $[\|W_{t}\|_\infty, \tau_{\max})$ \\
$\Delta w_{\min}$ & Minimum weight increment from one pulse at $w = 0$ \\
$F(w), G(w)$ & Symmetric and asymmetric components evaluated as $F(w/\kappa), G(w/\kappa)$ \\
$q_+(w), q_-(w)$ & Pulse response factors evaluated as $q_\pm(w) := q_\pm(w/\kappa)$ \\
$U(w, \Delta w)$ & Approximate analog update over logical weights \\
$U_p^{\text{BL}}(w, s)$ & Pulse-based update using bit length in weight domain\\
\midrule
\multicolumn{2}{l}{\textbf{Shared / Mapping}} \\
\midrule
$\kappa$ & Mapping constant: $W = \kappa C$ \\
\bottomrule
\end{tabular}
\vspace{1em}
\caption{Notations in the conductance and weight domains}
\label{tab:notation_conductance_weight}
\vspace{-2em}
\end{table}
\section{Notations}
\label{section:notations}
In this section, we define a series of notations that will be used in the analysis.

\textbf{Pseodo-inverse of diagonal matrix or vector.} For a given diagonal matrix $U\in\reals^{D\times D}$ with its $i$-th diagonal element $[U]_i$, we define the pseudo-inverse of a diagonal matrix $U$ as $U^\dag$, which is also a diagonal matrix with its $i$-th diagonal element: 
\begin{align}
    \label{definition:pseudo-inverse}
    [U^\dag]_i := \begin{cases}
        1 / [U]_i, ~~~&{[U]_i \ne 0}, \\
        0, ~~~&{[U]_i = 0}.
    \end{cases}
\end{align}
By definition, the pseudo-inverse satisfies $U U^\dag V = U^\dag U V$ for any diagonal matrix $U\in\reals^D$ and any matrix $V\in\reals^D$. With a slight abuse of notation, we also define the pseudo-inverse of a vector $W\in\reals^D$ as $W^\dag := \diag(W)^\dag$.

\textbf{Weighted norm.} For a given weight vector $S\in\reals^D_+$, the weighted norm $\|\cdot\|_S$ of vector $W\in\reals^D$ is defined by:
\begin{align}
    \|W\|_S := \sqrt{\sum_{d=1}^D [S]_d[W]_d^2}
    = W^\top \diag(S) W
\end{align}
where $\diag(S)\in\reals^{D\times D}_+$ rearranges the vector $S$ into a diagonal matrix.
\begin{lemma}
    \label{lemma:properties-weighted-norm}
    $\|W\|_S$ has the following properties: (1) $\|W\|_S=\|W\odot\sqrt{S}\|$; (2) $\|W\|_S\le\|W\| \sqrt{\|S\|_{\infty}}$; (3) $\|W\|_S\ge\|W\| \sqrt{\min\{[S]_i : i\in\ccalI\}}$.
\end{lemma}

\section{Useful Lemmas}

\subsection{Lemma \ref{lemma:pulse-update-error}: Pulse update error}
\label{Section:Proof_of_Lemma pulse-update-error}
\LemmaTTConvergenceScvx*
\begin{proof}[Proof of Lemma \ref{lemma:pulse-update-error}]
Each weight update \(\Delta w_{ij}\) is the sum of \(BL\) independent Bernoulli trials , with $BL$ large enough to satisfy $\frac{\alpha |x_i \delta_j|}{BL \cdot \Delta w_{\min}} \leq 1$:
\begin{align}
    \Delta w_{ij} = \sum_{t=1}^{BL} Z_t
\end{align}
where:
\begin{align}
    Z_t =
\begin{cases}
\Delta w_{\min} \cdot \operatorname{sign}(x_i \delta_j), & \text{with probability } p := \frac{|\alpha x_i \delta_j|}{BL \cdot \Delta w_{\min}}, \\
0, & \text{with probability } 1 - p.
\end{cases}
\end{align}
Then:
\begin{align}
    \mathbb{E}[Z_t] = \Delta w_{\min} \cdot \operatorname{sign}(x_i \delta_j) \cdot p = \frac{\alpha x_i \delta_j}{BL}, \quad
\mathbb{E}[Z_t^2] = \Delta w_{\min}^2 \cdot p.
\end{align}
The equality holds for $\operatorname{sign}(x_i \delta_j) \cdot |x_i \delta_j| = x_i \delta_j$. So the variance of a single trial is:
\begin{align}
    \operatorname{Var}[Z_t] = \mathbb{E}[Z_t^2] - \mathbb{E}[Z_t]^2 = \Delta w_{\min}^2 p (1 - p).
\end{align}
Then summing over \(BL\) trials:
\begin{align}
    \mathbb{E}[\Delta w_{ij}] = BL \cdot \mathbb{E}[Z_t]  = \alpha x_i \delta_j, \quad
\operatorname{Var}[\Delta w_{ij}] = BL \cdot \Delta w_{\min}^2 \cdot p \cdot (1 - p).
\end{align}
Thus, 
\begin{align}
    \mathbb{E}[\zeta_{ij}] = \mathbb{E}[\Delta w_{ij}]- \alpha x_i \delta_j=0.
\end{align}
Moreover, substituting \(p = \dfrac{|\alpha x_i \delta_j|}{BL \cdot \Delta w_{\min}}\) into $\operatorname{Var}[\Delta w_{ij}]$, we get:
\begin{align}
    \operatorname{Var}[\Delta w_{ij}] = BL \cdot \Delta w_{\min}^2 \cdot \frac{|\alpha x_i \delta_j|}{BL \cdot \Delta w_{\min}} \cdot \left(1 - \frac{|\alpha x_i \delta_j|}{BL \cdot \Delta w_{\min}} \right).
\end{align}
Thus,
\begin{align}
    \operatorname{Var}[\zeta_{ij}] = \operatorname{Var}[\Delta w_{ij}] = \alpha |x_i \delta_j| \cdot \Delta w_{\min} \cdot \left(1 - \frac{\alpha |x_i \delta_j|}{BL \cdot \Delta w_{\min}} \right) = \Theta ( \alpha \cdot \Delta w_{\min}).
\end{align}
\end{proof}

\subsection{Lemma \ref{lemma:lip-analog-update}: Lipschitz continuity of analog update}
\begin{lemma}
    \label{lemma:lip-analog-update}
    Under Assumption \ref{assumption:pulse-response-symmetry}, the analog increment defined in \eqref{biased-update} is Lipschitz continuous with respect to $\Delta W$ in terms of any weighted norm $\|\cdot\|_S$, i.e., for any $W, \Delta W, \Delta W'\in\reals^D$ and $S\in\reals^{D}_+$, it holds
    \begin{align}
        &\|\Delta W \odot F(W) -|\Delta W| \odot G(W)-(\Delta W' \odot F(W) -|\Delta W'| \odot G(W))\|_S
        \\
        \le&\ F_{\max}\|\Delta W-\Delta W'\|_S.
        \nonumber
    \end{align}
\end{lemma}
The proof of Lemma \ref{lemma:lip-analog-update} can be found in \citep [ Section C]{wu2025analog}.
\begin{lemma}
    \label{lemma:F-G-property}
    Under Assumption \ref{assumption:pulse-response-symmetry}, the following statements about the response factors are valid; 
    (1) the symmetric part $F(\cdot)$ is upper bounded by a constant $F_{\max} > 0$, i.e. $F(w)\le F_{\max},  \forall w\in{\mathbb R}$;
    (2) The following inequality holds
    \begin{align}
    \label{constraint:response}
        -F(w) \le G(w) \le F(w)
 \end{align}
    where $G(w)=-F(w)$ and $G(w)=F(w)$ hold only when $w=\tau_{\min}$ and $w=\tau_{\max}$, respectively. 
\end{lemma}
\subsection{Lemma \ref{lemma:element-wise-product-error}: Element-wise product error}
\begin{lemma}
    \label{lemma:element-wise-product-error}
    Let \( U, V, Q \) be vectors indexed by \( \ccalI \). Then the following inequality holds
    \begin{align}
        \langle U, V \odot Q \rangle 
        \ge C_+ \langle U, V\rangle 
        - C_- \langle |U|, |V|\rangle
    \end{align}
    where the constant $C_+$ and $C_-$ are defined by
    \begin{align}
        C_+ :=& \frac{1}{2}\left(\max_{i\in\ccalI}\{[Q]_i\} + \min_{i\in\ccalI}\{[Q]_i\}\right), \\
        % ==================
        C_- :=& \frac{1}{2}\left(\max_{i\in\ccalI}\{[Q]_i\} - \min_{i\in\ccalI}\{[Q]_i\}\right).
    \end{align}
\end{lemma}
The proof of Lemma \ref{lemma:element-wise-product-error} can be found in \citep [ Section C]{wu2025analog}.
% \section{Proof of Theorem \ref{theorem:one-tile-asymptotic-error}:
% Lower Bound of Asymptotic Error}
\section{Proof of Analog Stochastic Gradient Descent Convergence}
\label{section:Analog_Stochastic_Gradient_Descent_Convergence}
\subsection{Convergence of Analog SGD}
In this section, we derive the convergence guarantee of \textit{Analog SGD} under the hardware-constrained update rule in  ~\eqref{eq:biased-update-conductance}.
\begin{theorem}[Convergence of Analog SGD, long version of Theorem \ref{theorem_Convergence of_Analog_SGD_short}]
\label{theorem_Convergence of_Analog_SGD}
Suppose Assumptions \ref{assumption:noise}, \ref{assumption:bounded-saturation}, \ref{assumption:pulse-response-symmetry} hold, if the learning rate is set as 
$\alpha =\mathcal{O}(\sqrt{ \frac{ 2(f(W_0) - f^*)}{\sigma^2 T}})$, then
it holds that:
\[
 \frac{1}{T} \sum_{t=0}^{T-1} \mathbb{E}[\|W^* - W_t\|^2]
\leq  \mathcal{O} \Bigg(R_T\sqrt{ \frac{ 2(f(W_0) - f^*)\sigma^2}{ T}}\Bigg)+4\sigma^2 S_T + R_T\Delta w_{\min}
\]
where $S_T$  and $R_T$ denote the amplification factor given by :
\begin{align}
  S_T := \frac{1}{T} \sum_{t=0}^{T-1} \frac{\|W_t\|^2_\infty / \tau_{\max}^2}{1 - \|W_t\|^2_\infty / \tau_{\max}^2}
, \qquad
R_T := \frac{1}{T} \sum_{t=0}^{T-1} \frac{2L}{1 - \|W_t\|^2_\infty / \tau_{\max}^2}.  
\end{align}
\end{theorem}
\begin{proof}[Proof of Theorem \ref{theorem_Convergence of_Analog_SGD}]
The $L$-smooth assumption (Assumption \ref{assumption:Lip}) implies that:
\begin{align}
\label{inequality:analog_sgd_converge}
\mathbb{E}_{\xi_t, \zeta_t}[f(W_{t+1})]
\le f(W_t) + \underbrace{\mathbb{E}_{\xi_t, \zeta_t}[\langle \nabla f(W_t), W_{t+1} - W_t \rangle]}_{(a)}
+ \underbrace{\frac{L}{2} \mathbb{E}_{\xi_t, \zeta_t}[\|W_{t+1} - W_t\|^2]}_{(b)}.
\end{align}
The term (a) in \eqref{inequality:analog_sgd_converge} can be bounded by Assumption \ref{assumption:noise} that $\mathbb{E}_{\xi_t, \zeta_t}[\nabla f(W_t;\xi_t)] = \nabla f(W_t)$ and $2 \langle U,V\rangle = \|U+V\|^2 - \|U\|^2 - \|V\|^2  $:
\begin{align}
\label{inequality:analog_sgd_converge_T2}
&\quad \mathbb{E}_{\xi_t, \zeta_t}[\langle \nabla f(W_t), W_{t+1} - W_t \rangle]\\
&= \alpha_t  \mathbb{E}_{\xi_t, \zeta_t}\left[\left\langle \nabla f(W_t) \odot \sqrt{F(\frac{W_t}{\kappa})}, \frac{W_{t+1} - W_t}{\alpha_t \sqrt{F(\frac{W_t}{\kappa})}} +  \sqrt{F(\frac{W_t}{\kappa})}\odot (\nabla  f(W_{t}; \xi_{t})-\nabla  f(W_{t}))-\frac{\zeta_t}{\alpha_t \sqrt{F(\frac{W_t}{\kappa})}}\right\rangle \right] \notag \\
&= -\frac{\alpha_t}{2} \left\| \sqrt{F(\frac{W_t}{\kappa})}\odot\nabla f(W_t)\right\|^2 \notag \\
&\quad  - \frac{1}{2\alpha_t} \mathbb{E}_{\xi_t, \zeta_t}\left[\left\| \frac{W_{t+1} - W_t}{ \sqrt{F(\frac{W_t}{\kappa})}} + \alpha_t  \sqrt{F(\frac{W_t}{\kappa})}\odot(\nabla  f(W_{t}; \xi_{t})-\nabla  f(W_{t}))-\frac{\zeta_t}{\sqrt{F(\frac{W_t}{\kappa})}}\right\|^2\right] \notag \\
&\quad + \frac{1}{2\alpha_t} \mathbb{E}_{\xi_t}\left[\left\|\frac{W_{t+1} - W_t}{ \sqrt{F(\frac{W_t}{\kappa})}} + \alpha_t  \sqrt{F(\frac{W_t}{\kappa})}\odot\nabla  f(W_{t}; \xi_{t})-\frac{\zeta_t}{\sqrt{F(\frac{W_t}{\kappa})}}\right\|^2\right].\nonumber
\end{align}
The second term in \eqref{inequality:analog_sgd_converge_T2} can be bounded by
\begin{align}
    \label{inequality:ASGD-convergence-1-2-2}
    & \frac{1}{2\alpha_t} \mathbb{E}_{\xi_t, \zeta_t}\left[\left\| \frac{W_{t+1} - W_t}{ \sqrt{F(\frac{W_t}{\kappa})}} + \alpha_t  \sqrt{F(\frac{W_t}{\kappa})}\odot(\nabla  f(W_{t}; \xi_{t})-\nabla  f(W_{t}))-\frac{\zeta_t}{\sqrt{F(\frac{W_t}{\kappa})}}\right\|^2\right]\\
    =&\ \frac{1}{2\alpha}\mathbb{E}_{\xi_t, \zeta_t}\lB\lnorm\frac{W_{t+1}-W_t+\alpha (\nabla f(W_t; \xi_t)-\nabla f(W_t))\odot F(\frac{W_t}{\kappa}) -\zeta_t}{\sqrt{F(\frac{W_t}{\kappa})}}\rnorm^2\rB
    \nonumber\\
    \ge&\ \frac{1}{2\alpha F_{\max}}\mathbb{E}_{\xi_t, \zeta_t}\lB \left\| W_{t+1}-W_t+\alpha (\nabla f(W_t; \xi_t)-\nabla f(W_t))\odot F(\frac{W_t}{\kappa}) -\zeta_t\right\|^2\rB
    \nonumber.
\end{align}
The last inequality holds by defining a constant $F_{\max}$ such that $\|F(W)\|_{\infty} \leq F_{\max}.$ The third term in \eqref{inequality:analog_sgd_converge_T2} can be bounded by variance decomposition and bounded variance assumption (Assumption \ref{assumption:noise})
\begin{align}
    \label{inequality:ASGD-convergence-1-2-3}
    &\ \frac{1}{2\alpha_t} \mathbb{E}_{\xi_t}\left[\left\|\frac{W_{t+1} - W_t}{ \sqrt{F(\frac{W_t}{\kappa})}} + \alpha_t  \sqrt{F(\frac{W_t}{\kappa})}\odot\nabla  f(W_{t}; \xi_{t})-\frac{\zeta_t}{\sqrt{F(\frac{W_t}{\kappa})}}\right\|^2\right]
    \\
    =&\ \frac{\alpha_t}{2}\mathbb{E}_{\xi_t}\lB\lnorm|\nabla f(W_t; \xi_t)|\odot\frac{G(\frac{W_t}{\kappa})}{\sqrt{F(\frac{W_t}{\kappa})}} \rnorm^2\rB
    \nonumber \\
    \le&\ \frac{\alpha_t}{2}\lnorm|\nabla f(W_t)|\odot\frac{G(\frac{W_t}{\kappa})}{\sqrt{F(\frac{W_t}{\kappa})}}\rnorm^2
    +\frac{\alpha_t\sigma^2}{2}\lnorm\frac{G(\frac{W_t}{\kappa})}{\sqrt{F(\frac{W_t}{\kappa})}}\rnorm^2_\infty.
    \nonumber 
\end{align}
Define the saturation vector the saturation vector \( H(W_ {t}) \in \mathbb{R}^{D} \) as:
\begin{align}
\label{eq:saturation_vector}
&\quad H({W_t}) := F\Bigl(\frac{W_t}{\kappa} \Bigr)^{ \odot2} - G\Bigl(\frac{W_t}{\kappa} \Bigr) ^{ \odot2} \\& = \Bigl(F\Bigl(\frac{W_t}{\kappa} \Bigr) + G\Bigl(\frac{W_t}{\kappa} \Bigr)\Bigr) \odot \Bigl(F\Bigl(\frac{W_t}{\kappa} \Bigr) - G\Bigl(\frac{W_t}{\kappa} \Bigr)\Big)\nonumber\\ & = q_+\Bigl(\frac{W_t}{\kappa} \Bigr) \odot q_-\Bigl(\frac{W_t}{\kappa} \Bigr) .\nonumber
\end{align}
Note that the first term in the \ac{RHS} of \eqref{inequality:analog_sgd_converge} and the second term in the \ac{RHS} of \eqref{inequality:ASGD-convergence-1-2-3} can be bounded by
\begin{align}
    &\ \label{inequality:ASGD-converge-linear-3}
    - \frac{\alpha_t}{2} \left\|\nabla f(W_t) \odot \sqrt{F(\frac{W_t}{\kappa})}\right\|^2 
    + \frac{\alpha_t}{2} \lnorm|\nabla f(\frac{W_t}{\kappa})| \odot \frac{G(\frac{W_t}{\kappa})}{\sqrt{F(\frac{W_t}{\kappa})}}\rnorm^2 \\
    =&\ -\frac{\alpha_t}{2}\sum_{d\in[D]} \left( [\nabla f(W_t)]_d^2 \left( [F(\frac{W_t}{\kappa})]_d-\frac{[G(\frac{W_t}{\kappa})]^2_d}{[F(\frac{W_t}{\kappa})]_d}\right)\right)
    \nonumber \\
    =&\ -\frac{\alpha_t}{2}\sum_{d\in[D]} \left( [\nabla f(W_t)]_d^2 \left( \frac{[F(\frac{W_t}{\kappa})]^2_d-[G(\frac{W_t}{\kappa})]^2_d}{[F(\frac{W_t}{\kappa})]_d}\right)\right)
    \nonumber \\
    \le&\ -\frac{\alpha_t}{2 F_{\max}}\sum_{d\in[D]} \left( [\nabla f(W_t)]_d^2 \left( [F(\frac{W_t}{\kappa})]_d^2-[G(\frac{W_t}{\kappa})]^2_d\right)\right)
    \nonumber\\
    =&\ -\frac{\alpha_t}{2 F_{\max}}\|\nabla f(W_t)\|^2_{H(W_t)}
    \le 0.
    \nonumber
\end{align}
Plugging \eqref{inequality:ASGD-convergence-1-2-2} to \eqref{inequality:ASGD-converge-linear-3} into \eqref{inequality:analog_sgd_converge_T2}, we bound the term (a) by
\begin{align}
    \label{inequality:ASGD-convergence-1-2-final}
    &\ \mathbb{E}_{\xi_t, \zeta_t}[\langle \nabla f(W_t), W_{t+1} - W_t \rangle] \\
    =&\  -\frac{\alpha_t}{2 F_{\max}}\|\nabla f(W_t)\|^2_{H(W_t)}
    +\frac{\alpha_t\sigma^2}{2}\lnorm\frac{G(\frac{W_t}{\kappa})}{\sqrt{F(\frac{W_t}{\kappa})}}\rnorm^2_\infty
    \nonumber \\
    &\ - \frac{1}{2\alpha F_{\max}}\mathbb{E}_{\xi_t, \zeta_t}\lB \left\| W_{t+1}-W_t+\alpha (\nabla f(W_t; \xi_t)-\nabla f(W_t))\odot F(\frac{W_t}{\kappa}) -\zeta_t \right\|^2\rB.
    \nonumber
\end{align}

The term (b) in \eqref{inequality:analog_sgd_converge} can be bounded by  $\mbE_{\xi_t}[\|\nabla f(W_t;\xi_t)-\nabla f(W_t)\|^2]\le\sigma^2$:
\begin{align}
\label{inequality:ASGD-convergence-1-3}
   & \quad\frac{L}{2} \mathbb{E}_{\xi_t, \zeta_t}[\|W_{t+1} - W_t\|^2]\\
&\le L \mathbb{E}_{\xi_t,\zeta_t}\left[\left\|W_{t+1} - W_t + \alpha_t (\nabla  f(W_{t}; \xi_{t})-\nabla  f(W_{t}))\odot F(\frac{W_t}{\kappa})-\zeta_t\right\|^2 \right]\bkeq
\quad+  L \mathbb{E}_{\xi_t}\left[\left\|\alpha_t(\nabla  f(W_{t}; \xi_{t})-\nabla  f(W_{t}))\odot F(\frac{W_t}{\kappa})+\zeta_t \right\|^2\right]\nonumber\\
& \le L \mathbb{E}_{\xi_t, \zeta_t}\left[\left\|W_{t+1} - W_t + \alpha_t (\nabla  f(W_{t}; \xi_{t})-\nabla  f(W_{t}))\odot F(\frac{W_t}{\kappa})-\zeta_t \right\|^2\right]\bkeq \quad
+ 2\alpha_t^2 F_{\max}^2 L \sigma^2 +2L F_{\max}^2\cdot \Theta(\alpha_t\Delta w_{\min})\nonumber.
\end{align}
Substituting \eqref{inequality:ASGD-convergence-1-2-final} and \eqref{inequality:ASGD-convergence-1-3} back into \eqref{inequality:analog_sgd_converge}, we have
\begin{align}
    \label{inequality:ASGD-convergence-2}
   &\quad \mathbb{E}_{\xi_t, \zeta_t}[f(W_{t+1})]\\
 &\ \le f(W_t) - \frac{\alpha_t}{2 F_{\max}} \|\nabla f(W_t)\|^2_{H(W_t)}
    +2\alpha_t^2LF_{\max}^2\sigma^2
    +\frac{\alpha_t\sigma^2}{2}\lnorm\frac{G(\frac{W_t}{\kappa})}{\sqrt{F(\frac{W_t}{\kappa})}} \rnorm^2_\infty + 2L F_{\max}^2\cdot \Theta(\alpha_t\Delta w_{\min})
 \nonumber    \\
    &\ - \frac{1}{F_{\max}}\left(\frac{1}{2\alpha_t}-LF_{\max}\right)\mathbb{E}_{\xi_t, \zeta_t}\left[\left\|W_{t+1}-W_t+\alpha (\nabla f(W_t; \xi_t)-\nabla f(W_t))\odot F(\frac{W_t}{\kappa})-\zeta_t\right\|^2\right]\nonumber \\
  &
    \leq f(W_t) - \frac{\alpha_t}{2 F_{\max}} \|\nabla f(W_t)\|^2_{H(W_t)}
    +2\alpha_t^2LF_{\max}^2\sigma^2
    +\frac{\alpha_t\sigma^2}{2}\lnorm\frac{G(\frac{W_t}{\kappa})}{\sqrt{F(\frac{W_t}{\kappa})}} \rnorm^2_\infty + 2L F_{\max}^2\cdot \Theta(\alpha_t\Delta w_{\min}).  \nonumber
\end{align}
The last inequality holds 
when $\alpha_t \le \frac{1}{2L F_{\max}}$.  Taking average over $t$, we get:
\begin{align}
\label{inequality:29}
   &\quad \frac{1}{T} \sum_{t=0}^{T-1} \mathbb{E}[\|\nabla f(W_t)\|^2_{H(W_t)}]
\\ & \le
\frac{2F_{\max}(f(W_0) - f(W_{T}))}{\alpha_t T}
+ 4L F_{\max}^3\alpha_t\sigma^2
+ \sigma^2F_{\max}  \frac{1}{T} \sum_{t=0}^{T-1} 
\lnorm\frac{G(\frac{W_t}{\kappa})}{\sqrt{F(\frac{W_t}{\kappa})}} \rnorm^2_\infty + 4L F_{\max}^3 \Theta(\Delta w_{\min}) \nonumber\\ &\le \mathcal{O} \Bigg(F_{\max}^2\sqrt{ \frac{ 8L(f(W_0) - f^*)\sigma^2}{ T}}\Bigg)+ \S_TF_{\max}\sigma^2 
+ 4L F_{\max}^3 \Theta(\Delta w_{\min}). \nonumber
\end{align}
The second inequality holds by choosing $\alpha =\mathcal{O}(\sqrt{ \frac{ f(W_0) - f^*}{2L F_{\max}^2 \sigma^2 T}})$. $S_T$ denotes the amplification factors given by:
\begin{align}
  S_T := \frac{1}{ T} \sum_{t=0}^{T-1} \lnorm\frac{G(\frac{W_t}{\kappa})}{\sqrt{F(\frac{W_t}{\kappa})}} \rnorm^2_\infty. 
\end{align}
Since $ \mu (W - W^*) \leq \nabla f(W)$, 
we multiply both sides by \(\mu\) and normalize by the constant \(H_{\min}\), defined as
$
H_{\min} := \min_{t \in {0,1,\dots,T-1}} H_{\min,t},
\text{where } H_{\min,t} \leq \|H(W_t)\|_{\infty}.$
This yields the upper bound for Analog SGD on general response factors:
\begin{align}
    & \quad \frac{1}{T} \sum_{t=0}^{T-1} \mathbb{E}[\|W^* - W_t\|^2]
\nonumber\\ &\leq \mathcal{O} \Bigg(\frac{F_{\max}^2}{H_{\min}}\sqrt{ \frac{ 8L(f(W_0) - f^*)\sigma^2}{ T}}\Bigg)+ \frac{S_TF_{\max}}{H_{\min}}\sigma^2 
+ 4L \frac{F_{\max}^3 }{H_{\min}}\Theta(\Delta w_{\min}) .
\end{align}
For concrete illustration, we now analyze the asymmetric linear device described in Section \ref{section:relation-with-wu2024} and \eqref{eq:analog_lowerbound}.
We can naturally get that $F_{\max}=1$, $H_{\min,t}=1-\frac{\|W_t\|_{\infty}^2}{\tau_{\max}^2} $. Rearranging \eqref{inequality:ASGD-convergence-2} as
\begin{align}
    \label{inequality:ASGD-convergence-2-ALD}
   &\quad \mathbb{E}_{\xi_t, \zeta_t}[f(W_{t+1})]\\
  &   \leq f(W_t) - \frac{\alpha_t}{2(1-\frac{\|W_t\|_{\infty}^2}{\tau_{\max}^2} )} \|\nabla f(W_t)\|^2
    +2\alpha_t^2L\sigma^2
    +\frac{\alpha_t\sigma^2\|W_t\|_{\infty}^2}{2\tau_{\max}^2}  + 2L \cdot \Theta(\alpha_t\Delta w_{\min}).  \nonumber
\end{align}
Divide both sides of \eqref{inequality:ASGD-convergence-2-ALD} by $1 -\|W_t\|^2_\infty / \tau_{\max}^2 > 0$ and average over $t$:
\begin{align}
\label{inequality:41}
    \frac{1}{T} \sum_{t=0}^{T-1} \mathbb{E}[\|\nabla f(W_t)\|^2]
    & \le
\frac{2R_T(f(W_0) - \mathbb{E}[f(W_{T})])}{\alpha_t T}
+ R_T\alpha_t\sigma^2
+ 4 \sigma^2 S_T+ R_T\Theta(\Delta w_{\min}) \nonumber \\
&\le
\frac{2R_T(f(W_0) - f^*)}{\alpha_t T}
+ R_T\alpha_t\sigma^2
+ 4 \sigma^2 S_T+ R_T\Theta(\Delta w_{\min}). 
\end{align}
Here, $S_T$ and $R_T$ denote the amplification factors given by :
\begin{align}
  S_T := \frac{1}{T} \sum_{t=0}^{T-1} \frac{\|W_t\|^2_\infty / \tau_{\max}^2}{1 - \|W_t\|^2_\infty / \tau_{\max}^2}
, \qquad
R_T := \frac{1}{T} \sum_{t=0}^{T-1} \frac{2L}{1 - \|W_t\|^2_\infty / \tau_{\max}^2}.  
\end{align}
Choosing $\alpha =\mathcal{O}(\sqrt{ \frac{ 2(f(W_0) - f^*)}{\sigma^2 T}})$, when $T \to \infty$, it satisfies that $\alpha \leq 
\frac{1}{2L}$ and \eqref{inequality:41} becomes:
\begin{align}
    \frac{1}{T} \sum_{t=0}^{T-1} \mathbb{E}[\|\nabla f(W_t)\|^2]
\leq  \mathcal{O} \left(R_T\sqrt{ \frac{ 2(f(W_0) - f^*)\sigma^2}{ T}}\right)+4\sigma^2 S_T + R_T\Delta w_{\min}.
\end{align}
Since $ \mu (W - W^*) \leq \nabla f(W)$, multiplying $\mu$ on both sides, we get:
\begin{align}
    \frac{1}{T} \sum_{t=0}^{T-1} \mathbb{E}[\|W^* - W_t\|^2]
\leq  \mathcal{O} \Bigg(R_T\sqrt{ \frac{ 2(f(W_0) - f^*)\sigma^2}{ T}}\Bigg)+4\sigma^2 S_T + R_T\Delta w_{\min}.
\end{align}

which completes the proof.
 \end{proof}

\subsection{Lower bound of asymptotic error for Analog SGD}
\label{Sec:one-tile-asymptotic-error}
Under the convergence of Analog SGD with hardware-constrained update rule in Theorem \ref{theorem_Convergence of_Analog_SGD}, we derive a lower bound on the asymptotic error floor that arises when training with a single analog tile on non-ideal AIMC hardware. 

\begin{restatable}[Asymptotic error bound under quantization, long version of Theorem \ref{theorem:one-tile-asymptotic-error}]{theorem}{ThmOneTileAsymptoticErrorlong}
\label{theorem1_long}
Suppose Assumptions \ref{assumption:noise}-\ref{assumption:pulse-response-symmetry} hold, $\alpha = \frac{1}{2L}$, there exists an instance where \textit{Analog SGD} generates a sequence $\{W_t\}_{t=0}^{T-1}$ such that: 
\[
\frac{1}{T} \sum_{t=0}^{T-1} \mathbb{E}[\|W^* - W_t\|^2] \geq  \Omega(\sigma^2S_T+ R_T\Delta w_{\min})
\]
\end{restatable}
This section provides the lower bound of Analog SGD on noisy asymmetric linear devices with limited conductance states under Assumptions~\ref{assumption:noise}--\ref{assumption:pulse-response-symmetry}. 
The proof is completed based on the following example from ~\citep[Section G.2]{wu2024towards}.

\textbf{(Example)} Consider an example where all the coordinates are identical, i.e., $W_t = w_t \mathbf{1}$ for some $w_t \in \mathbb{R}$ where $\mathbf{1} \in \mathbb{R}^D$ is the all-one vector. Define $W^* = w^* \mathbf{1}$ where $w^* \in \mathbb{R}$ is a constant scalar and a quadratic function $f(W) := \frac{L}{2} \|W - W^*\|^2$ whose minimum is $W^*$. Initialize the weight on $W_0 = W^*$. Furthermore, consider the sample noise $ \xi_t 
$ defined as $\xi_t = \varepsilon_t \mathbf{1}$, 
where random variable $\varepsilon_t \in \mathbb{R}$ is sampled by:
\begin{align}
\label{noise_distribution}
\varepsilon_t =
\begin{cases}
    \varepsilon_t^+ := \frac{\sigma}{\sqrt{D}} \sqrt{\frac{1 - p_t}{p_t}}, & \text{w.p. } p_t, \\
    \varepsilon_t^- := -\frac{\sigma}{\sqrt{D}} \sqrt{\frac{p_t}{1 - p_t}}, & \text{w.p. } 1 - p_t,
\end{cases}
\qquad \text{with } p_t = \frac{1}{2} \left(1 - \frac{w_t}{\tau_{\max}} \right).
\end{align}

As a reminder, it is always valid that $|w_t| = \|W_t\|_\infty \leq \tau_{\max}$ (see~\citep[Theorem 5]{wu2024towards}) and $0 \leq p_t \leq 1$. Therefore, the noise distribution is well-defined. Furthermore, without loss of generality, we assume $|w^*| \leq \frac{\tau_{\max}}{4}$ and $\sigma \leq \frac{\tau_{\max} L \sqrt{D}}{4 \sqrt{3}}$. We define the objective $f(w; \varepsilon_t) := \frac{L}{2} \left(w - w^* + \frac{\varepsilon_t}{L} \right)^2$, whose minimum is  $w^*_{\xi_t} = w^* - \frac{\varepsilon_t}{L}$.
The noise $\varepsilon_t$ satisfies ~\citep[Assumption 7]{wu2024towards}.
\vspace{-1em}
\begin{proof}[Proof of Theorem \ref{theorem1_long}]
Consider the example constructed above. Before deriving the lower bound, we demonstrate Assumption \ref{assumption:noise}--\ref{assumption:Lip} hold. It is obvious that $\nabla f(W) = L(W - W^*)$ satisfies Assumption \ref{assumption:Lip}. In addition, Assumption \ref{assumption:noise} could be verified by noticing \eqref{noise_distribution} implies $\mathbb{E}_{\xi_t}[\xi_t] = 0$ and $\mathbb{E}_{\xi_t}[\|\xi_t\|^2] \leq \sigma^2$. Assumption~\ref{assumption:bounded-saturation} is verified by  ~\citep[Lemma~2]{wu2024towards}. We now proceed to derive the lower bound. Manipulating the recursion \eqref{eq:analog_lowerbound}, we have the following result:
\begin{align}
\label{inequality:lowerbound_W}
& \quad \mathbb{E}_{\xi_t,\zeta_t}[\|W_{t+1} - W^*\|^2] \\
&= \mathbb{E}_{\xi_t,\zeta_t} [\| W_t  - \alpha_t\nabla f(W_{t}; \xi_{t})
    - \frac{\alpha_t}{\tau_{\max}} \left| \nabla f(W_{t}; \xi_{t}) \right| \odot W_{t} -W^*+  \zeta_t\|^2 ] \notag \\
    &= \mathbb{E}_{\xi_t} [\| W_t  - \alpha_t\nabla f(W_{t}; \xi_{t})
    - \frac{\alpha_t}{\tau_{\max}} \left| \nabla f(W_{t}; \xi_{t}) \right| \odot W_{t} -W^*\|^2 ] + \Theta(\alpha_t \Delta w_{\min})\notag \\
&=\mathbb{E}_{\xi_t} [\| W_t  - \alpha_t\nabla f(W_{t}; \xi_{t})
    -W^*\|^2 ] + \mathbb{E}_{\xi_t} [\|  \frac{\alpha_t}{\tau_{\max}} \left| \nabla f(W_{t}; \xi_{t}) \right| \odot W_{t}\|^2] \bkeq \quad-2\mathbb{E}_{\xi_t} [\langle W_t  - \alpha_t\nabla f(W_{t}; \xi_{t})
    -W^*, \frac{\alpha_t}{\tau_{\max}} \left| \nabla f(W_{t}; \xi_{t}) \right| \odot W_{t}\rangle ]+\Theta(\alpha_t \Delta w_{\min})\notag.
\end{align}
The second equality holds for $\|U+V\|^2 = \|U\|^2+\|V\|^2+2\langle U,V\rangle$ with $\mathbb{E}_{\zeta_t} [2\langle U,V\rangle]=\mathbb{E}_{\zeta_t} [\Theta(\zeta_t)] = 0$ here, and
$\mathbb{E}_{\zeta_t}[\zeta_t^2] = \Theta (\alpha_t \Delta w_{\min})$.
The first term on the right-hand side (RHS) of \eqref{inequality:lowerbound_W} can be bounded as:
\begin{align}
\label{eq:39}
    &\quad \mathbb{E}_{\xi_t} \left[\left\| W_t  - \alpha_t\nabla f(W_t; \xi_t)
    -W^* \right\|^2 \right] \\
   &= \| W_t -W^*\|^2 - 2\alpha_t  \mathbb{E}_{\xi_t} \left[\left\langle W_t 
    -W^*,\nabla f(W_t; \xi_t) \right\rangle \right] 
    + \alpha_t^2  \mathbb{E}_{\xi_t} \left[\left\| \nabla f(W_t; \xi_t) \right\|^2 \right] \nonumber \\
    &\geq (1 - 2\alpha_t L)\| W_t - W^*\|^2 + \alpha_t^2  \mathbb{E}_{\xi_t} \left[\left\| \nabla f(W_t; \xi_t) \right\|^2 \right]. \nonumber
\end{align}
Here the second equality uses the unbiasedness of the stochastic gradient, i.e., 
\(
\mathbb{E}_{\xi_t}[\nabla f(W_t; \xi_t)] = \nabla f(W_t),
\)
and the inequality follows from the Lipschitz condition 
\(
\langle W_t - W^*, \nabla f(W_t) \rangle \le L \|W_t - W^*\|^2.
\)
The second term in the RHS of \eqref{inequality:lowerbound_W} can be bounded as:
\begin{align}
\label{eq:40}
    &\quad \mathbb{E}_{\xi_t} [\|  \frac{\alpha_t}{\tau_{\max}} \left| \nabla f(W_{t}; \xi_{t}) \right| \odot W_{t}\|^2]
    \\
    &=\frac{\alpha_t^2 \|W_t\|^2_{\infty}}{\tau_{\max}^2}\mathbb{E}_{\xi_t} [\| | \nabla f(W_{t}; \xi_{t})| \|^2] = \frac{\alpha_t^2 \|W_t\|^2_{\infty}}{\tau_{\max}^2}\mathbb{E}_{\xi_t} [\|\nabla f(W_{t}; \xi_{t}) \|^2]\nonumber
\end{align}
where the first equality uses $W_t = w_t \mathbf{1}$. 
From \eqref{noise_distribution}, we have $\mathbb{E}_{\xi_t} [\|\nabla f(W_{t}; \xi_{t}) \|^2] = \|\nabla f(W_{t}) \|^2 + \mathbb{E}_{\xi_t} [\|\xi_{t}\|^2] = L^2\|W_{t} - W^*\|^2 +  \sigma^2$, substituting the equation into \eqref{eq:39} and \eqref{eq:40} yields:
\begin{align}
\label{eq:41}
     & \quad \mathbb{E}_{\xi_t} [\| W_t  - \alpha_t\nabla f(W_{t}; \xi_{t})
    -W^*\|^2 ]+ \mathbb{E}_{\xi_t} [\|  \frac{\alpha_t}{\tau_{\max}} \left| \nabla f(W_{t}; \xi_{t}) \right| \odot W_{t}\|^2]
    \\ & \geq
  \big  (1-2\alpha_t L+ \alpha_t^2 L^2(1+\frac{\|W_t\|^2_{\infty}}{\tau_{\max}^2})\big)\| W_t
    -W^*\|^2 +\alpha_t^2 \sigma^2(1+\frac{\|W_t\|^2_{\infty}}{\tau_{\max}^2}).\nonumber
\end{align}
The third term in the RHS of \eqref{inequality:lowerbound_W} can be bounded as:
\begin{align}
\label{eq:42}
  & \quad  -2\mathbb{E}_{\xi_t} [\langle W_t  - \alpha_t\nabla f(W_{t}; \xi_{t})
    -W^*, \frac{\alpha_t}{\tau_{\max}} \left| \nabla f(W_{t}; \xi_{t}) \right| \odot W_{t}\rangle ] \\
    &= -2\mathbb{E}_{\xi_t} [\langle W_t 
    -W^*, \frac{\alpha_t W_t}{\tau_{\max}} \odot \left| \nabla f(W_{t}; \xi_{t}) \right| \rangle ] + \frac{\alpha_t^2 W_t}{\tau_{\max}} \mathbb{E}_{\xi_t} [\langle \nabla f(W_{t}; \xi_{t}), \left| \nabla f(W_{t}; \xi_{t}) \right|\rangle ]\nonumber\\
    & = -2 \sum_{i=1}^D [\langle w_t
    -w^*, \frac{\alpha_t   w_t }{\tau_{\max}}(p_t ([\nabla f(W_t)]_i+\varepsilon_t^+)- (1-p_t) ([\nabla f(W_t)]_i+\varepsilon_t^-)) \rangle] \bkeq \quad +\frac{2\alpha_t^2 w_t }{\tau_{\max}} \sum_{i=1}^D  \big[ p_t ([\nabla f(W_t)]_i+\varepsilon_t^+)^2- (1-p_t) ([\nabla f(W_t)]_i+\varepsilon_t^-)^2 \big]\nonumber
\\    & \geq -\frac{2\alpha_t L\|W_{t}\|^2_{\infty}}{\tau_{\max}^2}\|  W_t
    -W^*\|^2+\frac{2\alpha_t^2 \|W_{t}\|^2_{\infty}}{\tau_{\max}^2}  (-L^2 \|W_t-W^*\|^2 + \sigma^2)\bkeq \quad
    -\frac{\alpha_t (w_t
    -w^*)  w_{t} \sigma \sqrt{D}}{\tau_{\max}} \sqrt{(1   -\|W_{t}\|^2_{\infty}/\tau_{\max}^2)}+ \frac{2\alpha_t^2 L (w_t
    -w^*)  w_{t} \sigma \sqrt{D}}{\tau_{\max}}\sqrt{(1   -\|W_{t}\|^2_{\infty}/\tau_{\max}^2)}\nonumber \\ &
 = -2(1+\alpha_t L)\frac{\alpha_t L\|W_{t}\|^2_{\infty}}{\tau_{\max}^2}\|  W_t
    -W^*\|^2+ \frac{2\alpha_t^2 \|W_{t}\|^2_{\infty}\sigma^2}{\tau_{\max}^2} \nonumber
\end{align}
where the first equation uses $W_t = w_t \mathbf{1}$, the second equality holds for ~\citep[Lemma~4]{wu2024towards}, which shows that $\nabla f(W_t)+\varepsilon_t^+\geq 0$ and $\nabla f(W_t)+\varepsilon_t^-\leq 0$, the last equation holds by setting $\alpha_t = \frac{1}{2L}$, the third equation holds by simplifying the second term as: 
\begin{align}
    & \quad  \sum_{i=1}^D [\langle w_t
    -w^*, \frac{\alpha_t   w_t }{\tau_{\max}} (p_t ([\nabla f(W_t)]_i+\varepsilon_t^+)- (1-p_t) ([\nabla f(W_t)]_i+\varepsilon_t^-)) \rangle] \\
    & = \sum_{i=1}^D [\langle w_t
    -w^*, \frac{\alpha_t   w_t }{\tau_{\max}}(p_t ([\nabla f(W_t)]_i+\frac{\sigma}{\sqrt{D}} \sqrt{\frac{1 - p_t}{p_t}})- (1-p_t) ([\nabla f(W_t)]_i-\frac{\sigma}{\sqrt{D}} \sqrt{\frac{p_t}{1 - p_t}})) \rangle] \nonumber
    \\ & = \sum_{i=1}^D [\langle w_t
    -w^*, \frac{\alpha_t   w_t  }{\tau_{\max}}(2p_t-1) [\nabla f(W_t)]_i+\frac{\alpha_t   w_{t} \sigma}{\tau_{\max}\sqrt{D}} \sqrt{(1 - p_t)p_t}\rangle ]\nonumber
     \\ & = \langle W_t
    -W^*, -\frac{\alpha_t   \|W_t\|^2_{\infty}  }{\tau_{\max}}\nabla f(W_t)\rangle +\frac{\alpha_t (w_t
    -w^*)  w_{t} \sigma \sqrt{D}}{2\tau_{\max}} \sqrt{(1   -\|W_{t}\|^2_{\infty}/\tau_{\max}^2)}\nonumber \\ & \leq  \frac{\alpha_t^2 L\|W_{t}\|^2_{\infty}}{\tau_{\max}^2}\|W_t-W^*\|^2 +\frac{\alpha_t (w_t
    -w^*)  w_{t} \sigma \sqrt{D}}{2\tau_{\max}} \sqrt{(1   -\|W_{t}\|^2_{\infty}/\tau_{\max}^2)} \nonumber.
\end{align}
The inequality holds for $\langle W_t - W^*,\ -\nabla f(W_t) \rangle 
= -\langle W_t - W^*,\ \nabla f(W_t) \rangle 
\le \|W_t - W^*\| \cdot \|\nabla f(W_t)\| 
\le L \|W_t - W^*\|^2$. Simplifying the second term as: 
\begin{align}
    &\quad  \sum_{i=1}^D  \big[ p_t ([\nabla f(W_t)]_i+\varepsilon_t^+)^2- (1-p_t) ([\nabla f(W_t)]_i+\varepsilon_t^-)^2 \big] \\
    & = 
    \sum_{i=1}^D  \big[ p_t \big([\nabla f(W_t)]_i^2+2[\nabla f(W_t)]_i\frac{\sigma}{\sqrt{D}} \sqrt{\frac{1 - p_t}{p_t}}+(\frac{\sigma}{\sqrt{D}} \sqrt{\frac{1 - p_t}{p_t}})^2 \big) \bkeq \quad - (1-p_t) \big([\nabla f(W_t)]_i^2+2[\nabla f(W_t)]_i(-\frac{\sigma}{\sqrt{D}} \sqrt{\frac{p_t}{1 - p_t}}) +(-\frac{\sigma}{\sqrt{D}} \sqrt{\frac{p_t}{1 - p_t}})^2 \big) \big]\nonumber \\
     & = \sum_{i=1}^D  \big[(1-2p_t) (-[\nabla f(W_t)]_i^2 + \frac{\sigma^2}{D})+ [\nabla f(W_t)]_i \frac{\sigma}{\sqrt{D}}\sqrt{(1   -\|W_{t}\|^2_{\infty}/\tau_{\max}^2)} \big] \nonumber \\
     & = -\frac{L^2 \|W_{t}\|^2_{\infty}}{\tau_{\max}}\|W_t-W^*\|^2 + \frac{\|W_t\|_{\infty} \sigma^2}{\tau_{\max} }+L\sigma\sqrt{D}(w_t-w^*)\sqrt{(1   -\|W_{t}\|^2_{\infty}/\tau_{\max}^2)}. \nonumber
\end{align} 
Substituting \eqref{eq:41} and \eqref{eq:42} into \eqref{inequality:lowerbound_W}, we get:
\begin{align}
\label{eq:43}
  & \quad  \mathbb{E}_{\xi_t,\zeta_t}[\|W_{t+1} - W^*\|^2]\\
  & \geq   \big  (1-2\alpha_t L+ \alpha_t^2 L^2(1+\frac{\|W_t\|^2_{\infty}}{\tau_{\max}^2})-2(1+\alpha_t L)\frac{\alpha_t L\|W_{t}\|^2_{\infty}}{\tau_{\max}^2}\big)  \| W_t
    -W^*\|^2 \bkeq \quad+\alpha_t^2 \sigma^2 (1+\frac{\|W_t\|^2_{\infty}}{\tau_{\max}^2})+ \frac{2\alpha_t^2 \|W_{t}\|^2_{\infty}\sigma^2}{\tau_{\max}^2}+ \Theta(\alpha_t \Delta w_{\min}).\nonumber \\
  & =    \big  (1-2\alpha_t L(1-\frac{\alpha_t L}{2})(1-\frac{\|W_{t}\|^2_{\infty}}{\tau_{\max}^2})\big)  \| W_t
    -W^*\|^2+ \alpha_t^2 \sigma^2(1+\frac{3\|W_t\|^2_{\infty}}{\tau_{\max}^2})+ \Theta(\alpha_t \Delta w_{\min}).\nonumber
\end{align}
Rearranging \eqref{eq:43}
as:
\begin{align}
\label{eq:44}
  & \quad \| W_t
    -W^*\|^2   
\\ & \geq \frac{\| W_t
    -W^*\|^2 -\mathbb{E}_{\xi_t,\zeta_t}[\|W_{t+1} - W^*\|^2]}{2\alpha_t L(1-\alpha_t L/2)(1-\|W_{t}\|^2_{\infty}/\tau_{\max}^2)} 
+\frac{(1+3\|W_t\|^2_{\infty}/\tau_{\max}^2)\alpha_t\sigma^2}{2L(1-\alpha_t L/2)(1-\|W_{t}\|^2_{\infty}/\tau_{\max}^2)}\bkeq \quad  + \frac{\Theta(\Delta w_{\min})}{2 L(1-\alpha_t L/2)(1-\|W_{t}\|^2_{\infty}/\tau_{\max}^2)}
\nonumber \\ &
\geq \frac{\|W_t\|^2_{\infty}/\tau_{\max}^2 \sigma^2}{L^2(1- \|W_t\|^2_{\infty}/\tau_{\max}^2)} + \frac{2\Theta(\Delta w_{\min})}{3L(1- \|W_t\|^2_{\infty}/\tau_{\max}^2)} .\nonumber
\end{align}
The inequality holds since $\alpha_t=\frac{1}{2L}$, and $\| W_t
    -W^*\|^2 \geq \mathbb{E}_{\xi_t,\zeta_t}[\|W_{t+1} - W^*\|^2]$ from ~\citep[Theorem 8]{wu2024towards}.
Taking the expectation over all $\xi_t,\zeta_t$ and take the average of \eqref{eq:44} for $t$ from $0$ to $T-1$, we obtain:
\begin{align}
\label{square_norm_of_the_gradient}
 &\quad \frac{1}{T} \sum_{t=0}^{T-1} \mathbb{E}[\| W_t
    -W^*\|^2] 
\\
&\geq \sigma^2 \cdot \frac{1}{T} \sum_{t=0}^{T-1} \frac{\|W_t\|^2_\infty / \tau_{\max}^2}{L^2(1 - \|W_t\|^2_\infty / \tau_{\max}^2)} + \frac{1}{T} \sum_{t=0}^{T-1} \frac{ 2\Theta(\Delta w_{\min})}{3L(1-\|W_t\|^2_\infty / \tau_{\max}^2)} \nonumber\\ &
=  \Omega( \sigma^2S_T+ R_T\Delta w_{\min}).\nonumber 
\end{align}
The proof of Theorem \ref{theorem:one-tile-asymptotic-error} is thus completed.
\end{proof}
\section{Proof of Theorem \ref{theorem:TT-convergence-scvx_long} and Corollary \ref{corollary:Convergence_rate}:  Convergence of Residual Learning}
\label{section:proof-TT-convergence-scvx}
This section provides the convergence of residual learning algorithm under Assumptions~\ref{assumption:noise}--\ref{assumption:pulse-response-symmetry}. To formalize the analysis, we first clarify the use of tile-specific update indices.
In the main text, we define each gradient update as one training step, denoted by the global counter \( t \). Since each tile \( W^{(n)} \) is updated only once every \( T_{n+1} \) updates of tile \( W^{(n+1)} \). As a result, \( W^{(n)} \) is not updated at every global step, exhibiting an inherently asynchronous update pattern. We introduce a local update counter \( t_n \) for each tile \( W^{(n)} \), which tracks the number of its own update steps. These local counters are related to the global counter \( t \) by the following approximate relation:
\[
t_n = \left\lfloor \frac{t + 1}{\prod_{n'=n+1}^{N} T_{n'}} \right\rfloor.
\]
As shown in Figure~\ref{fig:tile-update-schedule}, the update schedule exhibits a nested timing hierarchy where inner tiles are updated less frequently.
\begin{figure}[ht]
\centering
\vspace{-1em}
\begin{tikzpicture}
\matrix[matrix of nodes,
        nodes in empty cells,
        nodes={minimum width=0.9em, text height=0.9ex, text depth=0.15ex, font=\scriptsize},
        column sep=1.4em,
        row sep=0.2em] (m) {

$\boldsymbol{t_{N}}$ & $\boldsymbol{t_{N-1}}$ & $\boldsymbol{t_{N-2}}$ & $\boldsymbol{t_{N-3}}$ & $\cdots$ & $\boldsymbol{t_{0}}$ \\
0 & 0 & 0 & 0 & $\cdots$ & 0 \\
1 & 1 & 0 & 0 & $\cdots$ & 0\\
2 & 1 & 0 & 0 & $\cdots$ & 0\\
3 & 2 & 1 & 0 & $\cdots$ & 0\\
4 & 2 & 1 & 0 & $\cdots$ & 0\\
5 & 3 & 1 & 0 & $\cdots$ & 0 \\
6 & 3 & 1 & 0 & $\cdots$ & 0 \\
7 & 4 & 2 & 1 & $\cdots$ & 0\\
};
\draw[->, thick] 
  ([xshift=-1cm]m-2-1.north) -- 
  ([xshift=-1cm]m-9-1.south)
  node[midway, xshift=-0.5em, rotate=90, font=\footnotesize]{time direction};

% Arrows for transfer updates
\draw[->, thick] (m-3-1) -- (m-3-2); 
\draw[->, thick] (m-5-1) -- (m-5-2); 
\draw[->, thick] (m-7-1) -- (m-7-2); 
\draw[->, thick] (m-9-1) -- (m-9-2); 
\draw[->, thick] (m-5-2) -- (m-5-3); 
\draw[->, thick] (m-9-2) -- (m-9-3); 
\draw[->, thick] (m-9-3) -- (m-9-4);

\end{tikzpicture}
\vspace{-0.5em}
\caption{Illustration of local index evolution \( t_n \) in the cascading residual learning phase, assuming all inner loop length \( T_n = 2 \). Arrows indicate transfer updates from \( W^{(n+1)} \) to \( W^{(n)} \).}
\label{fig:tile-update-schedule}
\vspace{-2em}
\end{figure}

\subsection{Main proof}
\begin{restatable}[Convergence of residual learning, long version of Theorem \ref{theorem:TT-convergence-scvx_long}]{theorem}{ThmTTConvergenceScvxlong}
\label{theorem:TT-convergence-scvx}
Suppose Assumptions~\ref{assumption:noise}--\ref{assumption:pulse-response-symmetry} hold. Let the scaling parameter satisfy \(\gamma \in (0, H_{\min}/
\sqrt{6}F_{\max}^{2}]\), and set the mapping constant as
\( 
\kappa=(\sigma L_G W_{\max})^{\frac{1}{2}}(\gamma^{N} \Delta w_{\min})^{-
\frac{1}{4}}.
\)
    For all \(n \in \{0,\ldots, N-1\}\), set the learning rate \(\beta = \Theta(\gamma^2)\), and choose the inner loop length
\(
T_{n} \geq \Theta \bigl(\gamma^{-1}\bigr),
\) except for $T_0= \Theta(1)$.
When \(n = N\), set the learning rate \(\alpha = \Theta(1)\), and choose
\(
T_{N} \geq \Theta \bigl(\gamma^{-N}\bigr).
\)
Define the Lyapunov sequence as:
\[
 \mathbb{J}_{k} := \sum_{n=0}^N \|W^{(n)}_{t_n+kT_n-1}-P^*_{n}(\overline{W}^{(n)}_{t_{n-1}+k})\|^2.
\]
Since \(t = \prod_{n=0}^N T_n k= \mathcal{O}(\gamma^{-2N}) k\) is the total number of gradient evaluations, with $\rho \in (0,\frac{2}{3})$,
the Lyapunov function $\mathbb{J}_k$ 
satisfies:
\[
 \mathbb{E}[\mathbb{J}_k] \le \mathcal{O}((1 - \rho)^{\gamma^{2N} t} ) \mathbb{E}[\mathbb{J}_0] + \Theta\big(\gamma^{-\frac{4N}{3}} (\sigma \Delta w_{\min})^{\frac{2}{3}}\big).  \nonumber
\]
\end{restatable}
\begin{proof}[Proof of Theorem \ref{theorem:TT-convergence-scvx}]
We begin by presenting two lemmas essential for establishing the convergence proof.

\textbf{One inner loop contraction.}

Lemma~\ref{lemma:TT-WN-descent} establishes that tile \( W^{(N)} \) undergoes a descent in expected distance to its local stationary point \( P^*_{N}(\overline{W}^{(N)}) \) after one inner loop with \( T_N \) updates. The update dynamic is defined as:
\begin{align}
\label{eq:TT-WN-descent}
    W_{t+1}^{(N)} =
    W_{t}^{(N)} - \alpha \nabla f(\overline{W}_{t}; \xi_{t})\odot F\Bigl( \frac{W_{t}^{(N)}}{\kappa} \Bigr) -|\alpha \nabla f(\overline{W}_{t}; \xi_{t})| \odot G \Bigl( \frac{W_{t}^{(N)}}{\kappa} \Bigr) + \zeta_{t}. 
\end{align}
\vspace{-1em}
\begin{restatable}[Descent lemma of the main sequence $W^{(N)}$, long version of Lemma \ref{lemma:2}]{lemma}{LemmaTTWnDescentScvxWS}\label{lemma:TT-WN-descent}
        Suppose Assumptions~\ref{assumption:noise} --\ref{assumption:pulse-response-symmetry} hold, the learning rate satisfies $\alpha \leq \frac{C_{k, +}}{4 \gamma^N (\mu+L) F_{\max}^2}$, the mapping constant is set as \( \kappa=(\sigma L_G W_{\max})^{\frac{1}{2}}(\gamma^{N} \Delta w_{\min})^{-
\frac{1}{4}} \). Denote $\mathbb{E}_{\boldsymbol{\xi_N},\boldsymbol{\zeta_N}} := \mathbb{E}_{\xi_{t:t+T_N-1},\zeta_{t:t+T_N-1}} $.
    It holds that:
\begin{align}
\label{inequality:TT_convergence-scvx-C3N-final}
& \mathbb{E}_{\boldsymbol{\xi_N}, \boldsymbol{\zeta_N}}\left[\| W_{t+(k+1)T_N-1}^{(N)} - P^*_{N}(\overline{W}^{(N)}_{t_{N-1}+k})\|^2\right] \\ \le &
 \left(1 - \frac{\alpha\mu L \gamma^N}{4(\mu + L)} \right)^{T_N} \| W_{t+kT_N}^{(N)}-P^*_{N}(\overline{W}^{(N)}_{t_{N-1}+k})\|^2 + 
        \frac{8(\mu + L)\alpha}{\gamma^N \mu L} F_{\max}^2 \sigma^2 +\gamma^{-\frac{4N}{3}}  \Theta ((\sigma \Delta w_{\min})^\frac{2}{3}).\nonumber
\end{align}
\end{restatable}
Lemma~\ref{lemma:TT-n-descent} establishes that a single update of tile \( W^{(n)} \) leads to a descent in its expected distance to the local stationary point \( P^*_{n}(\overline{W}^{(n)}) \) after one inner loop with  \( T_n \) updates. The update dynamic is defined as:
\begin{align}
\label{eq:TT-update-n-compact-II}
    W_{t_n+1}^{(n)} =
    W_{t_n}^{(n)} + \beta W^{(n+1)}_{t_{n+1} + T_{n+1} - 1} \odot F\Bigl( \frac{W_{t_n}^{(n)}}{\kappa} \Bigr) -|\beta W^{(n+1)}_{t_{n+1} + T_{n+1} - 1}| \odot G\Bigl( \frac{W_{t_n}^{(n)}}{\kappa} \Bigr) + \zeta_{t_n}. 
\end{align}

\begin{restatable}[Descent lemma of lower level sequences $W^{(n)}$, long version of Lemma \ref{lemma:3}]{lemma}{LemmaTTWNDescentScvxCRL}
    \label{lemma:TT-n-descent}
    Following the same assumptions of Lemma \ref{lemma:TT-WN-descent}, for \(n \in \{0,\ldots, N-1\}\), the learning rate satisfies that $\beta \leq \frac{F_{\max}^3 \gamma}{3H_{\min}}$. 
    Denote
    % $\mathcal{T}_n := \prod_{n' =n+1}^{N-1} T_{n'}$, 
    $ \mathbb{E}_{\boldsymbol{\zeta_n}} := \mathbb{E}_{\zeta_{t_n+kT_n:t_n+(k+1)T_n-1}} $. 
    It holds that:
\begin{align}
\label{inequality:TT-convergence-scvx-C3n-final}
   & \quad \mathbb{E}_{\boldsymbol{\zeta_n}}[\|W_{ t_{n}+(k+1)T_n-1}^{(n)}-P^*_{n}(\overline{W}^{(n)}_{t_{n-1}+k})\|^2]
  \\ &  \leq \left( 1-\frac{\beta H_{\min}}{2\gamma  F_{\max}}\right)^{T_n}\|W_ {t_{n}+kT_n}^{(n)}-  P^*_{n}(\overline{W}^{(n)}_{t_{n-1}+k})\|^2\bkeq \quad  + \frac{6 F_{\max}^4  \gamma^2}{H_{\min}^2}\| W_{t_{n+1}+(k+1)T_{n+1}-1}^{(n+1)}- P^*_{n+1}(\overline{W}^{(n+1)}_{t_{n}+k})\|^2 +\frac{2F_{\max}\Theta (\Delta w_{\min}) }{H_{\min}}\bkeq \quad-(\frac{2\gamma^2}{H_{\min}} -\frac{6\beta\gamma F_{\max}}{H_{\min}})\Big \|P^*_{n+1}(\overline{W}^{(n+1)}_{t_{n}+k})\odot F\Bigl(\frac{W_ {t_{n}}^{(n)}}{\kappa} \Bigr) - |P^*_{n+1}(\overline{W}^{(n+1)}_{t_{n}+k})|\odot G\Bigl(\frac{W_ {t_{n}}^{(n)}}{\kappa} \Bigr) \Big\|^2. \nonumber 
\end{align}
\end{restatable}
The proof of Lemma  \ref{lemma:TT-WN-descent} and \ref{lemma:TT-n-descent} are deferred to Section \ref{section:proof-lemma:TT-WN-descent} and \ref{section:proof-lemma:TT-n-descent}. We then proceed to prove the convergence of the algorithm with the result of Lemma  \ref{lemma:TT-WN-descent} and \ref{lemma:TT-n-descent} . Define a Lyapunov function as:
\begin{align}
     \mathbb{J}_{k} := \sum_{n=0}^N \|W^{(n)}_{t_n+kT_n-1}-P^*_{n}(\overline{W}^{(n)}_{t_{n-1}+k})\|^2.
\end{align}
\textbf{Bounding the drifting optimality gap.}

To derive the recursion of the Lyapunov function, we require an additional inequality to characterize the \textit{drift optimality}, when $n \in [1,N]$.  
Observe that between one time step increment on $t_{n-1}$, only the component \(W^{(n-1)}\) of the stationary point \(P^*_{n}(\overline{W}^{(n)}) = \gamma^{-n}(W^* - \overline{W}^{(n)})\) is updated due to the structure of the inner-loop algorithm we obtain:
\begin{align}
    \label{inequality:TT-convergence-scvx-Pstar-T3}
    &\ \mathbb{E}_{\zeta_{t_n}}[ \|P^*_{n}(\overline{W}^{(n)}_{t_{n-1}+k+1})- P^*_{n}(\overline{W}^{(n)}_{t_{n-1}+k})\|^2]\\
    = &\frac{1}{\gamma}\mathbb{E}_{\zeta_{t_n}}\Big[ \Big\|\beta \Bigl(W_{t_{n}+kT_n-1}^{(n)}\odot F\Bigl(\frac{W_ {t_{n-1}+k}^{(n-1)}}{\kappa} \Bigr) - |W_{t_{n}+kT_n-1}^{(n)}|\odot G\Bigl(\frac{W_ {t_{n-1}+k}^{(n-1)}}{\kappa} \Bigr)\Bigr)+ \zeta_{t_n}\Big\|^2 \Big] \nonumber
    \\ 
   \leq & \frac{3\beta^2}{\gamma} \Big\|P^*_{n}(\overline{W}^{(n)}_{t_{n-1}+k})\odot F\Bigl(\frac{W_ {t_{n-1}+k}^{(n-1)}}{\kappa}\Bigr) - |P^*_{n}(\overline{W}^{(n)}_{t_{n-1}+k})|\odot G\Bigl(\frac{W_ {t_{n-1}+k}^{(n-1)}}{\kappa} \Bigr)\Big\|^2 +\frac{3\beta^2}{\gamma}
\Big\|W_{t_{n}+kT_n-1}^{(n)} \bkeq\odot F\Bigl(\frac{W_ {t_{n-1}+k}^{(n-1)}}{\kappa}\Bigr) - |W_{t_{n}+kT_n-1}^{(n)} |\odot G\Bigl(\frac{W_ {t_{n-1}+k}^{(n-1)}}{\kappa}\Bigr)-\Bigl(P^*_{n}(\overline{W}^{(n)}_{t_{n-1}+k})\odot F\Bigl(\frac{W_ {t_{n-1}+k}^{(n-1)}}{\kappa} \Bigr)\bkeq - |P^*_{n}(\overline{W}^{(n)}_{t_{n-1}+k})|\odot G\Bigl(\frac{W_ {t_{n-1}+k}^{(n-1)}}{\kappa} \Bigr)\Bigr)\Big\|^2+\frac{\mathbb{E}_{\zeta_{t_n}}[\zeta_{t_n}^2]}{\gamma} 
    \nonumber\\
    % ==================
    % ==================
    \le&\ \frac{3\beta^2}{\gamma} \Big\|P^*_{n}(\overline{W}^{(n)}_{t_{n-1}+k})\odot F\Bigl(\frac{W_ {t_{n-1}+k}^{(n-1)}}{\kappa}\Bigr) - |P^*_{n}(\overline{W}^{(n)}_{t_{n-1}+k})|\odot G\Bigl(\frac{W_ {t_{n-1}+k}^{(n-1)}}{\kappa}\Bigr)\Big\|^2
 \bkeq + \frac{3\beta^2}{\gamma}\|W_{t_{n}+kT_n-1}^{(n)} - P^*_{n}(\overline{W}^{(n)}_{t_{n-1}+k})\|^2 +\frac{3\beta}{\gamma}\Theta(\Delta w_{\min}).
    \nonumber
\end{align}
Inequality \eqref{inequality:TT-convergence-scvx-Pstar-T3} is obtained by is bounded by Lemma \ref{lemma:pulse-update-error}, \ref{lemma:lip-analog-update} and Cauchy-Schwarz inequality. Therefore, the \textit{drift optimality} can be bounded by substituting \eqref{inequality:TT-convergence-scvx-Pstar-T3}  with $\|U+V\|^2 \leq 2\|U\|^2 +2\|V\|^2$: 
\begin{align}
\label{inequality: optimalty drift}
    & \quad \mathbb{E}_{\zeta_{t_n}}[\|W_{t_{n}+(k+1)T_n-1}^{(n)} - P^*_{n}(\overline{W}^{(n)}_{t_{n-1}+k+1})\|^2]\\ & \leq 2\|W_{t_{n}+(k+1)T_n-1}^{(n)} - P^*_{n}(\overline{W}^{(n)}_{t_{n-1}+k})\|^2+2\mathbb{E}_{\zeta_{t_n}}[\|P^*_{n}(\overline{W}^{(n)}_{t_{n-1}+k+1})- P^*_{n}(\overline{W}^{(n)}_{t_{n-1}+k})\|^2]\nonumber \\
   & \leq (2+\frac{6\beta^2}{\gamma})\|W_{t_{n}+(k+1)T_n-1}^{(n)} - P^*_{n}(\overline{W}^{(n)}_{t_{n-1}+k})\|^2 +\frac{6\beta^2}{\gamma}\Big\|P^*_{n}(\overline{W}^{(n)}_{t_{n-1}+k})\odot F\Bigl(\frac{W_{t_{n-1}+k}^{(n-1)}}{\kappa}\Bigr) \bkeq \quad - |P^*_{n}(\overline{W}^{(n)}_{t_{n-1}+k})|\odot G\Bigl(\frac{W_ {t_{n-1}+k}^{(n-1)}}{\kappa}\Bigr)\Big\|^2+ \frac{6\beta}{\gamma}\Theta(\Delta w_{\min}). \nonumber
\end{align}
\textbf{Establishing convergence.}

Define \(\mathcal{F}_k\) as the \(\sigma\)-algebra generated by all random variables up to time \(k\), including \(\{W^{(n)}_{t_n}\}_{t \leq k, n \in [0, N]}\). 
We write conditional expectations \(\mathbb{E}[~\cdot \mid \mathcal{F}_k]\) compactly as \(\mathbb{E}_k[~\cdot~]\) for brevity. For notational consistancy, we denote $W^* = :P^*_{0, t_{-1}+k}$ for all $k$.
Expanding $\mathbb{J}_{k+1}$ with \eqref{inequality: optimalty drift}, \eqref{inequality:TT_convergence-scvx-C3N-final} and \eqref{inequality:TT-convergence-scvx-C3n-final}, we get: 
\begin{align}
\label{inequality:descent 
_of_Lyapunov}
   &\quad\mathbb{E}_k[ \mathbb{J}_{k+1}]- \mathbb{J}_{k}\\& = \sum_{n'=0}^N \|W^{(n')}_{t_n'+(k+1)T_n'-1}-P^*_{N}(\overline{W}^{(n')}_{t_{n'-1}+k+1})\|^2-\sum_{n'=0}^N \|W^{(n')}_{t_n'+kT_n'-1}-P^*_{n'}(\overline{W}^{(n')}_{t_{n'-1}+k})\|^2 \nonumber \\&\leq
  \sum_{n'=1}^N   \bigg((2+\frac{6\beta^2}{\gamma})\|W_{t_{n'}+(k+1)T_n'-1}^{(n')} - P^*_{n'}(\overline{W}^{(n')}_{t_{n'-1}+k})\|^2 + \frac{6\beta^2}{\gamma}\Big\|P^*_{n'}(\overline{W}^{(n')}_{t_{n'-1}+k})\odot F\Bigl(\frac{W_ {t_{n'-1}+k}^{(n'-1)}}{\kappa}\Bigr)\bkeq \quad- |P^*_{n'}(\overline{W}^{(n')}_{t_{n'-1}+k})|\odot G\Bigl(\frac{W_ {t_{n'-1}+k}^{(n'-1)}}{\kappa}\Bigr)\Big\|^2 + \frac{6\beta}{\gamma}\Theta(\Delta w_{\min})\bigg)+\|W^{(0)}_{t_0+(k+1)T_0-1}-P^*_{0, t_{-1}+k}\|^2\bkeq \quad - \sum_{n'=0}^N \|W^{(n')}_{t_n'+kT_n'-1}-P^*_{n'}(\overline{W}^{(n')}_{t_{n'-1}+k})\|^2  \nonumber\\& \leq (2+\frac{6\beta^2}{\gamma})\Big( \big(1 - \frac{\alpha\mu L \gamma^N}{4(\mu + L)} \big)^{T_N} \| W_{t+kT_N-1}^{(N)}-P^*_{N}(\overline{W}^{(N)}_{t_{N-1}+k})\|^2 + 
        \frac{8(\mu + L)\alpha}{\gamma^N \mu L} F_{\max}^2 \sigma^2 +\gamma^{-\frac{4N}{3}}  \bkeq \quad\Theta ((\sigma \Delta w_{\min})^\frac{2}{3})\Big) 
     + \frac{6\beta^2}{\gamma} \Big\|P^*_{N}(\overline{W}^{(N)}_{t_{N-1}+k})\odot F\Bigl(\frac{W_ {t_{N-1}+k}^{(N-1)}}{\kappa}\Bigr) - |P^*_{N}(\overline{W}^{(N)}_{t_{N-1}+k})|\odot G\Bigl(\frac{W_ {t_{N-1}+k}^{(N-1)}}{\kappa}\Bigr)\Big\|^2 \bkeq \quad+ \frac{6\beta}{\gamma}\Theta(\Delta w_{\min})+
 \sum_{n'=1}^{N-1} \Bigg(\Big(2+\frac{6\beta^2}{\gamma}\Big) \Big(\big(-\frac{2\gamma^2}{H_{\min}}+\frac{6\beta\gamma F_{\max}}{H_{\min}}\big)\Big\|P^*_{n'}(\overline{W}^{(n')}_{t_{n'-1}+k})\odot F\Bigl(\frac{W_ {t_{n'-1}+k}^{(n'-1)}}{\kappa}\Bigr)\bkeq \quad - |P^*_{n'}(\overline{W}^{(n')}_{t_{n'-1}+k})|\odot G\Bigl(\frac{W_ {t_{n'-1}+k}^{(n'-1)}}{\kappa}\Bigr)\Big\|^2  +\big(1-\frac{\beta H_{\min}}{2\gamma  F_{\max}}\big)^{T_{n'}}\|W_ {t_{n'}+kT_n'-1}^{(n')}-  P^*_{n'}(\overline{W}^{(n')}_{t_{n'-1}+k})\|^2 \bkeq \quad + \frac{6 F_{\max}^4  \gamma^2}{H_{\min}^2}\|W_{t_{n'+1}+kT_{n'+1}-1}^{(n'+1)} - P^*_{n'}(\overline{W}^{(n'+1)}_{t_{n'}+k})\|^2 +\frac{2F_{\max}\Theta (\Delta w_{\min}) }{H_{\min}}\Big) +  \frac{6\beta^2}{\gamma} \Big\|P^*_{n'}(\overline{W}^{(n')}_{t_{n'-1}+k}) \bkeq \quad\odot F\Bigl(\frac{W_ {t_{n'-1}+k}^{(n'-1)}}{\kappa}\Bigr)- |P^*_{n'}(\overline{W}^{(n')}_{t_{n'-1}+k})|\odot G\Bigl(\frac{W_ {t_{n'-1}+k}^{(n'-1)}}{\kappa}\Bigr)\Big\|^2 + \frac{6\beta}{\gamma}\Theta(\Delta w_{\min})\Bigg)\bkeq \quad + \left( 1-\frac{\beta H_{\min}}{2\gamma  F_{\max}}\right)^{T_0}\|W_ {t_{0}+kT_0}^{(0)}-  P^*_{0,t_{-1}+k}\|^2  + \frac{6 F_{\max}^4  \gamma^2}{H_{\min}^2}\| W_{t_{1}+(k+1)T_{1}-1}^{(1)}- P^*_{1}(\overline{W}^{(1)}_{t_{0}+k})\|^2 \bkeq \quad+\frac{2F_{\max}\Theta (\Delta w_{\min}) }{H_{\min}} - \sum_{n'=0}^N \|W^{(n')}_{t_n'+kT_n-1}-P^*_{n'}(\overline{W}^{(n')}_{t_{n'-1}+k})\|^2 \nonumber\\
 & \leq \sum_{n'=1}^{N} \underbrace{\Big((2+\frac{6\beta^2}{\gamma})(-\frac{2\gamma^2}{H_{\min}}+\frac{6\beta\gamma F_{\max}}{H_{\min}})+\frac{6\beta^2}{\gamma}\Big)}_{(A)}  \Big\|P^*_{n'}(\overline{W}^{(n')}_{t_{n'-1}+k})\odot F\Bigl(\frac{W_ {t_{n'-1}+k}^{(n'-1)}}{\kappa}\Bigr) - |P^*_{n'}(\overline{W}^{(n')}_{t_{n'-1}+k})|\bkeq\quad\odot G\Bigl(\frac{W_ {t_{n'-1}+k}^{(n'-1)}}{\kappa}\Bigr)\Big\|^2+  \sum_{n'=1}^{N-1} \underbrace{\Big((2+\frac{6\beta^2}{\gamma})\big(1-\frac{\beta H_{\min}}{2\gamma  F_{\max}}\big)^{T_{n'}}+\frac{6 F_{\max}^4  \gamma^2}{H_{\min}^2}-1\Big)}_{(B)}\|W_{t_{n'}+kT_n'-1}^{(n')} - P^*_{n'}(\overline{W}^{(n')}_{t_{n'-1}+k})\|^2 \bkeq \quad+\underbrace{\Big((2+\frac{6\beta^2}{\gamma}) \big(1 - \frac{\alpha\mu L \gamma^N}{4(\mu + L)} \big)^{T_N}+ \frac{6 F_{\max}^4  \gamma^2}{H_{\min}^2}-1\Big)}_{(C)} \| W_{t+kT_N-1}^{(N)}-P^*_{N}(\overline{W}^{(N)}_{t_{N-1}+k})\|^2 
 +(2+\frac{6\beta^2}{\gamma}) \bkeq \quad \big( \frac{8(\mu + L)\alpha}{\gamma^N \mu L} F_{\max}^2 \sigma^2
  +\gamma^{-\frac{4N}{3}}  \Theta ((\sigma \Delta w_{\min})^\frac{2}{3})\big)
  +\sum_{n'=0}^{N-1} \Big((2+\frac{6\beta^2}{\gamma})\frac{2F_{\max}\Theta (\Delta w_{\min}) }{H_{\min}}+\frac{6\beta}{\gamma}\Theta(\Delta w_{\min})\Big)\bkeq \quad +\Big(\big(1-\frac{\beta H_{\min}}{2\gamma  F_{\max}}\big)^{T_{0}}-1\Big)\|W_{t_{0}+kT_0-1}^{(0)} - P^*_{0,t_{-1}+k}\|^2 . \nonumber
\end{align}
The coefficient $\Big(\big(1-\frac{\beta H_{\min}}{2\gamma  F_{\max}}\big)^{T_{n'}}-1\Big)$ for the last term is negative when $T_{0} \geq 1$.
To achieve descent on Lyapunov function, we choose \begin{align}
    \beta = \Theta(\gamma^2), \quad \alpha = \Theta(1), \quad
   \gamma  \le  \sqrt{\frac{1-\frac{3}{2}\rho}{6}} 
\frac{H_{\min}}{F_{\max}^{2}}, \quad
T_{n} \geq \Theta \bigl(\gamma^{-1}\bigr), n \in \{1,\ldots,N-1\},\quad
T_{N} \geq \Theta \bigl(\gamma^{-N}\bigr)\nonumber
\end{align} 
to satisfy the following conditions, where $\rho \in (0,\frac{2}{3})$
is a constant:
\begin{align}
   A := (2+\frac{6\beta^2}{\gamma})(-\frac{2\gamma^2}{H_{\min}}+\frac{6\beta\gamma F_{\max}}{H_{\min}})+\frac{6\beta^2}{\gamma}& \leq 0, \label{inequality:condition_1}\\
   B:= (2+\frac{6\beta^2}{\gamma})\big(1-\frac{\beta H_{\min}}{2\gamma  F_{\max}}\big)^{T_{n}}+\frac{6 F_{\max}^4  \gamma^2}{H_{\min}^2}-1 &\leq -\rho, \quad n \in [1,N-1], \label{inequality:condition_2}\\
   C:= (2+\frac{6\beta^2}{\gamma}) \big(1 - \frac{\alpha\mu L \gamma^N}{4(\mu + L)} \big)^{T_N}+ \frac{6 F_{\max}^4  \gamma^2}{H_{\min}^2}-1 &\le -\rho.\label{inequality:condition_3}
\end{align}
Let $\beta =c_\beta \gamma^2$, the inequality \eqref{inequality:condition_1} becomes:
\begin{align}
    \left(2 + 6c_\beta^2 \gamma^3\right) \left(-2\gamma^2 + 6c_\beta F_{\max}\gamma^3 \right) + 6c_\beta^2 \gamma^3 H_{\min}\leq 0
\end{align}
which holds when $\gamma \in (0,
\frac{H_{\min}}{\sqrt{6}F_{\max}^{2}}]$ and $c_\beta \in (0, \frac{1}{5F_{\max}}]$.
For \eqref{inequality:condition_2} and \eqref{inequality:condition_3}, we first give the upper bound of $\gamma$ as $\gamma  \le  
\frac{\sqrt{1-\frac{3}{2}\rho}H_{\min}}{\sqrt{6}F_{\max}^{2}}
 $ in order to satisfy
\(
\frac{6F_{\max}^{4}\gamma^{2}}{H_{\min}^{2}}-1 \le -\frac{3}{2}\rho
\). We then bound the inner loops count $T_n$, $n\in\{0,\ldots,N-1\}$ by defining :
\begin{align}
    \lambda_1 = \frac{\beta H_{\min}}{2\gamma F_{\max}}
          = \frac{c_\beta H_{\min}}{2F_{\max}}\gamma
          =: c_1\gamma= \Theta(\gamma),
\qquad
A_1 = 2 + \frac{6\beta^2}{\gamma}
    = 2 + 6c_\beta^{2}\gamma^{3}
    = \mathcal{O}(1).
\end{align}
Using $(1-\lambda_1)^{T_n}\le e^{-\lambda_1 T_n}$,
inequality~\eqref{inequality:condition_2} must satisfy
$A_1 e^{-\lambda_1 T_n}\le \rho/2$, which yields
 $T_n  \ge 
\frac{1}{\lambda_1} 
\ln\ \Bigl(\frac{2A_1}{\rho}\Bigr)
= \Theta\ \bigl(\gamma^{-1}\bigr)$.
Similarly,
let 
\(
\lambda_2 = \frac{\alpha\mu L\gamma^{N}}{4(\mu+L)}
          =: c_2\gamma^{N}  =\Theta(\gamma^{N}),
\)
to satisfy
$A_1 e^{-\lambda_2 T_N}\le \rho/2$, we bound $T_N$ as
$T_N  \ge 
\frac{1}{\lambda_2} 
\ln\ \Bigl(\frac{2A_1}{\rho}\Bigr)
= \Theta\ \bigl(\gamma^{-N}\bigr)$.
Once inequalities \eqref{inequality:condition_1}, \eqref{inequality:condition_2}, and \eqref{inequality:condition_3} are satisfied, the Lyapunov descent inequality \eqref{inequality:descent 
_of_Lyapunov} simplifies to:
\begin{align}
\label{inequality:descent_of_Lyapunov_compact}
\mathbb{E}_k[ \mathbb{J}_{k+1}] \leq (1-\rho) \mathbb{J}_{k} + \gamma^{-\frac{4N}{3}} \Theta\big((\sigma \Delta w_{\min})^{\frac{2}{3}}\big).
\end{align}
    Taking full expectation over \(\mathcal{F}_k\) on both sides, the recurrence becomes:
\begin{align}
  \mathbb{E}[\mathbb{J}_{k+1}] \le (1 - \rho) \mathbb{E}[\mathbb{J}_k] + \gamma^{-\frac{4N}{3}} \Theta\big((\sigma \Delta w_{\min})^{\frac{2}{3}}\big).
\end{align}
Unrolling this inequality over $k$ steps, we get:
\begin{align}
\label{inequality:65}
  \mathbb{E}[\mathbb{J}_k] &\le (1 - \rho)^k\mathbb{E}[\mathbb{J}_0] + \gamma^{-\frac{4N}{3}} \Theta\big((\sigma \Delta w_{\min})^{\frac{2}{3}}\big) \sum_{t=0}^{K-1} (1 - \rho)^t \\
&\le (1 - \rho)^k \mathbb{E}[\mathbb{J}_0] + \frac{\gamma^{-\frac{4N}{3}} \Theta\big((\sigma \Delta w_{\min})^{\frac{2}{3}}\big)}{\rho}.  \nonumber
\end{align}
Since $\rho \in (0,1) = \Theta(1)$, we get:
\begin{align}
  \mathbb{E}[\mathbb{J}_k] \le (1 - \rho)^k \mathbb{E}[\mathbb{J}_0] + \Theta\big(\gamma^{-\frac{4N}{3}} (\sigma \Delta w_{\min})^{\frac{2}{3}}\big).  
\end{align}
Since $W^{(N)}$ dominates the compute cost, one outer iteration
(i.e.\ one update of $\mathbb{J}_k$) consumes
\(
t  =  \prod_{n=0}^N T_n k  = \mathcal{O}(\gamma^{-2N})\cdot k
\)
gradient evaluations.  
Therefore the averaged Lyapunov function, as a function of the total
gradient evaluations, obeys:
\begin{align}
  \mathbb{E}[\mathbb{J}_k] \le \mathcal{O}((1 - \rho)^{\gamma^{2N} t})\mathbb{E}[\mathbb{J}_0] + \Theta\big(\gamma^{-\frac{4N}{3}} (\sigma \Delta w_{\min})^{\frac{2}{3}}\big)  \nonumber
\end{align}
which completes the proof.
\end{proof}
\ThmTTConvergencecorollary*
\begin{proof}[Proof of Corollary \ref{corollary:Convergence_rate}]
From \eqref{inequality:65}, taking the limit as \(k \to \infty\), we obtain:
\begin{align}
\label{inequality:68}
\limsup_{k \to \infty} \mathbb{E}[\mathbb{J}_k]
\le \Theta( \gamma^{- \frac{4N}{3}} (\sigma \Delta w_{\min})^{\frac{2}{3}} ).
\end{align}
Once \eqref{inequality:68} holds for \(\mathbb{J}_k\) , each component of \(\mathbb{J}_k\) also satisfies:
\begin{align}
\label{inequality:69}
\limsup_{t_n \to \infty} \mathbb{E}\big[\|W^{(n)}_{t_n} - P^*_n(\overline{W}^{(n)}_{t_{n-1}})\|^2\big]
\leq 
\Theta\big(\gamma^{- \frac{4N}{3}} (\sigma \Delta w_{\min})^{\frac{2}{3}}\big).
\end{align}
Substituting \(P^*_n(\overline{W}^{(n)}) = \gamma^{-n}(W^* - \overline{W}^{(n)})\) into \eqref{inequality:69}, we derive the bound on the scaled residual:
\begin{align}
\limsup_{t_n \to \infty} \mathbb{E}\big[\|W^* - \overline W^{(n+1)}_{t_n}\|^2\big]
\leq 
\Theta\big(\gamma^{2n - \frac{4N}{3}} (\sigma \Delta w_{\min})^{\frac{2}{3}}\big), \quad n \in [0, N].
\end{align}
In particular, when \(n = N\), we obtain the desired result as:
\begin{align}
\limsup_{t \to \infty} \mathbb{E}\big[\|W^* - \overline{W}_{t}\|^2\big]
\leq 
\Theta\big(\gamma^{\frac{2N}{3}} (\sigma \Delta w_{\min})^{\frac{2}{3}}\big).
\end{align}
This demonstrates that increasing the number of tiles leads to an error decay rate proportional to \(\gamma^{\frac{2}{3}}\), thereby completing the proof.
\end{proof}

\subsection{Proof of Lemma \ref{lemma:TT-WN-descent}: Descent of the main sequence \texorpdfstring{$W_{t}^{(N)}$}{Wtnn}}
\label{section:proof-lemma:TT-WN-descent}
\LemmaTTWnDescentScvxWS*
\begin{proof}[Proof of Lemma \ref{lemma:TT-WN-descent}]
The proof begins from manipulating the norm $\|W_{t+1}^{(N)} - P^*_{N}(\overline{W}^{(N)}_{t_{N-1}})\|^2$
\begin{align}
    \label{inequality:TT-convergence-PN-S1}
    & \quad \mbE_{\xi_{t}, \zeta_{t}}[\|W_{t+1}^{(N)}-P^*_{N}(\overline{W}^{(N)}_{t_{N-1}})\|^2] \\ &\nonumber
    =\|W_{t}^{(N)}-P^*_{N}(\overline{W}^{(N)}_{t_{N-1}})\|^2 + 2\mbE_{\xi_{t}, \zeta_{t}}[\langle W_{t}^{(N)}-P^*_{N}(\overline{W}^{(N)}_{t_{N-1}}), W_{t+1}^{(N)}-  W_{t}^{(N)}\rangle] \bkeq \quad+ \mbE_{\xi_{t}, \zeta_{t}}[\|W_{t+1}^{(N)}-  W_{t}^{(N)}\|^2].\nonumber
\end{align}
To bound the second term, we first apply the update dynamics given in ~\eqref{eq:TT-WN-descent} to obtain the following equality:
\begin{align}
    \label{inequality:TT-convergence-N-S1-T2}
    &\ 2\mbE_{\xi_{t}, \zeta_{t}}[\langle  W_{t}^{(N)}-P^*_{N}(\overline{W}^{(N)}_{t_{N-1}}),W_{t+1}^{(N)}-  W_{t}^{(N)}\rangle]
    \\
    % ==================
    =&\ -2 \mbE_{\xi_{t}, \zeta_{t}}\Big[\Big\langle  W_{t}^{(N)}-P^*_{N}(\overline{W}^{(N)}_{t_{N-1}}), \alpha\nabla f(\overline W_{t}; \xi_{t}) \odot F\Bigl( \frac{W_{t}^{(N)}}{\kappa} \Bigr) + |\alpha\nabla f(\overline W_{t}; \xi_{t}) |\bkeq \odot G\Bigl( \frac{W_{t}^{(N)}}{\kappa} \Bigr)-\zeta_{t}\Big\rangle \Big] 
    \nonumber\\
    % ==================
    =&\ -2\alpha \mbE_{\xi_{t}}\Big[\Big\langle    W_{t}^{(N)}-P^*_{N}(\overline{W}^{(N)}_{t_{N-1}}), \nabla f(\overline W_{t}; \xi_{t}) \odot F\Bigl( \frac{W_{t}^{(N)}}{\kappa} \Bigr)\Big\rangle\Big]
    \bkeq
    -2\alpha \mbE_{\xi_{t}}\Big[\Big\langle    W_{t}^{(N)}-P^*_{N}(\overline{W}^{(N)}_{t_{N-1}}), |\nabla f(\overline W_{t}; \xi_{t})| \odot G\Bigl( \frac{W_{t}^{(N)}}{\kappa} \Bigr)\Big\rangle\Big]
    \nonumber\\
    % ==================
    =&\ -2\alpha \Big \langle    W_{t}^{(N)}-P^*_{N}(\overline{W}^{(N)}_{t_{N-1}}), \nabla f(\overline W_{t}) \odot F \Bigl(\frac{W_{t}^{(N)}}{\kappa} \Bigr) \Big\rangle
    \bkeq
    + 2\alpha \mbE_{\xi_{t}}\Big[\Big\langle    W_{t}^{(N)}-P^*_{N}(\overline{W}^{(N)}_{t_{N-1}}), (|\nabla f(\overline W_{t})|-|\nabla f(\overline  W_{t}; \xi_{t})|)\odot G \Bigl(  \frac{W_{t}^{(N)}}{\kappa} \Bigr)\Big\rangle \Big]
    \nonumber\\ & - 2\alpha \Big \langle    W_{t}^{(N)}-P^*_{N}(\overline{W}^{(N)}_{t_{N-1}}), |\nabla f(\overline W_{t}) |\odot G \Bigl(  \frac{W_{t}^{(N)}}{\kappa} \Bigr)\Big\rangle \nonumber\\
    % ==================
    \leq &\ -2\alpha \underbrace{\Big\langle     W_{t}^{(N)}-P^*_{N}(\overline{W}^{(N)}_{t_{N-1}}), \nabla f(\overline W_{t}) \odot F\Bigl( \frac{W_{t}^{(N)}}{\kappa} \Bigr) + |\nabla f(\overline W_{t}) |\odot G\Bigl( \frac{W_{t}^{(N)}}{\kappa} \Bigr)\Big\rangle}_{(T1)}
    \bkeq
    + 2\alpha \underbrace{\mbE_{\xi_{t}}\Big[\Big\langle     W_{t}^{(N)}-P^*_{N}(\overline{W}^{(N)}_{t_{N-1}}), (|\nabla f(\overline W_{t})|-|\nabla f(\overline W_{t}; \xi_{t})|)\odot G \Bigl(  \frac{W_{t}^{(N)}}{\kappa} \Bigr)\Big\rangle \Big]}_{(T2)} \nonumber
\end{align}

\textbf{Upper bound of the first term $(T1)$.}
With Lemma \ref{lemma:element-wise-product-error}, the second term in the \ac{RHS} of \eqref{inequality:TT-convergence-PN-S1} can be bounded by:
\begin{align}
    \label{inequality:TT-convergence-PN-S1-T2-T1-S1}
    &\ -2\alpha \Big\langle     W_{t}^{(N)}-P^*_{N}(\overline{W}^{(N)}_{t_{N-1}}), \nabla f(\overline W_{t}) \odot F\Bigl( \frac{W_{t}^{(N)}}{\kappa} \Bigr) + |\nabla f(\overline W_{t}) |\odot G\Bigl( \frac{W_{t}^{(N)}}{\kappa} \Bigr)\Big\rangle
    \\
    % ==================
    =&\ -2\alpha \Big\langle   W_{t}^{(N)}-P^*_{N}(\overline{W}^{(N)}_{t_{N-1}}), \nabla f(\overline W_{t}) \odot q_s\Bigl( \frac{W_{t}^{(N)}}{\kappa} \Bigr)\Big\rangle
    \nonumber\\
    % ==================
    \le&\ -2\alpha   C_{k, +} \langle   W_{t}^{(N)}-P^*_{N}(\overline{W}^{(N)}_{t_{N-1}}), \nabla f(\overline W_{t}) \rangle
     +2\alpha  C_{k, -} \langle |  W_{t}^{(N)}-P^*_{N}(\overline{W}^{(N)}_{t_{N-1}})|, |\nabla f(\overline W_{t})| \rangle
    \nonumber
\end{align}
where 
$C_{k,+}$ and $C_{k,-}$ are defined as:
\begin{align}
    C_{k, +} :=& \frac{1}{2}\left(\max_{i\in\ccalI}\{q_s([\frac{W_{t}^{(N)}}{\kappa}]_i)\} + \min_{i\in\ccalI}\{q_s([\frac{W_{t}^{(N)}}{\kappa}]_i)\}\right), \\
    % ==================
    C_{k, -} :=& \frac{1}{2}\left(\max_{i\in\ccalI}\{q_s([\frac{W_{t}^{(N)}}{\kappa}]_i)\} - \min_{i\in\ccalI}\{q_s([\frac{W_{t}^{(N)}}{\kappa}]_i)\}\right).
\end{align}
In the inequality above, the first term can be bounded by the strong convexity of $f$. 
Let $\varphi(W^{(N)}) := f( \overline W^{(N)} + \gamma^N W^{(N)})$ which is $\gamma^{2N} L$-smooth and {$\gamma^{2N}\mu$-strongly convex}. It can be verified that $\varphi(W_{t}^{(N)})$ has gradient $\nabla \varphi(W_{t}^{(N)}) =  \gamma^N
 \nabla f(\overline W_{t_{N-1}}^{(N)}+\gamma^N W_{t}^{(N)})= \gamma^N
 \nabla f(\overline W_{t_{N}})$ and optimal point $P^*( \overline  W_{t_{N-1}}^{(N)})$.
Leveraging Theorem 2.1.9 in \citep{nesterov2018lectures}, we have:
\begin{align}
    &\ 2\alpha  C_{k,+}\langle   W_{t}^{(N)}-P^*_{N}(\overline{W}^{(N)}_{t_{N-1}}), \nabla f(\overline W_{t}) \rangle
  \\=& 
     \frac{2\alpha  C_{k,+}}{\gamma^N }\langle \nabla \varphi(W_{t}^{(N)}),   W_{t}^{(N)}-P^*_{N}(\overline{W}^{(N)}_{t_{N-1}})\rangle\nonumber\\ =&\frac{2\alpha  C_{k,+}}{\gamma^N}\langle \nabla \varphi(W_{t}^{(N)})-\nabla \varphi(P^*_{N}(\overline{W}^{(N)}_{t_{N-1}})),   W_{t}^{(N)}-P^*_{N}(\overline{W}^{(N)}_{t_{N-1}})\rangle    \nonumber
    \\
    % ==================
    \ge&\ \frac{2\alpha  C_{k,+}}{\gamma^N }\left(\frac{\gamma^{2N}\mu \cdot \gamma^{2N} L}{\gamma^{2N}\mu+\gamma^{2N} L}\|  W_{t}^{(N)}-P^*_{N}(\overline{W}^{(N)}_{t_{N-1}})\|^2 + \frac{1}{\gamma^{2N}\mu+\gamma^{2N}L}\|\nabla \varphi(W_{t}^{(N)})\|^2\right)
    \nonumber\\
    % ==================
    =&\ \frac{2\alpha  C_{k,+}\mu L \gamma^{N}}{\mu+L}\|  W_{t}^{(N)}-P^*_{N}(\overline{W}^{(N)}_{t_{N-1}})\|^2 + \frac{2\alpha  C_{k,+}}{\gamma^{N}(\mu+L)}\|\nabla f(\overline W_{t})\|^2.
    \nonumber
\end{align}
The second term in the \ac{RHS} of \eqref{inequality:TT-convergence-PN-S1-T2-T1-S1} can be bounded by the following inequality:
\begin{align}
    \label{inequality:TT-convergence-PN-S1-T2-T1-S1-T2}
    & 2\alpha  C_{k,-} \Big\langle |  W_{t}^{(N)}-P^*_{N}(\overline{W}^{(N)}_{t_{N-1}})|, |\nabla f(\overline W_{t})|\Big \rangle
    \\
    % ==================
    \le&\ \frac{\alpha C_{k,-}^2\gamma^{N}(\mu+L)}{C_{k,+}} \|  W_{t}^{(N)}-P^*_{N}(\overline{W}^{(N)}_{t_{N-1}})\|^2 
    + \frac{\alpha C_{k,+}}{\gamma^{N}(\mu+L)} \|\nabla f(\overline W_{t})\|^2.
    \nonumber
\end{align}
Therefore, \eqref{inequality:TT-convergence-PN-S1-T2-T1-S1} becomes:
\begin{align}
\label{inequality:TT-convergence-PN-S1-T2-final}
    &\ -2\alpha \langle     W_{t}^{(N)}-P^*_{N}(\overline{W}^{(N)}_{t_{N-1}}), \nabla f(\overline W_{t}) \odot F\Bigl( \frac{W_{t}^{(N)}}{\kappa} \Bigr) + |\nabla f(\overline W_{t}) |\odot G\Bigl( \frac{W_{t}^{(N)}}{\kappa} \Bigr)\rangle
    \\
    % ==================
    \le&\ - \gamma^{N} \left(\frac{2\alpha \mu L C_{k, +}}{\mu+L}-\frac{\alpha C_{k,-}^2(\mu+L)}{C_{k,+}}\right)\|  W_{t}^{(N)}-P^*_{N}(\overline{W}^{(N)}_{t_{N-1}})\|^2 - \frac{\alpha C_{k, +}}{\gamma^{N}(\mu+L)}\|\nabla f(\overline W_{t})\|^2.
    \nonumber
\end{align}

\textbf{Upper bound of the second term $(T2)$.} Leveraging the Young's inequality, we have:
\begin{align}
    &  2\alpha \mbE_{\xi_{t}}\Big[\Big\langle     W_{t}^{(N)}-P^*_{N}(\overline{W}^{(N)}_{t_{N-1}}), (|\nabla f(\overline W_{t})|-|\nabla f(\overline W_{t}; \xi_{t})|)\odot G \Bigl(  \frac{W_{t}^{(N)}}{\kappa} \Bigr)\Big\rangle \Big]
    \\
    % ==================
    \le&\ \frac{\alpha \mu L C_{k, +}\gamma^{N}}{(\mu+L)} \|  W_{t}^{(N)}-P^*_{N}(\overline{W}^{(N)}_{t_{N-1}})\|^2 
    \bkeq
    + \frac{\alpha(\mu+L)}{\mu L C_{k, +}\gamma^{N}} \mbE_{\xi_{t}}\Big[\Big\|(|\nabla f(\overline W_{t})|-|\nabla f(\overline W_{t}; \xi_{t})|)\odot G \Bigl(\frac{1}{\kappa}W_{t}^{(N)}\Bigr)\Big\|^2 \Big]
    \nonumber\\
    % ==================
    \lemark{a}&\ \frac{\alpha \mu L C_{k, +}\gamma^{N}}{(\mu+L)} \|  W_{t}^{(N)}-P^*_{N}(\overline{W}^{(N)}_{t_{N-1}})\|^2 
    \bkeq
    + \frac{\alpha(\mu+L)}{\mu L C_{k, +}\gamma^{N}} \mbE_{\xi_{t}}\Big[\Big\|(|\nabla f(\overline W_{t})-\nabla f(\overline W_{t}; \xi_{t})|)\odot G\Bigl( \frac{W_{t}^{(N)}}{\kappa} \Bigr)\Big\|^2\Big]
    \nonumber\\
    % ==================
    \eqmark{b}&\ \frac{\alpha \mu L C_{k, +}\gamma^{N}}{(\mu+L)} \|  W_{t}^{(N)}-P^*_{N}(\overline{W}^{(N)}_{t_{N-1}})\|^2 
    + \frac{\alpha(\mu+L)\sigma^2}{\mu L C_{k, +}\gamma^{N}}\Big\| G\Bigl( \frac{W_{t}^{(N)}}{\kappa} \Bigr)\Big\|_\infty^2
    \nonumber
\end{align}
where $(a)$ applies $||x|-|y|| \le |x-y|$ for any $x, y\in\reals$, $(b)$ uses the bounded variance assumption (see Assumption \ref{assumption:noise}).
Combining the upper bound of $(T1)$ and $(T2)$, we bound \eqref{inequality:TT-convergence-N-S1-T2} by:
\begin{align}
    \label{inequality:TT-convergence-P-S1-T1-T2}
    &\ 2\mbE_{\xi_{t}}[\langle   W_{t}^{(N)}-P^*_{N}(\overline{W}^{(N)}_{t_{N-1}}),  W_{t+1}^{(N)}-W_{t}^{(N)}\rangle]
    \\
    % ==================
    \le&\ - \gamma^{N}\left(\frac{\alpha \mu L C_{k, +}}{\mu+L}-\frac{\alpha C_{k,-}^2(\mu+L)}{C_{k,+}}\right)\|  W_{t}^{(N)}-P^*_{N}(\overline{W}^{(N)}_{t_{N-1}})\|^2 
    -  \frac{\kappa^2\alpha C_{k, +}}{\gamma^{N}(\mu+L)}\|\nabla f(\overline W_{t})\|^2 \bkeq
    + \frac{\alpha(\mu+L)\sigma^2}{\mu L C_{k, +}\gamma^{N}}\lnorm G\Bigl( \frac{W_{t}^{(N)}}{\kappa} \Bigr)\rnorm_\infty^2 
    \nonumber\\
    % ==================
    \le&\ - \frac{\gamma^{N}\alpha \mu L C_{k, +}}{2(\mu+L)}\|  W_{t}^{(N)}-P^*_{N}(\overline{W}^{(N)}_{t_{N-1}})\|^2 
    - \frac{\alpha C_{k, +}}{\gamma^{N}(\mu+L)}\|\nabla f(\overline W_{t})\|^2 
    + \frac{\alpha(\mu+L)\sigma^2}{\mu L C_{k, +}\gamma^{N}}\lnorm G(\frac{W_{t}^{(N)}}{\kappa})\rnorm_\infty^2 
    \nonumber
\end{align}
where the last inequality holds for  
$C_{k,-} \ll C_{k,+}$ , which is  sufficiently close to 0, and the following inequality holds:
\begin{align}
    (\mu+L)\frac{C_{k,-}^2}{C_{k,+}^2}
    \le \frac{\mu L}{2(\mu+L)}.
\end{align}
Furthermore, the last term in the \ac{RHS} of \eqref{inequality:TT-convergence-PN-S1} can be bounded by the Lipschitz continuity of analog update (see Lemma \ref{lemma:lip-analog-update}) and the bounded variance assumption (see Assumption \ref{assumption:noise}) as:
\begin{align} 
    \label{inequality:TT-convergence-PN-S1-T3}
   &  \mbE_{\xi_{t}, \zeta_{t}}\Big[\Big\|W_{t+1}^{(N)}-  W_{t}^{(N)}\|^2]
    \\ =&\ \mbE_{\xi_{t}, \zeta_{t}}[\|\alpha \nabla f(\overline W_{t}; \xi_{t})\odot F\Bigl( \frac{W_{t}^{(N)}}{\kappa} \Bigr) - \alpha  |\nabla f(\overline W_{t}; \xi_{t})|\odot G\Bigl(\frac{W_{t}^{(N)}}{\kappa}\Bigr)+ \zeta_{t}\Big\|^2 \Big] \nonumber
    % ==================
      \\
    % ==================
    \leq &\ 2\alpha^2 F_{\max}^2 \mbE_{\xi_{t}}[\|\nabla f(\overline W_{t};\xi_{t})\|^2 ]+2 \alpha \Theta(\Delta w_{\min})
    \nonumber\\
    % ==================
    \leq&\ 2\alpha^2  F_{\max}^2  \|\nabla f(\overline W_{t})\|^2 + 2\alpha^2  F_{\max}^2 \sigma^2 +2\alpha \Theta(\Delta w_{\min})
    \nonumber\\
    % ==================
    \le&\ \frac{\alpha C_{k, +}}{2\gamma^{N}(\mu+L)}\|\nabla f(\overline W_{t})\|^2 + 2  \alpha^2 F_{\max}^2 \sigma^2 +2\alpha \Theta(\Delta w_{\min})
    \nonumber
\end{align}
where the first inequality holds by $\|U+V\|^2 \leq 2\|U\|^2+ 2\|V\|^2$ and $\mbE_{\zeta_{t}}[\zeta_{t}^2] = \alpha \Theta(\Delta w_{\min})$, the last inequality holds if $\alpha \leq \frac{C_{k, +}}{4 \gamma^N (\mu+L) F_{\max}^2}$ .
Plugging inequality \eqref{inequality:TT-convergence-P-S1-T1-T2} and \eqref{inequality:TT-convergence-PN-S1-T3} above into \eqref{inequality:TT-convergence-PN-S1} yields:
\begin{align}
\label{inequality:70}
    &\ \mbE_{\xi_{t},\zeta_{t}}[\|  W_{t+1}^{(N)}-P^*_{N}(\overline{W}^{(N)}_{t_{N-1}})\|^2] \\
    \le&\ 
    \left( 1 - \frac{\alpha \mu L C_{k, +}\gamma^{N}}{2(\mu+L)}\right) \|  W_{t}^{(N)}-P^*_{N}(\overline{W}^{(N)}_{t_{N-1}})\|^2 -\frac{\alpha C_{k, +}}{2\gamma^{N}(\mu+L)}\|\nabla f(\overline W_{t})\|^2 \nonumber\\ &
    + \frac{\alpha(\mu+L)\sigma^2}{\mu L C_{k, +}\gamma^{N}}\Big\|G\Bigl( \frac{W_{t}^{(N)}}{\kappa} \Bigr)\Big\|_\infty^2
    + 2\alpha^2  F_{\max}^2 \sigma^2+ 2\alpha \Theta(\Delta w_{\min}).
    \nonumber
\end{align}
From the definition of $C_{k,+}$, when the saturation degree of $W_{t}^{(N)}$ is properly limited, we have $C_{k,+} \ge \frac{1}{2}$ since $\alpha \gamma^n$ is sufficiently small. Therefore, we have:
\begin{align}
\saveeq{inequality-saved:TT-WN-descent}{
    &\ \mbE_{\xi_{t},\zeta_{t}}[\|  W_{t+1}^{(N)}-P^*_{N}(\overline{W}^{(N)}_{t_{N-1}})\|^2] \\
    \le&\ 
    \left(1-\frac{\mu L \alpha\gamma^{N}}{4(\mu+L)}\right) \|  W_{t}^{(N)}-P^*_{N}(\overline{W}^{(N)}_{t_{N-1}})\|^2 
    -  \frac{\alpha }{4\gamma^{N}(\mu+L)}\|\nabla f(\overline W_{t})\|^2 \nonumber \\ & +\frac{2\alpha(\mu+L)\sigma^2}{\mu L\gamma^{N}}\Big\| G\Bigl( \frac{W_{t}^{(N)}}{\kappa} \Bigr)\Big\|_\infty^2
    + 2\alpha^2 F_{\max}^2 \sigma^2+2\alpha \Theta( \Delta w_{\min})    \nonumber}.
\end{align}
Under Assumption \ref{assumption:pulse-response-symmetry}, which indicates $G(0) = 0$ and the Lipschitz continuity of the response functions, we can directly bound the term $\Big\|G \Bigl(\frac{W_{t}^{(N)}}{\kappa} \Bigr)\Big\|^2_\infty$ in \eqref{inequality:70} as:
    \begin{align}
    \label{inequality:L_G}
    & \quad\Big\|G \Bigl(\frac{W_{t}^{(N)}}{\kappa} \Bigr)\Big\|^2_\infty
        \le \Big\|G \Bigl(\frac{W_{t}^{(N)}}{\kappa} \Bigr)\Big\|^2
        =\Big\|G \Bigl(\frac{W_{t}^{(N)}}{\kappa} \Bigr)-G(0)\Big\|^2 
   \le \frac{L_G^2}{\kappa^2} \| W_{t}^{(N)}\|^2_{\infty}
    \end{align}
     where $L_G\ge 0$ is a Lipschitz constant.  
Perform \( T_N \) iterations using the recursive process in \eqref{inequality:70}, and denote the expectation over the noise sequence $\mathbb{E}_{\xi_{t:t+T_N-1},\zeta_{t:t+T_N-1}} $ as \( \mathbb{E}_{\boldsymbol{\xi_N}, \boldsymbol{\zeta_N}} \), we obtain the following upper bound:
\begin{align}
\label{eq:85}
& \mathbb{E}_{\boldsymbol{\xi_N}, \boldsymbol{\zeta_N}}\left[\| W_{t+T_N-1}^{(N)} - P^*_{N}(\overline{W}^{(N)}_{t_{N-1}}) \|^2\right]\\
\le & \left(1- \frac{\alpha\mu L \gamma^N }{4(\mu+L)}\right)^{T_{N}} \|  W_{t}^{(N)}-P^*_{N}(\overline{W}^{(N)}_{t_{N-1}})\|^2 
   +\sum_{i=0}^{T_{N}-1} \left(1- \frac{\alpha\mu L \gamma^N}{4(\mu+L)}\right)^i   \bigg(  2\alpha^2 F_{\max}^2 \sigma^2
  \bkeq +\frac{2\alpha(\mu+L)\sigma^2 L_G^2 \|W_{t}^{(N)}\|_\infty^2}{\mu L \gamma^N \kappa^2}+2\alpha\Theta(\Delta w_{\min})\bigg)\nonumber 
    \\ \le\ & 
    \left(1 - \frac{\alpha\mu L \gamma^N}{4(\mu + L)} \right)^{T_N} \| W_{t}^{(N)}-P^*_{N}(\overline{W}^{(N)}_{t_{N-1}})\|^2 + 
        \frac{8(\mu + L)\alpha}{\gamma^N \mu L} F_{\max}^2 \sigma^2 
      + \frac{8(\mu + L)^2 \sigma^2 L_G^2 W_{\max}^2}{\gamma^{2N} \mu^2 L^2 \kappa^2}  \nonumber \\
    & +\frac{8 (\mu + L)\kappa}{\gamma^N \mu L}  \Theta(\Delta c_{\min})
   \nonumber \\ \le &
 \left(1 - \frac{\alpha\mu L \gamma^N}{4(\mu + L)} \right)^{T_N} \| W_{t}^{(N)}-P^*_{N}(\overline{W}^{(N)}_{t_{N-1}})\|^2 + 
        \frac{8(\mu + L)\alpha}{\gamma^N \mu L} F_{\max}^2 \sigma^2 +\gamma^{-\frac{4N}{3}}  \Theta ((\sigma \Delta w_{\min})^\frac{2}{3}).
    \nonumber
\end{align}
The second inequality holds for $\sum_{i=0}^{T_{N}-1} \left(1- \frac{\alpha\mu L \gamma^N}{8(\mu+L)}\right)^i \leq \frac {8(\mu+L)}{\alpha\mu L \gamma^N}$, and we define $W_{\max} \in [\|W_{t}^{(N)}\|_\infty, \tau_{\max})$ for all $t$.
The last inequality holds by choosing the mapping constant as:
\begin{align}
\label{inequality:72}
\kappa = (\sigma L_G W_{\max})^{\frac{2}{3}}(\gamma^{N} \Delta c_{\min})^{-
\frac{1}{3}}=(\sigma L_G W_{\max})^{\frac{2}{3}}(\frac{\gamma^{N} \Delta w_{\min}}{\kappa})^{-
\frac{1}{3}}.
\end{align}
The second equality holds by substituting \eqref{eq:mapping}. Rearranging \eqref{inequality:72}, we get:
\begin{align}
\kappa=(\sigma L_G W_{\max})^{\frac{1}{2}}(\gamma^{N} \Delta w_{\min})^{-
\frac{1}{4}}
\end{align}
When $k \neq 0$, \eqref{eq:85} can be written as the general case:
\begin{align}
& \mathbb{E}_{\boldsymbol{\xi_N}, \boldsymbol{\zeta_N}}\left[\| W_{t+(k+1)T_N-1}^{(N)} - P^*_{N}(\overline{W}^{(N)}_{t_{N-1}+k})\|^2\right] \\ \le &
 \left(1 - \frac{\alpha\mu L \gamma^N}{4(\mu + L)} \right)^{T_N} \| W_{t+kT_N}^{(N)}-P^*_{N}(\overline{W}^{(N)}_{t_{N-1}+k})\|^2 + 
        \frac{8(\mu + L)\alpha}{\gamma^N \mu L} F_{\max}^2 \sigma^2 +\gamma^{-\frac{4N}{3}}  \Theta ((\sigma \Delta w_{\min})^\frac{2}{3})\nonumber
\end{align}
which completes the proof.
\end{proof}
\subsection{Proof of Lemma \ref{lemma:TT-n-descent}: Descent of lower level sequence \texorpdfstring{$W_{t_n}^{(n)}$}{Wtnn}}
\label{section:proof-lemma:TT-n-descent}

\LemmaTTWNDescentScvxCRL*
\vspace{-1.5em}
\begin{proof}[Proof of Lemma \ref{lemma:TT-n-descent}]
The proof begins from manipulating the norm $\|W_{t_{n}+1}^{(n)}-P^*_{n}(\overline{W}^{(n)}_{t_{n-1}})\|^2$:
\begin{align}
    \label{inequality:TT-convergence-scvx-A1}
    &\mathbb{E}_{\zeta_{t_n}}[\|W_{t_{n}+1}^{(n)}-P^*_{n}(\overline{W}^{(n)}_{t_{n-1}})\|^2] \\ =& \|W_ {t_{n}}^{(n)}-P^*_{n}(\overline{W}^{(n)}_{t_{n-1}})\|^2 + 2\mathbb{E}_{\zeta_{t_n}}[\langle W_{t_{n}}^n-P^*_{n}(\overline{W}^{(n)}_{t_{n-1}}), W_{t_{n}+1}^{(n)}- W_ {t_{n}}^{(n)}\rangle]  + \mathbb{E}_{\zeta_{t_n}}[\|W_{t_{n}+1}^{(n)}- W_ {t_{n}}^{(n)}\|^2]. \nonumber
\end{align}
 Substituting update dynamic \eqref{eq:TT-update-n-compact-II}, we bound the second term of \eqref{inequality:TT-convergence-scvx-A1} as following: 
\begin{align}
    \label{inequality:TT-convergence-scvx-A1-T2}
    &\ 2\mathbb{E}_{\zeta_{t_n}}[\langle W_ {t_{n}}^{(n)}-  P^*_{n}(\overline{W}^{(n)}_{t_{n-1}}),  W_{t_{n}+1}^{(n)}-W_ {t_{n}}^{(n)}\rangle]
    \\
    % ==================
    =&\  2\mathbb{E}_{\zeta_{t_n}}\Big[\Big\langle W_ {t_{n}}^{(n)}-  P^*_{n}(\overline{W}^{(n)}_{t_{n-1}}), \beta \Bigl(W_{t_{n+1}+T-1}^{(n+1)}\odot F \Bigl(\frac{W_ {t_{n}}^{(n)}}{\kappa}\Bigr) + |W_{t_{n+1}+T-1}^{(n+1)}|\odot G \Bigl(\frac{W_ {t_{n}}^{(n)}}{\kappa}\Bigr)\Bigr ) + \zeta_{t_n} \Big\rangle \Big]
    \nonumber\\
    % ==================
    \leq &\ 2\beta \Big\langle W_ {t_{n}}^{(n)}-  P^*_{n}(\overline{W}^{(n)}_{t_{n-1}}),  P^*_{n+1}(\overline{W}^{(n+1)}_{t_{n}})\odot F\Bigl(\frac{W_ {t_{n}}^{(n)}}{\kappa} \Bigr) - |P^*_{n+1}(\overline{W}^{(n+1)}_{t_{n}})|\odot G \Bigl(\frac{W_ {t_{n}}^{(n)}}{\kappa} \Bigr)\Big\rangle 
    \bkeq
    + 2\beta\mathbb{E}_{{\xi_{t:t+\mathcal{T}_n-1}}}\Big[\Big \langle W_ {t_{n}}^{(n)}-  P^*_{n}(\overline{W}^{(n)}_{t_{n-1}}), W_{t_{n+1}+T-1}^{(n+1)}\odot F\Bigl(\frac{W_ {t_{n}}^{(n)}}{\kappa} \Bigr) - |W_{t_{n+1}+T-1}^{(n+1)}|\odot G\Bigl(\frac{W_ {t_{n}}^{(n)}}{\kappa} \Bigr)    \nonumber\\
& -\Bigl(P^*_{n+1}(\overline{W}^{(n+1)}_{t_{n}})\odot F\Bigl(\frac{W_ {t_{n}}^{(n)}}{\kappa} \Bigr) - |P^*_{n+1}(\overline{W}^{(n+1)}_{t_{n}})|\odot G\Bigl(\frac{W_ {t_{n}}^{(n)}}{\kappa} \Bigr)\Big\rangle \Big].
    \nonumber
\end{align}
The last inequality holds by Young's inequality. The first term in the \ac{RHS} of \eqref{inequality:TT-convergence-scvx-A1-T2} can be bounded by:
\begin{align}
    \label{inequality:TT-convergence-scvx-A1-T2-T1}
    &\  2\beta \Big\langle W_ {t_{n}}^{(n)}-  P^*_{n}(\overline{W}^{(n)}_{t_{n-1}}),  P^*_{n+1}(\overline{W}^{(n+1)}_{t_{n}})\odot F\Bigl(\frac{W_ {t_{n}}^{(n)}}{\kappa} \Bigr) - |P^*_{n+1}(\overline{W}^{(n+1)}_{t_{n}})|\odot G\Bigl(\frac{W_ {t_{n}}^{(n)}}{\kappa} \Bigr) \Big\rangle  \nonumber  \\
     =&\ 2\beta \Bigg\langle (W_ {t_{n}}^{(n)}-P^*_{n}(\overline{W}^{(n)}_{t_{n-1}}))\odot \sqrt{F\Bigl(\frac{W_ {t_{n}}^{(n)}}{\kappa} \Bigr)}, \frac{P^*_{n+1}(\overline{W}^{(n+1)}_{t_{n}})\odot F\Bigl(\frac{W_ {t_{n}}^{(n)}}{\kappa} \Bigr) - |P^*_{n+1}(\overline{W}^{(n+1)}_{t_{n}})|\odot G\Bigl(\frac{W_ {t_{n}}^{(n)}}{\kappa} \Bigr)}{\sqrt{F\Bigl(\frac{W_ {t_{n}}^{(n)}}{\kappa} \Bigr)}}\Bigg\rangle
    \nonumber\\
    % ========================
    \eqmark{a}&\ -\frac{2\beta}{\gamma} \Bigg\langle (W_ {t_{n}}^{(n)}-P^*_{n}(\overline{W}^{(n)}_{t_{n-1}}))\odot \sqrt{F\Bigl(\frac{W_ {t_{n}}^{(n)}}{\kappa} \Bigr)}, (W_ {t_{n}}^{(n)}-P^*_{n}(\overline{W}^{(n)}_{t_{n-1}}))\odot \sqrt{F\Bigl(\frac{W_ {t_{n}}^{(n)}}{\kappa} \Bigr)} \Bigg\rangle
	\bkeq
	+ \frac{2\beta}{\gamma}\Bigg\langle (W_ {t_{n}}^{(n)}-P^*_{n}(\overline{W}^{(n)}_{t_{n-1}}))\odot \sqrt{F\Bigl(\frac{W_ {t_{n}}^{(n)}}{\kappa} \Bigr)}, |W_ {t_{n}}^{(n)}-P^*_{n}(\overline{W}^{(n)}_{t_{n-1}})|\odot \frac{G\Bigl(\frac{W_ {t_{n}}^{(n)}}{\kappa} \Bigr)}{\sqrt{F\Bigl(\frac{W_ {t_{n}}^{(n)}}{\kappa} \Bigr)}}\Bigg\rangle
    \nonumber\\
    % ========================
    \eqmark{b}&\ - \frac{\beta}{\gamma} \Bigg\|(W_ {t_{n}}^{(n)}-P^*_{n}(\overline{W}^{(n)}_{t_{n-1}}))  \odot \sqrt{F\Bigl(\frac{W_ {t_{n}}^{(n)}}{\kappa} \Bigr)}\Bigg\|^2 
	+ \frac{\beta}{\gamma} \Bigg\||W_ {t_{n}}^{(n)}-P^*_{n}(\overline{W}^{(n)}_{t_{n-1}})| \odot \frac{G\Bigl(\frac{W_ {t_{n}}^{(n)}}{\kappa} \Bigr)}{\sqrt{F\Bigl(\frac{W_ {t_{n}}^{(n)}}{\kappa} \Bigr)}}\Bigg\|^2 
    \bkeq
	- \frac{\beta}{\gamma} \Bigg\|(W_ {t_{n}}^{(n)}-P^*_{n}(\overline{W}^{(n)}_{t_{n-1}})) \odot \sqrt{F\Bigl(\frac{W_ {t_{n}}^{(n)}}{\kappa} \Bigr)} + |W_ {t_{n}}^{(n)}-P^*_{n}(\overline{W}^{(n)}_{t_{n-1}})|\odot \frac{G\Bigl(\frac{W_ {t_{n}}^{(n)}}{\kappa} \Bigr)}{\sqrt{F\Bigl(\frac{W_ {t_{n}}^{(n)}}{\kappa} \Bigr)}} \Bigg\|^2 
    \nonumber\\
    % ========================
    \lemark{c}&\ - \frac{\beta}{\gamma F_{\max}} \|W_ {t_{n}}^{(n)}-P^*_{n}(\overline{W}^{(n)}_{t_{n-1}})\|^2_{H(W_ {t_{n}}^{(n)})}  \nonumber\\
	&- \frac{\beta}{\gamma}\Bigg\| (W_ {t_{n}}^{(n)}-P^*_{n}(\overline{W}^{(n)}_{t_{n-1}})) \odot \sqrt{F\Bigl(\frac{W_ {t_{n}}^{(n)}}{\kappa} \Bigr)} + |W_ {t_{n}}^{(n)}-P^*_{n}(\overline{W}^{(n)}_{t_{n-1}})|\odot \frac{G\Bigl(\frac{W_ {t_{n}}^{(n)}}{\kappa} \Bigr)}{\sqrt{F\Bigl(\frac{W_ {t_{n}}^{(n)}}{\kappa} \Bigr)}} \Bigg\|^2 
    \nonumber\\
    % ========================
   \lemark{d} &\  - \frac{\beta}{ F_{\max}\gamma} \|W_ {t_{n}}^{(n)}-  P^*_{n}(\overline{W}^{(n)}_{t_{n-1}})\|^2_{H(W_ {t_{n}}^{(n)})}
	  \nonumber\\ & - \frac{\beta\gamma}{F_{\max}} \lnorm P^*_{n+1}(\overline{W}^{(n+1)}_{t_{n}})\odot F\Bigl(\frac{W_ {t_{n}}^{(n)}}{\kappa} \Bigr) - |P^*_{n+1}(\overline{W}^{(n+1)}_{t_{n}})|\odot G\Bigl(\frac{W_ {t_{n}}^{(n)}}{\kappa} \Bigr) \rnorm^2 
\end{align}
where $(a)$ holds by $P^*_{n+1}(\overline{W}^{(n+1)}_{t_{n}}) =\frac{W_ {t_{n}}^{(n)}-P^*_{n}(\overline{W}^{(n)}_{t_{n-1}})}{\gamma}$, $(b)$ leverages the equality $2\langle U, V\rangle = \|U\|^2+\|V\|^2-\|U-V\|^2$ for any $U, V\in\reals^D$, 
$(c)$ is achieved by the saturation vector \( H(W_ {t_{n}}^{(n)}) \in \mathbb{R}^{D} \) defined in \eqref{eq:saturation_vector}.
Thus:
\begin{align}
    & - \frac{\beta}{\gamma} \|(W_ {t_{n}}^{(n)}-P^*_{n}(\overline{W}^{(n)}_{t_{n-1}}))  \odot \sqrt{F\Bigl(\frac{W_ {t_{n}}^{(n)}}{\kappa} \Bigr)}\|^2 
	+ \frac{\beta}{\gamma} \Big\| |W_ {t_{n}}^{(n)}-P^*_{n}(\overline{W}^{(n)}_{t_{n-1}})| \odot \frac{G\Bigl(\frac{W_ {t_{n}}^{(n)}}{\kappa} \Bigr)}{\sqrt{F\Bigl(\frac{W_ {t_{n}}^{(n)}}{\kappa} \Bigr)}}\Big\|^2 \nonumber
\\ & 
    = -\frac{\beta}{\gamma} \sum_{d \in [D]} \left( [(W_ {t_{n}}^{(n)}-P^*_{n}(\overline{W}^{(n)}_{t_{n-1}}))]_d^2 \left( [F\Bigl(\frac{W_ {t_{n}}^{(n)}}{\kappa} \Bigr)]_d - \frac{[G\Bigl(\frac{W_ {t_{n}}^{(n)}}{\kappa} \Bigr)]_d^2}{[F\Bigl(\frac{W_ {t_{n}}^{(n)}}{\kappa} \Bigr)]_d} \right) \right)
\nonumber\\ & 
    = -\frac{\beta}{\gamma} \sum_{d \in [D]} \left( [(W_ {t_{n}}^{(n)}-P^*_{n}(\overline{W}^{(n)}_{t_{n-1}}))]_d^2 \left( \frac{[F\Bigl(\frac{W_ {t_{n}}^{(n)}}{\kappa} \Bigr)]_d^2 - [G\Bigl(\frac{W_ {t_{n}}^{(n)}}{\kappa} \Bigr)]_d^2}{[F\Bigl(\frac{W_ {t_{n}}^{(n)}}{\kappa} \Bigr)]_d} \right) \right)
\nonumber\\ & 
    \leq -\frac{\beta}{\gamma F_{\max}} \sum_{d \in [D]} \left( [(W_ {t_{n}}^{(n)}-P^*_{n}(\overline{W}^{(n)}_{t_{n-1}}))]_d^2 \left( [F\Bigl(\frac{W_ {t_{n}}^{(n)}}{\kappa} \Bigr)]_d^2 - [G\Bigl(\frac{W_ {t_{n}}^{(n)}}{\kappa} \Bigr)]_d^2 \right) \right) \nonumber\\ & 
    = -\frac{\beta}{\gamma F_{\max}} \|(W_ {t_{n}}^{(n)}-P^*_{n}(\overline{W}^{(n)}_{t_{n-1}})) \|^2_{H(W_ {t_{n}}^{(n)})}.
\end{align}
$(d)$ comes from :
\begin{align}
    &\ - \frac{\beta}{\gamma} \Big\|(W_ {t_{n}}^{(n)}-P^*_{n}(\overline{W}^{(n)}_{t_{n-1}})) \odot \sqrt{F\Bigl(\frac{W_ {t_{n}}^{(n)}}{\kappa} \Bigr)} + |W_ {t_{n}}^{(n)}-P^*_{n}(\overline{W}^{(n)}_{t_{n-1}})|\odot \frac{G\Bigl(\frac{W_ {t_{n}}^{(n)}}{\kappa} \Bigr)}{\sqrt{F\Bigl(\frac{W_ {t_{n}}^{(n)}}{\kappa} \Bigr)}} \Big\|^2  \nonumber
    \\
    % ==================
    =&\ - \beta\gamma \Big\| \Bigl(F\Bigl(\frac{W_ {t_{n}}^{(n)}}{\kappa} \Bigr)\Bigr)^{(-\frac{1}{2})}\odot\Bigl(\frac{W_ {t_{n}}^{(n)}-P^*_{n}(\overline{W}^{(n)}_{t_{n-1}})}{\gamma} \odot {F\Bigl(\frac{W_ {t_{n}}^{(n)}}{\kappa} \Bigr)}  \nonumber\\
    % ==================
  & \quad + \big|\frac{W_ {t_{n}}^{(n)}-P^*_{n}(\overline{W}^{(n)}_{t_{n-1}})}{\gamma}\big|\odot G\Bigl(\frac{W_ {t_{n}}^{(n)}}{\kappa} \Bigr)\Bigr) \Big\|^2 
    \nonumber\\
    % ==================
    \le&\ - \frac{\beta\gamma}{F_{\max}} \lnorm P^*_{n+1}(\overline W_{t_{n}}^{(n+1)}) \odot F\Bigl(\frac{W_ {t_{n}}^{(n)}}{\kappa} \Bigr) - |P^*_{n+1}(\overline W_{t_{n}}^{(n+1)})|\odot G\Bigl(\frac{W_ {t_{n}}^{(n)}}{\kappa} \Bigr) \rnorm^2.
\end{align}
The second term in the \ac{RHS} of \eqref{inequality:TT-convergence-scvx-A1-T2} is bounded by  Lemma \ref{lemma:lip-analog-update} as:
\begin{align} \label{inequality:TT-convergence-scvx-A1-T2-T2}
    &\ 2\mathbb{E}_{\zeta_{t_n}}\Big[\Big\langle W_ {t_{n}}^{(n)}-  P^*_{n}(\overline{W}^{(n)}_{t_{n-1}}), W_{t_{n+1}+T-1}^{(n+1)}\odot F\Bigl(\frac{W_ {t_{n}}^{(n)}}{\kappa} \Bigr) - |W_{t_{n+1}+T-1}^{(n+1)}|\odot G\Bigl(\frac{W_ {t_{n}}^{(n)}}{\kappa} \Bigr)      \nonumber \\
& \quad -\Bigl(P^*_{n+1}(\overline{W}^{(n+1)}_{t_{n}})\odot F\Bigl(\frac{W_ {t_{n}}^{(n)}}{\kappa} \Bigr) - |P^*_{n+1}(\overline{W}^{(n+1)}_{t_{n}})|\odot G\Bigl(\frac{W_ {t_{n}}^{(n)}}{\kappa} \Bigr)\Bigr) \Big\rangle \Big]
   \\
    % ==================
    \le&\ 
    \frac{\beta}{2 F_{\max} \gamma}\|W_ {t_{n}}^{(n)}-  P^*_{n}(\overline{W}^{(n)}_{t_{n-1}})\|^2_{H(W_ {t_{n}}^{(n)})}
     + 2\beta F_{\max}\gamma \Big\| W_{t_{n+1}+T-1}^{(n+1)}\odot F\Bigl(\frac{W_ {t_{n}}^{(n)}}{\kappa} \Bigr)-|W_{t_{n+1}+T-1}^{(n+1)}| \nonumber \\ & \odot G\Bigl(\frac{W_ {t_{n}}^{(n)}}{\kappa} \Bigr) 
    % ==================
     -\Bigl(P^*_{n+1}(\overline{W}^{(n+1)}_{t_{n}})\odot F\Bigl(\frac{W_ {t_{n}}^{(n)}}{\kappa} \Bigr) - |P^*_{n+1}(\overline{W}^{(n+1)}_{t_{n}})|\odot G\Bigl(\frac{W_ {t_{n}}^{(n)}}{\kappa} \Bigr)\Bigr)\Big\|^2_{H(W_ {t_{n}}^{(n)})^\dag}
    \nonumber    \\
    % ==================
    \le&\ 
    \frac{\beta}{2 F_{\max} \gamma}\|W_ {t_{n}}^{(n)}-  P^*_{n}(\overline{W}^{(n)}_{t_{n-1}})\|^2_{H(W_ {t_{n}}^{(n)})}
    + 2\beta F_{\max}^3\gamma \| W_{t_{n+1}+T-1}^{(n+1)} - P^*( \overline W_{t_n}^{(n+1)})\|^2_{H(W_ {t_{n}}^{(n)})^\dag}.
    \nonumber
\end{align}
    Plugging inequality \eqref{inequality:TT-convergence-scvx-A1-T2-T1}  and \eqref{inequality:TT-convergence-scvx-A1-T2-T2} above into \eqref{inequality:TT-convergence-scvx-A1-T2} yields:
\begin{align}
  \label{inequality:TT-convergence-scvx-A1-T2-final}
   &\quad 2 \mathbb{E}_{\zeta_{t_n}}[\langle W_ {t_{n}}^{(n)}-P^*_{n}(\overline{W}^{(n)}_{t_{n-1}}), W_{t_{n}+1}^{(n)}- W_ {t_{n}}^{(n)}\rangle] \\& \nonumber \leq  2\beta F_{\max}^3\gamma \| W_{t_{n+1}+T-1}^{(n+1)} -  P^*( \overline W_{t_n}^{(n+1)})\|^2_{H(W_ {t_{n}}^{(n)})^\dag} -\frac{\beta\gamma}{F_{\max}}\Big\| P^*_{n+1}(\overline{W}^{(n+1)}_{t_{n}})\odot F\Bigl(\frac{W_ {t_{n}}^{(n)}}{\kappa} \Bigr) \bkeq \quad - |P^*_{n+1}(\overline{W}^{(n+1)}_{t_{n}})|\odot G\Bigl(\frac{W_ {t_{n}}^{(n)}}{\kappa} \Bigr)\Big\|^2 -\frac{\beta }{2\gamma  F_{\max}} \|W_ {t_{n}}^{(n)}-  P^*_{n}(\overline{W}^{(n)}_{t_{n-1}})\|^2_{H(W_ {t_{n}}^{(n)})} 
 .\nonumber
\end{align}
We assume there exists a non-zero constant \( H_{\min}\) such that {\(\min \{ H(W_ {t_{n}}^{(n)}) \} \geq H_{\min}\)} for all \(t_{n}\) and $n$. Under this condition, we have the following inequalities:
\begin{align}
 \label{inequality:TT-convergence-scvx-H1}
    -\frac{\beta}{2 F_{\max} \gamma}\|W_ {t_{n}}^{(n)}-  P^*_{n}(\overline{W}^{(n)}_{t_{n-1}})\|^2_{H(W_ {t_{n}}^{(n)})} &\leq -\frac{\beta H_{\min}}{2 F_{\max} \gamma}\|W_ {t_{n}}^{(n)}-  P^*_{n}(\overline{W}^{(n)}_{t_{n-1}})\|^2,  \\
       2\beta F_{\max}^3\gamma\lnorm W_{t_{n+1}+T-1}^{(n+1)} -  P^*( \overline W_{t_n}^{(n+1)})\rnorm^2_{H(W_ {t_{n}}^{(n)})^\dag} &\leq  \frac{2\beta F_{\max}^3\gamma}{H_{\min}}\lnorm W_{t_{n+1}+T-1}^{(n+1)} -  P^*( \overline W_{t_n}^{(n+1)})\rnorm^2. \nonumber
    \end{align}
    Plugging inequality \eqref{inequality:TT-convergence-scvx-H1} above into \eqref{inequality:TT-convergence-scvx-A1-T2-final} yields:
\begin{align}
  \label{inequality:TT-convergence-scvx-H1-final}
   &\quad 2\mathbb{E}_{\zeta_{t_n}}[\langle W_ {t_{n}}^{(n)}-P^*_{n}(\overline{W}^{(n)}_{t_{n-1}}), W_{t_{n}+1}^{(n)}- W_ {t_{n}}^{(n)}\rangle] \\& \nonumber \leq  -\frac{\beta H_{\min}}{2\gamma  F_{\max}} \|W_ {t_{n}}^{(n)}-  P^*_{n}(\overline{W}^{(n)}_{t_{n-1}})\|^2
    +\frac{2\beta F_{\max}^3\gamma}{H_{\min}} \| W_{t_{n+1}+T-1}^{(n+1)} - P^*( \overline W_{t_n}^{(n+1)})\|^2\\& \nonumber \quad  -\frac{\beta\gamma}{F_{\max}} \Big\| P^*_{n+1}(\overline{W}^{(n+1)}_{t_{n}})\odot F\Bigl(\frac{W_ {t_{n}}^{(n)}}{\kappa} \Bigr) - |P^*_{n+1}(\overline{W}^{(n+1)}_{t_{n}})|\odot G\Bigl(\frac{W_ {t_{n}}^{(n)}}{\kappa} \Bigr) \Big\|^2.
\end{align}
The third term in the \ac{RHS} of \eqref{inequality:TT-convergence-scvx-A1} is bounded by Lemma \ref{lemma:lip-analog-update}
as: 
\begin{align}
    \label{inequality:TT-convergence-scvx-A1-T3}
    &\ \mathbb{E}_{\zeta_{t_n}}[\|W_{t_{n}+1}^{(n)}-W_ {t_{n}}^{(n)}\|^2]\\
    = &\mathbb{E}_{\zeta_{t_n}}\Big[ \Big\|\beta \Bigl(W_{t_{n+1}+T-1}^{(n+1)}\odot F\Bigl(\frac{W_ {t_{n}}^{(n)}}{\kappa} \Bigr) - |W_{t_{n+1}+T-1}^{(n+1)}|\odot G\Bigl(\frac{W_ {t_{n}}^{(n)}}{\kappa} \Bigr)\Bigr)+ \zeta_{t_n}\Big\|^2 \Big] \nonumber
    \\ 
    % ==================
    \le&\ 3\beta^2 \Big\|P^*_{n+1}(\overline{W}^{(n+1)}_{t_{n}})\odot F\Bigl(\frac{W_ {t_{n}}^{(n)}}{\kappa} \Bigr) - |P^*_{n+1}(\overline{W}^{(n+1)}_{t_{n}})|\odot G\Bigl(\frac{W_ {t_{n}}^{(n)}}{\kappa} \Bigr)\Big\|^2  + 3\beta \Theta (\Delta w_{\min}) \nonumber \\
    & + 3\beta^2 
\Big\|W_{t_{n+1}+T-1}^{(n+1)}\odot F\Bigl(\frac{W_ {t_{n}}^{(n)}}{\kappa} \Bigr) - |W_{t_{n+1}+T-1}^{(n+1)}|\odot G\Bigl(\frac{W_ {t_{n}}^{(n)}}{\kappa} \Bigr)-\Bigl(P^*_{n+1}(\overline{W}^{(n+1)}_{t_{n}})\odot F\Bigl(\frac{W_ {t_{n}}^{(n)}}{\kappa} \Bigr) \nonumber \\
    &- |P^*_{n+1}(\overline{W}^{(n+1)}_{t_{n}})|\odot G\Bigl(\frac{W_ {t_{n}}^{(n)}}{\kappa} \Bigr)\Bigr)\Big\|^2 
    \nonumber\\
    % ==================
    \le&\ 3\beta^2 \Big\|P^*_{n+1}(\overline{W}^{(n+1)}_{t_{n}})\odot F\Bigl(\frac{W_ {t_{n}}^{(n)}}{\kappa} \Bigr) - |P^*_{n+1}(\overline{W}^{(n+1)}_{t_{n}})|\odot G\Bigl(\frac{W_ {t_{n}}^{(n)}}{\kappa} \Bigr)\Big\|^2 +3\beta \Theta (\Delta w_{\min})
     \bkeq + 3\beta^2 \mathbb{E}_{{\xi_{t:t+\mathcal{T}_n-1}}}[\|W_{t_{n+1}+T-1}^{(n+1)} - P^*_{n+1}(\overline{W}^{(n+1)}_{t_{n}})\|^2]. 
    \nonumber
\end{align}
The second inequality holds by Cauchy-Schwarz inequality. 
 Plugging inequality \eqref{inequality:TT-convergence-scvx-H1-final}  and \eqref{inequality:TT-convergence-scvx-A1-T3} above into \eqref{inequality:TT-convergence-scvx-A1} yields:
 \begin{align}
    \label{inequality:TT-convergence-scvx-A1-FINAL-FINAL}
&\quad\mathbb{E}_{\zeta_{t_n}}[\|W_{t_{n}+1}^{(n)}-P^*_{n}(\overline{W}^{(n)}_{t_{n-1}})\|^2] \\& =\nonumber(1-\frac{\beta H_{\min}}{2\gamma  F_{\max}})\|W_ {t_{n}}^{(n)}-  P^*_{n}(\overline{W}^{(n)}_{t_{n-1}})\|^2  
    + \frac{2\beta F_{\max}^3 \gamma }{H_{\min}} \| W_{t_{n+1}+T-1}^{(n+1)} - P^*_{n+1}(\overline{W}^{(n+1)}_{t_{n}})\|^2  \\& \nonumber \quad  -(\frac{\beta\gamma}{F_{\max}}-3\beta^2)\Big \| P^*_{n+1}(\overline{W}^{(n+1)}_{t_{n}})\odot F\Bigl(\frac{W_ {t_{n}}^{(n)}}{\kappa} \Bigr) - |P^*_{n+1}(\overline{W}^{(n+1)}_{t_{n}})|\odot G\Bigl(\frac{W_ {t_{n}}^{(n)}}{\kappa} \Bigr) \Big\|^2 \nonumber  \\ & \quad+3\beta \Theta (\Delta w_{\min})  + 3\beta^2  \| W_{t_{n+1}+T-1}^{(n+1)} - P^*_{n+1}(\overline{W}^{(n+1)}_{t_{n}})\|^2\nonumber  \\ &\leq  ( 1-\frac{\beta H_{\min}}{2\gamma  F_{\max}})\|W_ {t_{n}}^{(n)}-  P^*_{n}(\overline{W}^{(n)}_{t_{n-1}})\|^2  
    + \frac{3\beta F_{\max}^3 \gamma }{H_{\min}}\|W_{t_{n+1}+T-1}^{(n+1)} - P^*_{n+1}(\overline{W}^{(n+1)}_{t_{n}})\|^2 \nonumber \\
    &\quad +\beta \Theta (\Delta w_{\min}) -(\frac{\beta\gamma}{F_{\max}}-3\beta^2)\Big \| P^*_{n+1}(\overline{W}^{(n+1)}_{t_{n}})\odot F\Bigl(\frac{W_ {t_{n}}^{(n)}}{\kappa} \Bigr) - |P^*_{n+1}(\overline{W}^{(n+1)}_{t_{n}})|\odot G\Bigl(\frac{W_ {t_{n}}^{(n)}}{\kappa} \Bigr) \Big\|^2 . \nonumber
\end{align}
The last inequality holds by setting  $\beta \leq  \frac{F_{\max}^3 \gamma}{3H_{\min}}$.
Executing $T_n$ iterations through \eqref{inequality:TT-convergence-scvx-A1-FINAL-FINAL} yields:
\begin{align}
\label{eq:101}
   & \quad \mathbb{E}_{\boldsymbol{\zeta_n}}[\|W_{ t_{n}+T_n-1}^n-P^*_{n}(\overline{W}^{(n)}_{t_{n-1}})\|^2]
   \\ &  \leq \left( 1-\frac{\beta H_{\min}}{2\gamma  F_{\max}}\right)^{T_n}\|W_ {t_{n}}^{(n)}-  P^*_{n}(\overline{W}^{(n)}_{t_{n-1}})\|^2  
    + \sum_{i=0}^{T_n-1}\left( 1-\frac{\beta H_{\min}}{2\gamma  F_{\max}}\right)^i (\beta \Theta (\Delta w_{\min})\bkeq \quad +\frac{3\beta F_{\max}^3  \gamma}{H_{\min}} \| W_{t_{n+1}+(i+1)T_{n+1}-1}^{n+1}- P^*_{n+1}( \overline W_{t_n+i}^{(n+1)})\|^2 ) \nonumber  \\ &  \leq \left( 1-\frac{\beta H_{\min}}{2\gamma  F_{\max}}\right)^{T_n}\|W_ {t_{n}}^{(n)}-  P^*_{n}(\overline{W}^{(n)}_{t_{n-1}})\|^2   + \frac{6 F_{\max}^4  \gamma^2}{H_{\min}^2}\| W_{t_{n+1}+T_{n+1}-1}^{(n+1)}- P^*_{n+1}(\overline{W}^{(n+1)}_{t_{n}})\|^2\bkeq\quad-(\frac{2\gamma^2}{H_{\min}}-\frac{6\beta\gamma F_{\max}}{H_{\min}})\Big \| P^*_{n+1}(\overline{W}^{(n+1)}_{t_{n}})\odot F\Bigl(\frac{W_ {t_{n}}^{(n)}}{\kappa} \Bigr) - |P^*_{n+1}(\overline{W}^{(n+1)}_{t_{n}})|\odot G\Bigl(\frac{W_ {t_{n}}^{(n)}}{\kappa} \Bigr) \Big\|^2\bkeq \quad +\frac{2F_{\max}\Theta (\Delta w_{\min}) }{H_{\min}} . \nonumber
\end{align}
When $k \neq 0$, \eqref{eq:101} can be written as the general case: \begin{align}
   & \quad \mathbb{E}_{\boldsymbol{\zeta_n}}[\|W_{ t_{n}+(k+1)T_n-1}^{(n)}-P^*_{n}(\overline{W}^{(n)}_{t_{n-1}+k})\|^2]
  \\ &  \leq \left( 1-\frac{\beta H_{\min}}{2\gamma  F_{\max}}\right)^{T_n}\|W_ {t_{n}+kT_n}^{(n)}-  P^*_{n}(\overline{W}^{(n)}_{t_{n-1}+k})\|^2\bkeq \quad  + \frac{6 F_{\max}^4  \gamma^2}{H_{\min}^2}\| W_{t_{n+1}+(k+1)T_{n+1}-1}^{(n+1)}- P^*_{n+1}(\overline{W}^{(n+1)}_{t_{n}+k})\|^2 +\frac{2F_{\max}\Theta (\Delta w_{\min}) }{H_{\min}}\bkeq \quad-(\frac{2\gamma^2}{H_{\min}} -\frac{6\beta\gamma F_{\max}}{H_{\min}})\Big \|P^*_{n+1}(\overline{W}^{(n+1)}_{t_{n}+k})\odot F\Bigl(\frac{W_ {t_{n}}^{(n)}}{\kappa} \Bigr) - |P^*_{n+1}(\overline{W}^{(n+1)}_{t_{n}+k})|\odot G\Bigl(\frac{W_ {t_{n}}^{(n)}}{\kappa} \Bigr) \Big\|^2 \nonumber 
\end{align}
which completes the proof.
\end{proof}

\newpage
\section{Pseudocode}
\label{Sec:Pseudocode}
\begin{algorithm}[H]
\small
\caption{Multi-timescale Residual Learning with Warm Start Initialization}
\label{alg:bit-slice-tt-switch}
\begin{algorithmic}[1]
\STATE Initialize $W^{(0)}$ with the weight from digital side, working tiles $W^{(n)} \gets 0$ for $n=1,\dots,N$
\STATE Initialize tile index counter: $t_n \gets 0$ for all $n$, and set \textbf{current update tile} $k \gets 0$
\STATE Initialize inner loop length: $T_n = \textit{Transfer\_every\_vec}[n]$
\STATE Initialize switching flag: $\texttt{trigger\_tile\_switch} \gets \texttt{False}$
\STATE Initialize loss history buffer: $\mathcal{L}$
\FOR{each iteration $t=1,2,\dots$}
    \STATE $ W_{t_{N}}^{(N)} \gets W_{t}^{(N)}
     - \alpha \nabla f(\overline{W}_{t};\xi_{t}) \odot F \bigl(W_{t}^{(N)}\bigr)
     - \bigl|\alpha \nabla f(\overline{W}_{t};\xi_{t})\bigr|
       \odot G \bigl(W_{t}^{(N)}\bigr)
     + \zeta_{t}$ 
   \\ \textcolor{gray}{// Gradient accumulation}
    \STATE Append $\ell_{t}$ to loss history $\mathcal{L}$
    \IF{\texttt{LossPlateau}$(\mathcal{L}, k)$ \textbf{and} $k \geq 0$}
        \STATE $\texttt{trigger\_tile\_switch} \gets \texttt{True}$ \hfill \textcolor{gray}{// Detected plateau: trigger tile switch}
    \ENDIF
\IF{$\texttt{trigger\_tile\_switch} = \texttt{True}$ \textbf{and} $k < N$}
        \STATE $k \gets k + 1$ \hfill \textcolor{gray}{// Progressive per-tile transfer switch}
        \STATE $\texttt{trigger\_tile\_switch} \gets \texttt{False}$
    \ENDIF

    \IF{$k < N$ \textbf{and} $t \bmod T_N = 0$}
        \STATE $W^{(k)} _{t_k} \gets  W_{t_k}^{(k)}
     + \beta  W^{(N)}_{t+T_N-1}\odot F \bigl(W_{t_k}^{(k)}\bigr)
     - \bigl|\beta W^{(N)}_{t+T_N-1}\bigr|
       \odot G \bigl(W_{t_k}^{(k)}\bigr)
     + \zeta_{t_k}$ 
\\ \textcolor{gray}{// Transfer update from $W^{(N)}$ to $W^{(k)}$}
        % \STATE $t_k \gets t_k+1$
    \ENDIF
     \hfill \textcolor{gray}{// Warm start initialization finished}
    \IF{$k \geq N$}      
        \FOR{$n = N-1$ \TO $0$} 
            \IF{$t_{n+1} \bmod T_{n+1} = 0$}
                \STATE $W^{(n)}_{t_n} \gets W_{t_n}^{(n)}
     + \beta  \Tilde W^{(n+1)}\odot F \bigl(W_{t_n}^{(n)}\bigr)
     - \bigl|\beta  \Tilde W^{(n+1)}\bigr|
       \odot G \bigl(W_{t_n}^{(n)}\bigr)
     + \zeta_{t_n}$   \\ \textcolor{gray}{// Transfer update from $W^{(n+1)}$ to $W^{(n)}$}
            \ENDIF
        \ENDFOR
    \ENDIF
    \STATE $\overline{W}_{t} = \sum_{n=0}^{N} \gamma^n W^{(n)}_{t_n}$ \hfill \textcolor{gray}{// Combine all tiles to form effective weight}
\ENDFOR
\vspace{1em}
\STATE \textbf{Function:} LossPlateau$(\mathcal{L}, k)$
\STATE \quad \textbf{if} $k \leq 3$: 
\STATE \quad \quad \textbf{if} $|\mathcal{L}| < 2$: \textbf{return} False \hfill \textcolor{gray}{// Not enough history}
\STATE \quad \quad \textbf{else}: \textbf{return} $\mathcal{L}[t] > \mathcal{L}[t{-}1]$ \hfill \textcolor{gray}{// Aggressive mode}
\STATE \quad \textbf{else}:  
\STATE \quad \quad \textbf{if} $|\mathcal{L}| < 6$: \textbf{return} False
\STATE \quad \quad \textbf{else}: 
\STATE \quad \quad \quad $v \gets 0$
 \STATE\quad \quad \quad \textbf{for}{ $i = t{-}5$ to $t{-}1$} \textbf{do}
    \STATE \quad \quad \quad \quad  \textbf{if} $\mathcal{L}[i+1] > \mathcal{L}[i]$: $v \gets v + 1$
 \STATE\quad \quad \quad \textbf{end for}
\STATE \quad \quad \quad \textbf{return} $v \ge 2$ \hfill \textcolor{gray}{// mild mode}
\end{algorithmic}
\end{algorithm}
\vspace{-2em}
Algorithm \ref{alg:bit-slice-tt-switch} includes an optional warm start phase (lines 1--18), which is only used in our experimental implementation to accelerate convergence and stabilize early training. This warm start is not required in the general method (Section~\ref{sec.3}) nor in the theoretical analysis (Section~\ref{Sec:Theorem}). The main results of our method should focus on the multiscale residual learning process beginning from line 19 onward.
The warm start process in this algorithm uses the gradient accumulated on the tile $W^{(N)}$ to successively update tiles $W^{(0)}, W^{(1)}, \dots, W^{(N-1)}$. Initially, only $W^{(0)}$ is initialized with the digital model weights (line 1), and the current update tile index is set to $k = 0$ (line 2). During training, $W^{(N)}$ is updated at every step using the gradient of the current composite weight $\overline{W}_{t}$ (line 6). Every $T_N$ steps, the content of $W^{(N)}$ is transferred to tile $W^{(k)}$ (line 14). This continues until the loss plateaus, as determined by the \texttt{LossPlateau} function (lines 7--9). When a plateau is detected, the algorithm increments $k$ (lines 10--13), thereby the content of $W^{(N)}$ is transferred to tile $W^{(k+1)}$. This procedure repeats until $k > N$, which means that all tiles have been updated, and then the warm start initialization is complete.

\section{Analog Circuit Implementation Details}
\label{appendix:complexity_analysis}

\begin{figure}[H]
  \centering
\includegraphics[width=0.9\linewidth]{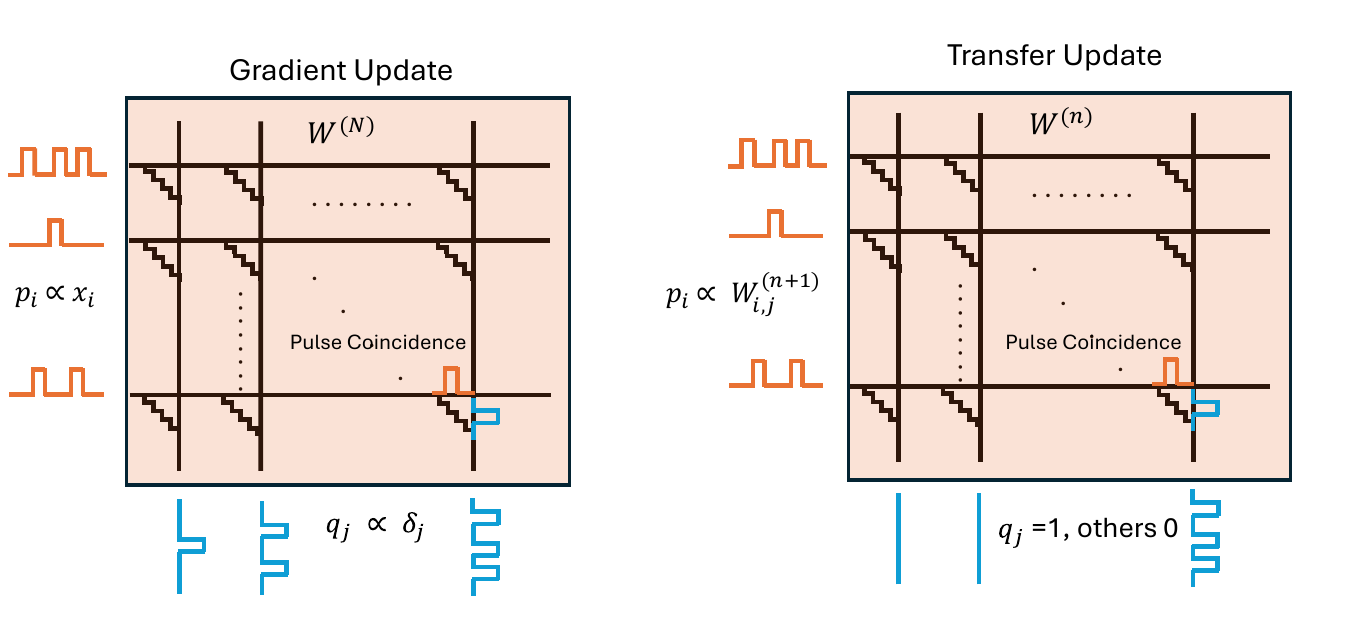}
  \vspace{-1 em}
  \caption{Implementation of gradient update and transfer update process.}
   \vspace{-1 em}
  \label{analog practice}
\end{figure}
\subsection{Circuit-level implementation of update process} 

Figure \ref{analog practice} illustrates how gradient updates are applied to tile \(W^{(N)}\) and how transfer updates are applied to the remaining tiles \(W^{(n)}\). As described in Section \ref{subsection:Pulse update}, stochastic pulse streams whose probabilities are proportional to the error \(\delta_j\) and input \(x_i\) are injected along each row and column of the crossbar array, realizing gradient updates with expectation $\Delta w_{ji} =\alpha\delta_j x_i
=\frac{\partial f(\overline{W})}{\partial w_{ji}}$. 
For each transfer update on \(W^{(n)}\), residual learning reads one column from \(W^{(n+1)}\), and the column selection may follow either a cyclic or a sequential schedule. To support the feasibility of implementing the composite weight structure in Figure~\ref{fig:analog_forward_practice}, we note that similar analog crossbar-based weight composition and accumulation schemes have been demonstrated in practice~\citep{song2024programming}. These prior works focus on inference, where the weights remain static, but this does not affect our in-memory training scenario because the ADCs are only used during forward and backward passes, which are present in both training and inference. Moreover, sharing analog peripheral circuits (e.g., ADCs, DACs, and drivers) across multiple subarrays has been adopted in several recent compute-in-memory designs to reduce power and area overhead~\citep{xu2024reharvest,xu2024cascaded}.
\subsection{Runtime and update complexity comparison.}
\begin{table}[t]
\centering
\vspace{-1em}
\resizebox{\linewidth}{!}{
\begin{tabular}{l|c|c|c|c}
\toprule
\textbf{Algorithm} & \textbf{TT-v2} & \textbf{Analog SGD} & \textbf{MP} & \textbf{Ours} \\
\midrule
Digital storage [byte] & $\mathcal{O}(D^2 + 2D)$ & $\mathcal{O}(2D)$ & $\mathcal{O}(D^2 + 2DB)$ & $\mathcal{O}(2D)$  \\
Memory ops [bit] & $\mathcal{O}(16D / n_s)$ & $\mathcal{O}(1)$ & $\mathcal{O}(16D^2 / B)$ & $\mathcal{O}(1)$ \\
FP ops & $\mathcal{O}(2D + 2D / n_s)$ & $\mathcal{O}(2D)$ & $\mathcal{O}(2D^2 + D)$ & $\mathcal{O}(2D)$\\
Analog ops [time] & $(l_{\mathrm{avg}} + \frac{1}{n_s}) t_{\mathrm{sp}} + \frac{t_{\mathrm{M}}}{n_s}$ & $l_{\mathrm{avg}}  t_{\mathrm{sp}}$ & $\frac{D}{B}   t_{\mathrm{sp}}$& $l_{\mathrm{avg}} \frac{ t_{\mathrm{sp}} n_s}{n_s-1} + \frac{ t_{\mathrm{M}}}{n_s-1}$  \\
$\approx$ Time est. [ns] & $56.3$ & $30.9$ & $3024.5$ & $95.9$ \\
\bottomrule
\end{tabular}
}
\caption{Comparison of complexity and estimated runtime for per-sample weight update. 
Here, $D$ is the vector/matrix dimension and $B$ is the mini-batch size. 
For the time estimates, we assume $D=512$, $B=100$, $n_s=2$ as the transfer period,  $l_{\mathrm{avg}}=5$ is the average number of pulses per sample, 
$t_{\mathrm{sp}}=5 ns$ is the duration of a single pulse, $t_{\mathrm{M}} = 40 ns$ is the time for matrix-vector readout and FP operations compute time is calculated assuming throughput of $0.7$ TFLOPS~\citep{jain2022heterogeneous},  Statistics for TT-v2, Analog SGD, and MP are based on~\citep{rasch2024fast}.}
\label{tab:weight-update-comparison}
\vspace{-2em}
\end{table}

As shown in Section \ref{section: Numerical Simulations} and supplement experiments in Section \ref{sec:Supplement_simulations}, our algorithm achieves superior performance under a limited-precision setting, consistently outperforming the Tiki-Taka series and in some cases performing on par with MP.
However, this performance advantage sometimes comes at the cost of using multiple analog tiles (from 3 to 8), in contrast to TT-v1, which employs only 2 tiles, TT-v2, which adds one digital tile to TT-v1, and MP, which uses a single analog tile combined with a high-precision digital unit. To evaluate the impact of using multiple tiles, we first summarize the per-sample weight update complexity of different algorithms in Table~\ref{tab:weight-update-comparison}. This table provides a unified metric for comparing the hardware load across methods, including digital storage, memory operations, floating-point operations, and analog operations. Based on these results, we then present a more detailed analysis of the specific hardware costs including digital storage, runtime, energy, and area to highlight the efficiency advantages of our proposed approach.

\textbf{Digital storage cost.} We evaluate the digital storage cost of our proposed method in comparison with TT-v2, Analog SGD and MP. Here, digital storage refers exclusively to the memory used to buffer intermediate forward input, backward error, and gradients during training, which reside in SRAM or DRAM. In analog crossbar-based training, only the input $x \in \mathbb{R}^D$ and error signal $\delta \in \mathbb{R}^D$ are digitized via ADCs and temporarily stored for the backward pass. These two vectors lead to a digital storage requirement of $O(2D)$ bytes for both Analog SGD and our method. In contrast, TT-v2 requires an additional $D \times D$ digital transfer buffer between auxiliary and core arrays, incurring a total cost of $O(2D + D^2)$. For MP, gradient accumulation is performed in the digital domain over a batch of size $B$, resulting in $O(D^2 + 2DB)$ digital storage. 

To quantify this storage in real scenarios, we extract the analog tile dimensions used in our experiments in Tables \ref{tab:num_states_comparison} and \ref{tab:num_states_comparison_cifar} and sum the corresponding vector and matrix sizes across all analog layers. Assuming each element occupies 1 byte (i.e., 8 bits of precision), we report total digital storage in kilobytes (KB). For MP, we incorporate mini-batch accumulation with batch size $B = 8$ for LeNet-5 and $B = 128$ for ResNet-18.
\begin{table}[H]
\centering
\small
\vspace{-0.5em}
\setlength{\tabcolsep}{6pt}
\renewcommand{\arraystretch}{1.2}
\begin{tabular}{l|cccc}
\toprule
\textbf{Model} & \textbf{TT-v2} & \textbf{Analog SGD} & \textbf{MP} & \textbf{Ours} \\
\midrule
LeNet-5   (KB)    & 80.2    & 2.13  & 94.8   & 2.13 \\
ResNet-18  (KB)   & 10,600  & 50.2  & 17,000 & 50.2 \\
\bottomrule
\end{tabular}
\vspace{-0.5em}
\caption{
Digital storage required by different algorithms on LeNet-5 and ResNet-18.
}
\label{tab:digital_storage_comparison}
\vspace{-2 em}
\end{table}
 Table
\ref{tab:digital_storage_comparison} demonstrates that on LeNet-5 and ResNet-18 , our method achieves digital storage reductions of $37\times$–$211\times$ compared to TT-v2, and $44\times$–$339\times$ compared to MP while matching the minimal cost of Analog SGD. This efficiency stems from the fact that increasing the number of tiles does not incur additional cost for storing $x$ and $\delta$, as they are computed collectively across all tiles as illustrated in Figure \ref{fig:analog_forward_practice}. Moreover, unlike MP and TT-v2, our method eliminates the need for any digital tile to store weights or gradients. We believe our algorithm offers a substantial advantage in terms of digital memory overhead.

\textbf{Runtime cost.} We analyze the runtime cost, which includes both the  \emph{FP operations time} and \emph{analog operations time}. 
For the MP baseline, the  outer-product $\delta x^T$ requires $D^2$ multiplications and $D^2$ additions, resulting in a total of $2D^2$ floating-point operations per input. Including additional scaling and preprocessing for $x$ and $\delta$, the total is $\mathcal{O}(2D^2 + D)$ FP ops. Dividing by the effective throughput of 0.175 TFLOPS (assuming 0.7 TFLOPS shared across 4 tiles as in \citep{rasch2024fast}), the FP operations take $2998.9 \text{ns}$, while the analog operations require $\frac{D}{B} t_{\text{sp}} = 25.6 \text{ns}$, resulting in a total estimated latency of $3024.5 \text{ns}$. In contrast, our method maintains the same $\mathcal{O}(2D)$ FP operations as Analog SGD, as it only requires computing the absolute maximum values of $x$ and $\delta$ to scale the probabilities used in stochastic pulse updates. 
For analog operations, we focus on deriving the runtime expression specific to our method, which includes the pulse update and the MVM-based readout for weight transfer. For pulse update, when the final tile $W^{(N)}$ is updated once, each preceding tile $W^{(n)}$ is updated approximately every  $(1/n_s)^{N-n}$ iterations, and the aggregate latency is bounded by $l_{\mathrm{avg}} \cdot t_{\text{sp}} \cdot \sum_{n=0}^{N} (1/n_s)^{N-n}$, whose upper bound as $N \to \infty$ converges to $\frac{l_{\mathrm{avg}} t_{\text{sp}} n_s}{n_s - 1}$. Similarly, the MVM-based readout incurs an additional delay bounded by $\frac{t_{\mathrm{M}}}{n_s - 1}$. As a result, our method’s analog latency is $l_{\mathrm{avg}} \cdot \frac{t_{\mathrm{sp}} n_s}{n_s - 1} + \frac{t_{\mathrm{M}}}{n_s - 1}$, which gives a total estimated latency of $95.9\text{ns}$ in Table \ref{tab:weight-update-comparison}.

We apply the same methodology to estimate runtime on full-model configurations. For both LeNet-5 and ResNet-18, we assume each layer is processed in parallel, and the slowest layer dominates the total latency. For MP, the largest matrix in LeNet-5 is of size $128 \times 512$, resulting in a total latency of $457.4 \text{ns}$. In the case of ResNet-18, the largest analog matrix is $512 \times 4608$, leading to an FP time of $13508.0 \text{ns}$ and an analog latency of $20.0 \text{ns}$ (with batch size $B = 128$), yielding a total latency of $13528.0 \text{ns}$ We summarize the runtime across all algorithms and architectures in Table \ref{tab:latencycomparison} below:

\begin{table}[H]
\centering
\small
\vspace{-0.5em}
\setlength{\tabcolsep}{6pt}
\renewcommand{\arraystretch}{1.2}
\begin{tabular}{l|cccc}
\toprule
\textbf{Model} & \textbf{TT-v2} & \textbf{Analog SGD} & \textbf{MP} & \textbf{Ours} \\
\midrule
LeNet-5 (ns)   & 56.3   & 30.9   & 457.4   & 95.9 \\
ResNet-18 (ns) & 126.5  & 77.7   & 13528.0 & 142.7 \\
\bottomrule
\end{tabular}
\vspace{-0.5em}
\caption{Estimated runtime for different analog training methods on LeNet-5 and ResNet-18.}
\label{tab:latencycomparison}
\vspace{-2 em}
\end{table}

These results highlight the efficiency of our method: on LeNet-5, MP is 4.8× slower than ours, while on ResNet-18, it is 94.7× slower. At the same time, our method introduces only modest latency overhead compared to TT-v2.

\textbf{Energy cost.} We analyze the  energy consumption per training sample  and compare our method to the MP baseline. As reported in \citep[Table 1]{le2018mixed}, MP consumes $83.2\text{nJ}$ to process a single training image on a two-layer perceptron (785 inputs, 250 hidden neurons, 10 outputs; total 198,760 synapses). For consistency, we adopt the same model configuration for energy estimation across methods.
For our method, the energy is composed of three parts: pulse update energy, transfer update energy, and forward/backward propagation energy. To estimate the  pulse update energy , we follow \citep{gokmen2016acceleration}, which reports a combined power consumption of $0.7\text{W}$ for op-amps and stochastic translators (STRs) on a $4096 \times 4096$ crossbar. We scale this power by the number of active rows and columns in our perceptron, yielding $P_{\text{scaled}} = 0.1107\text{W}$. Assuming a worst-case pulse update time of $\frac{l_{\mathrm{avg}}  t_{\mathrm{sp}}  n_s}{n_s - 1} = 50\text{ns}$, we obtain the pulse update energy:
$
E_{\text{pulse\_update}} = P_{\text{scaled}} \cdot 50 \text{ns} \approx 5.53 \text{nJ}.
$
In addition, we consider the  transfer update energy  for analog readout. Since in our architecture each tile is read roughly every $n_s$ steps, the aggregate readout time is upper bounded by the MVM latency of a single tile, $\frac{t_{\mathrm{M}}}{n_s - 1} = 40\text{ns}$. The MP work reports a  forward-propagation energy in computational memory of $7.29 \text{nJ}$ under similar conditions, which can be interpreted as the estimated energy cost for a single-tile analog forward pass closely matching the MVM read time in our design. Therefore, we adopt $7.29 \text{nJ}$ as an upper-bound estimate for our total transfer energy:
$
E_{\text{trans}} = 7.29 \text{nJ}.
$
Hence, the  total update energy per sample  in our method is the sum:
$$
E_{\text{update}} = E_{\text{pulse\_update}} + E_{\text{trans}} \approx 12.82 \text{nJ}.
$$
For the forward and backward passes, our method distributes the computation across $N$ analog tiles. Each tile requires the full sequence of operations including data input, PWM generation, read-voltage regulation, analog computation,  ADC conversion and data output. We adopt a conservative estimation by assuming that these operations are not shared across tiles, so the energy cost of each tile is independent and must be incurred individually. Following the values reported in \citep{le2018mixed}, the forward and backward passes consume $7.29 \text{nJ}$ and $2.15 \text{nJ}$ per tile, respectively. Consequently, the total propagation energy scales linearly with $N$ as $N \cdot (7.29 + 2.15) = N \cdot 9.44\text{nJ}$.
Putting all the components together, we compare the energy consumption per training sample in Table
\ref{tab:energy_breakdown}:
\begin{table}[H]
\centering
\small
\setlength{\tabcolsep}{10pt}
\renewcommand{\arraystretch}{1.2}
\begin{tabular}{l|c|c}
\toprule
\textbf{Component} & \textbf{MP (nJ)} & \textbf{Ours (nJ)} \\
\midrule
Weight update           & $62.03$         & $12.82$ \\
Forward/backward pass   & $21.21$         & $N \cdot 9.44$ \\
\midrule
\textbf{Total}          & $83.2$          & $12.82 + N \cdot 9.44$ \\
\bottomrule
\end{tabular}
\vspace{-0.5em}
\caption{Estimated energy consumption per image for MP and our method based on \citep{gokmen2016acceleration,le2018mixed}.}
\label{tab:energy_breakdown}
\vspace{-2 em}
\end{table}
As the table shows, our method becomes less energy-efficient than MP when $N \geq 8$. In practice, however, much of the forward/backward overhead can be shared across tiles. For example, the PWM counter and comparator logic that generate input vectors into modulated pulses, and the ADCs that digitize accumulated outputs can all be amortized across multiple tiles rather than duplicated. Such shared-ADC designs have been demonstrated in recent analog accelerators \citep{xu2024reharvest,xu2024cascaded,song2024programming}. Only the per-tile operational transconductance amplifiers used for voltage regulation, along with the intrinsic device conduction energy, scale directly with $N$. This means that the effective energy growth with tile count is substantially slower than the conservative upper bound assumed in Table~\ref{tab:energy_breakdown}. Combined with the fact that our method achieves higher accuracy than TT baselines with as few as 3–4 tiles, these considerations indicate that our design remains more energy-efficient than MP even when more than eight tiles are employed, making it a practical and scalable solution.

\textbf{Area cost.} For our methods, since the overall architecture closely resembles that of TT-v1 except for the increased number of tiles, the corresponding estimates for area and execution time are derived based on models presented in~\citep{gokmen2020algorithm} and~\citep{gokmen2016acceleration}. We analyze the area overhead of our method using the RPU tile design methodology described in \citep{gokmen2016acceleration}, which assumes a realistic CMOS-compatible fabrication stack. Specifically, RPU arrays are implemented in the back-end-of-line (BEOL) region, with resistive memory devices placed between intermediate metal layers. Each crossbar array contains $D \times D$ devices, and the interconnect pitch (wire width plus spacing) is set to  400 nm , based on typical dimensions of intermediate BEOL levels. This yields a physical tile area of $((0.4D) \mu\text{m})^2 = (0.16D^2) \mu\text{m}^2$. Following \citep{gokmen2016acceleration}, we adopt a baseline configuration of $D = 4096$, which corresponds to a tile area of approximately  $2.68 \text{mm}^2$. In our experimental configurations (e.g., LeNet-5 and ResNet-18), which use smaller crossbar sizes, the tile area is scaled proportionally under the same pitch assumption; for instance, a 128 × 512 tile occupies  $0.0105\text{mm}^2$ . Based on the actual tile dimensions used in each layer of Tables \ref{tab:num_states_comparison} and \ref{tab:num_states_comparison_cifar}, we estimate the total analog area required by our method to be  $0.0128\text{mm}^2 $ for LeNet-5 and $ 1.69 \text{mm}^2 $ for ResNet-18.

To map logical weights to physical devices, each weight $W$ is represented as the difference between two conductance values: a main conductance $C_{\text{main}}$ and a reference $C_{\text{ref}}$, i.e., $W \propto C_{\text{main}} - C_{\text{ref}}$. This implies that both Analog SGD and MP require  2 analog tiles per layer . TT-v1 and TT-v2, which maintain both core and assistant matrices, require  4 analog tiles per layer . In our method, the number of physical analog tiles is  twice  the count reported in Tables \ref{tab:num_states_comparison} and \ref{tab:num_states_comparison_cifar}, due to our multi-tile residual structure. However, because these tiles can be vertically stacked using BEOL integration, the actual die area remains compact. Even without stacking, the total area overhead of our method remains within practical limits. For example, in the worst-case configuration using 10 analog tiles, our method incurs approximately  10×  the area of MP and  5×  that of TT-v2. Yet this level of overhead is still feasible: modern processors such as IBM Power8 CPUs \citep{stuecheli2013next} supports chip areas up to  $600 \text{mm}^2 $. Such systems are capable of integrating hundreds of analog tiles, indicating that our method remains scalable and realistic under practical hardware constraints.

In summary, our multi-tile framework provides clear advantages over TT-v2 and MP: it reduces digital storage by up to two orders of magnitude and achieves substantially lower runtime latency, benefiting from parallel analog updates even when more tiles are used. While area and energy scale with tile count, area overhead can be mitigated through BEOL stacking, and energy only exceeds MP under \emph{the most conservative estimates} when $N > 8$, since those estimates assume that all I/O, PWM logic, and ADC resources are replicated per tile rather than shared, meaning that in practical our design can tolerate substantially more tiles before its energy surpasses MP.
Importantly, our method consistently delivers higher accuracy than TT baselines with only 3–4 tiles, keeping energy well within practical limits.
These results establish our approach as a scalable and efficient solution for high-precision analog training.

\section{Supplement simulations}
\label{sec:Supplement_simulations}
 To further demonstrate the scalability, robustness, and practical relevance of our proposed method, we conduct a series of supplementary experiments across diverse settings, as supplement simulations to Section \ref{section: Numerical Simulations}. 
 We further explore the case of 80-state ReRAM devices with a larger portion of the model implemented in analog. This experiment examines whether our approach can still retain its advantages when the analog model is scaled up, providing insights into its applicability for future generations of higher-precision memristor devices.
 To better understand the algorithmic behavior, we perform an ablation study on the geometric scaling factor and its influence on training dynamics. Lastly, we extend our method to a Transformer-based natural language modeling task to demonstrate its applicability beyond vision workloads. Together, these results strengthen the case for our method as a scalable and general solution for analog training under various model architectures and hardware regimes.
\subsection{CIFAR-10 and CIFAR-100 experiments}
We first provide supplementary results on CIFAR-10 with ResNet-18 in Table \ref{tab:num_states_comparison_cifar}, which were omitted from the main text due to space constraints.
To further validate the effectiveness of our method on more challenging datasets, we conduct additional experiments on CIFAR-100 using ResNet-18, with devices limited to 4 conductance states. Given the increased complexity and number of classes in CIFAR-100, we extend the training schedule to 400 epochs to ensure sufficient convergence. 
\begin{table}[h]
\centering
\small
\vspace{-0.5em}
\setlength{\tabcolsep}{6pt}
\renewcommand{\arraystretch}{1.2}
\begin{tabular}{ccc|c|ccccc}
\toprule
\textbf{\# States} & \textbf{TT-v1} & \textbf{TT-v2} & \textbf{MP} & \textbf{Ours (4 tiles)} & \textbf{Ours (6 tiles)} & \textbf{Ours (8 tiles)} \\
\midrule
4  & 53.83{\tiny$\pm$0.14} & 87.89{\tiny$\pm$0.06} & 92.77{\tiny$\pm$0.05} & 87.35{\tiny$\pm$0.08} & 89.79{\tiny$\pm$0.14} & 90.45{\tiny$\pm$0.09} \\
10 & 84.18{\tiny$\pm$0.06} & 89.17{\tiny$\pm$0.11} & 93.45{\tiny$\pm$0.08} & 90.76{\tiny$\pm$0.06} & 91.07{\tiny$\pm$0.06} & 90.88{\tiny$\pm$0.05} \\
% 20 & 96.57{\tiny$\pm$0.10} & 94.36{\tiny$\pm$1.16} & 96.34{\tiny$\pm$0.04} & 83.98{\tiny$\pm$0.11} & 79.88{\tiny$\pm$1.26} & 80.80{\tiny$\pm$0.22} \\
\bottomrule
\end{tabular}
\vspace{-0.5 em}
\caption{
Test accuracy on CIFAR10 under 4 and 10 conductance states using analog ResNet-18. 
}
\label{tab:num_states_comparison_cifar}
\vspace{-1em}
\end{table}

\begin{table}[H]
\centering
\small
\begin{tabular}{l|ccc|ccc}
\toprule
\textbf{Model} & \textbf{TT-v1} & \textbf{TT-v2} & \textbf{MP} & \textbf{Ours (4 tiles)} & \textbf{Ours (6 tiles)} & \textbf{Ours (8 tiles)} \\
\midrule
ResNet-18 & 14.97{\tiny$\pm$1.93} & 27.91{\tiny$\pm$0.65} & 64.08{\tiny$\pm$0.44} & 58.36{\tiny$\pm$0.36} & 59.68{\tiny$\pm$0.33} & 60.62{\tiny$\pm$0.24} \\
\bottomrule
\end{tabular}
\vspace{-0.5em}
\caption{Test accuracy on CIFAR-100 using 4-state analog devices.}
\label{tab:cifar100_results}
\vspace{-2 em}
\end{table}

The results summarized in Table \ref{tab:cifar100_results} demonstrate that our method consistently outperforms TT-v1 and TT-v2 baselines and approaches the performance of the MP method as $6-8$ tiles are used.

\subsection{Experiments on 80-state ReRAM devices with increased analog deployment.}
\label{subsec:80-state ReRAM}
We further conduct experiments on 80-state ReRAM devices with increased analog deployment. Although 10-state devices are often used as an extreme but mature example (as most reported devices reach tens of states, see Table~\ref{table:device_granularity}), our algorithm also provides an effective solution for improving training accuracy on these relatively high-state devices. We observe that as the proportion of analog layers increases, for example, when further converting  \texttt{layer2} in ResNet-34 to analog, TT-v1 and TT-v2 degrades severely due to error accumulation across layers even with 80 states. In contrast, our method maintains higher accuracy, surpassing TT baselines with only 3–4 tiles and approaching MP performance with 8 tiles. We believe this accuracy collapse is a general challenge in scaling analog training to larger models, as errors from device non-idealities accumulate across more analog tiles, and our algorithm offers a practical solution without incurring extra latency, energy, or digital storage cost.
\begin{table}[H]
\centering
\small
\begin{tabular}{l|ccc|ccc}
\toprule
\textbf{Dataset} & \textbf{TT-v1} & \textbf{TT-v2} & \textbf{MP} & \textbf{Ours (3 tiles)} & \textbf{Ours (5 tiles)} & \textbf{Ours (7 tiles)} \\
\midrule
CIFAR-10 & 10.04{\tiny$\pm$0.04} & 75.65{\tiny$\pm$0.17} & 87.32{\tiny$\pm$0.10} & 82.59{\tiny$\pm$0.27} & 83.97{\tiny$\pm$0.20} & 84.96{\tiny$\pm$0.08} \\
CIFAR-100 & 1.27{\tiny$\pm$0.04} & 34.80{\tiny$\pm$0.16} & 58.06 {\tiny$\pm$0.07} & 42.12{\tiny$\pm$0.47} & 45.14{\tiny$\pm$0.13} & 50.82 {\tiny$\pm$0.14}\\
\bottomrule
\end{tabular}
\vspace{-0.5em}
\caption{Test accuracy on CIFAR-10 and CIFAR-100 using 80-state ReRAM devices with more layer in ResNet-34 converted to analog.}
\label{tab:80states}
\vspace{-2 em}
\end{table}

\subsection{Sensitivity to geometric scaling factor}
\begin{figure}[h]
    \centering
    \begin{subfigure}[t]{0.48\linewidth}
        \centering
        \includegraphics[width=\linewidth]{figures/gamma_ablation_4states.png}
        \label{fig:gamma_ablation_4states}
    \end{subfigure}
    \hfill
    \begin{subfigure}[t]{0.48\linewidth}
        \centering
        \includegraphics[width=\linewidth]{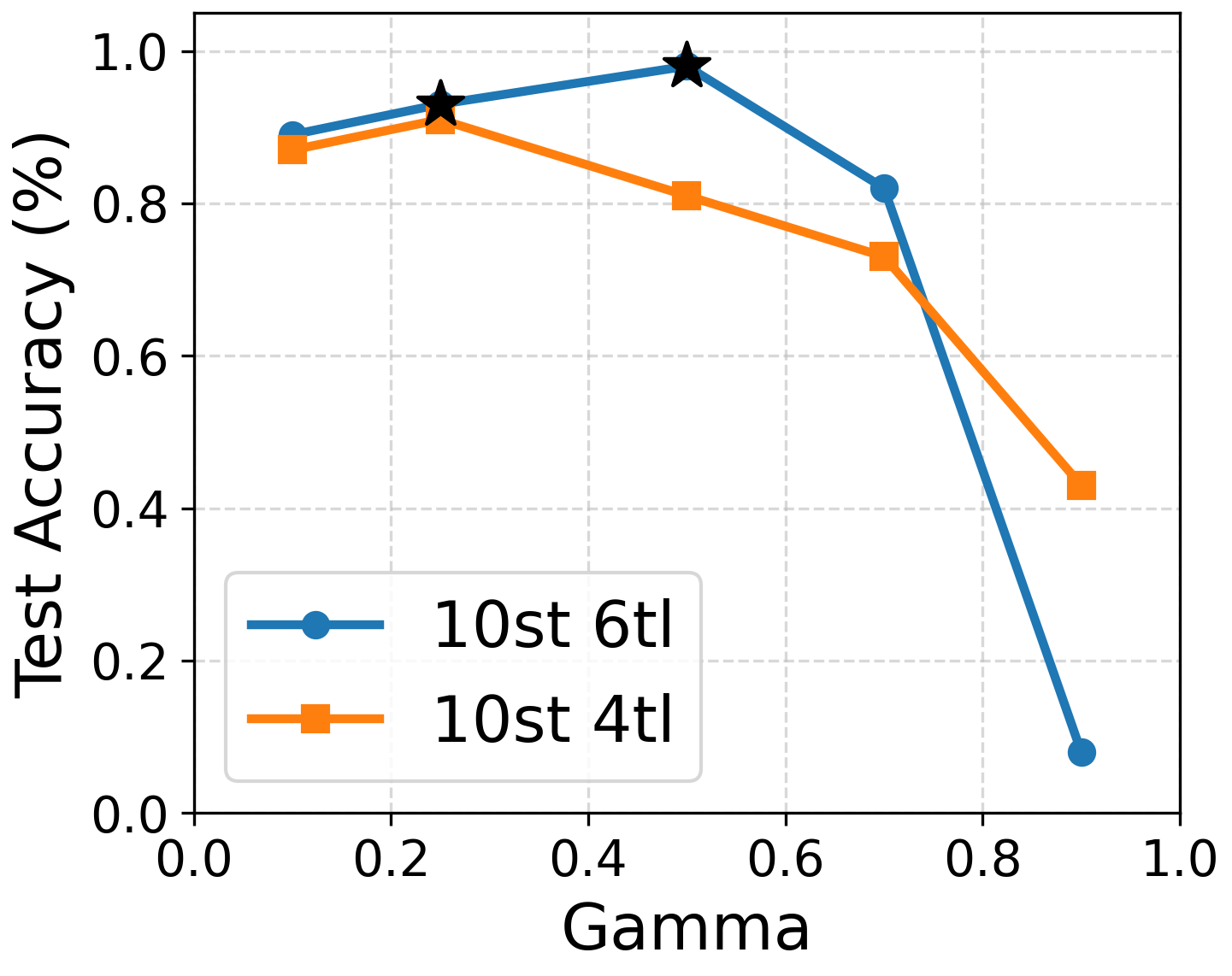}
        \label{fig:gamma_ablation_10states}
    \end{subfigure}
    \vspace{-1em}
    \caption{Gamma ablation study on LeNet-5 using different device states with 4-tile and 6-tile configurations.}
    \vspace{-1em}\label{fig:gamma_ablation_combined}
\end{figure}

The geometric scaling factor $\gamma$ plays a critical role in determining the effectiveness of residual representation across multiple analog tiles. Intuitively, if each tile has a dynamic range $[-\tau_{\text{max}}, \tau_{\text{max}}]$, then $\gamma$ should be chosen such that $
\gamma \cdot (2\tau_{\text{max}}) \approx \Delta w_{\min}
$
to ensure that the representable range of the next tile fully lies within the resolution of the previous tile. This design ensures that each additional tile effectively increases the representable precision by a factor equal to the number of conductance states. However, in practice, device non-idealities such as conductance saturation and asymmetric update dynamics reduce the usable dynamic range of each tile. To account for these effects, we heuristically choose $\gamma$ slightly larger than $\frac{\Delta w_{\min} }{(2\tau_{\text{max}})} =\frac{1}{n_{states}}$.
We show the corresponding ablation results in Figure~\ref{fig:gamma_ablation_combined}, where the peak 
$\gamma$ value is consistent with our hypothesis. These results support the need for careful selection of $\gamma$ based on device characteristics.

\subsection{Extension to Transformer-based NLP tasks}

To further demonstrate the scalability and general applicability of our method, we conducted an additional experiment on a natural language processing task using a GPT-2-style Transformer architecture. Specifically, we trained a 6-layer, 6-head, 768-dimensional model from scratch on the standard \textit{Shakespeare character-level language modeling} benchmark. The total number of trainable parameters is approximately 10.65M, and each iteration processes 16384 tokens.
We ran the training for 5000 iterations using 4-tile analog devices under \textbf{non-ideal I/O conditions} to simulate realistic hardware noise. The analog device used has 4 discrete states. We compare our method against representative analog training baselines, including TT-v1, TT-v2 and MP. 
\begin{table}[H]
\centering
\begin{tabular}{l|c|c|c|c}
\toprule
\textbf{Method} & \textbf{TT-v1} & \textbf{TT-v2} & \textbf{MP} & \textbf{Ours (4 tiles)} \\
\midrule
 Loss & 3.0336 & 2.6137 & 2.7213 & \textbf{ 2.5971} \\
\bottomrule
\end{tabular}
\vspace{-0.5em}
\caption{Validation loss on 6-layer GPT-style model.}
\label{tab:char_modeling}
\vspace{-2em}
\end{table}
As shown in Table~\ref{tab:char_modeling}, our method achieves comparable final validation loss, demonstrating both accuracy and robustness on this standard NLP benchmark.
These results confirm that our method is not limited to vision tasks, but also scales effectively to Transformer-based sequence modeling, maintaining accuracy and resilience under analog non-idealities.

\section{Simulation Details}
\label{sec:Simulation_Details}
This section provides details about the experiments in Section \ref{section: Numerical Simulations}. The analog training algorithms, including Mixed Precision and Tiki-Taka, are provided by the open-source simulation toolkit AIHWKIT, which has an Apache-2.0 license; see \url{https://github.com/IBM/aihwkit}. We use the Softbound device provided by AIHWKIT to simulate the asymmetric linear device. Digital algorithms (including SGD) and datasets (including MNIST and CIFAR10) used in this paper are provided by PyTorch, which has a BSD license; see \url{https://github.com/pytorch/pytorch}. Our implementation builds upon the \texttt{TT-v1} preset in \texttt{AIHWKIT v0.9.2}, with modifications to the gradient routing and transfer mechanisms to support our proposed Residual Learning scheme. 
We conduct our experiments on an NVIDIA RTX 3090 GPU, which has 24GB memory. The simulation time ranges from one to three hours, depending on the model size, dataset, and the number of training epochs. The code is available at \url{https://anonymous.4open.science/r/Code-CC44/}.

\subsection{MNIST and Fashion-MNIST training with analog LeNet-5}
\textbf{Data and preprocessing.} The MNIST dataset is used with standard normalization and no data augmentation. The training and testing sets use the default PyTorch torchvision splits (60{,}000 training and 10{,}000 testing samples). In our experiments, we utilize the full training set with a batch size of 8.
The Fashion-MNIST dataset is used with the default PyTorch \texttt{torchvision} splits, consisting of 60{,}000 training and 10{,}000 testing samples. No additional data augmentation is applied. A simple normalization transform is used through \texttt{ToTensor()}, and the full training set is utilized with a batch size of 16.

 \textbf{Model architecture.} 
We adopt a LeNet-5 model in which all convolutional and fully-connected layers are implemented using AIHWKit's analog modules (\texttt{AnalogConv2d}, \texttt{AnalogLinear}). Digital non-linear operations, such as Tanh activations and MaxPooling, are interleaved and remain executed in the digital domain.

\noindent \textbf{Optimizer and learning rate.} 
For MNIST, we employ the \texttt{AnalogSGD} optimizer with an initial global learning rate of 0.05 applied uniformly to all trainable parameters. 
For Fashion-MNIST, we use the \texttt{AnalogSGD} optimizer with an initial global learning rate of 0.2 for our method, and 0.1 for TT-v1, TT-v2, and MP.
A learning rate scheduler based on \texttt{LambdaLR} decays this global rate by a factor of 0.5 every 30 epochs.
In analog layers, for our algorithm, each tile $W^{(n)}$ is assigned a fixed internal learning rate \texttt{transfer\_lr\_vec[n]}$ = 0.1\cdot 1.2^n$. These internal learning rates remain constant throughout the training and are not affected by the global schedule, as \texttt{scale\_transfer\_lr=False}. For TT-v1, we set the auxiliary tile learning rate \texttt{fast\_lr} $\alpha$ to 0.01 and the transfer learning rate \texttt{transfer\_lr} to 0.1 on both datasets. For TT-v2  we set $\alpha=0.1$ and $\beta=1$ for MNIST and $\alpha=0.05$ for Fashion-MNIST. Additionally, we set \texttt{scale\_transfer\_lr=True} for TT-v1, TT-v2 and MP as default. 

\textbf{Tile parameter configuration.}
We configure the behavior of each analog tile through the following parameter vectors, all generated dynamically as a function of the total number of tiles \texttt{num\_tile}:
\begin{itemize}
    \item For MNIST:
    \begin{itemize}
        \item \texttt{transfer\_every\_vec = [2 * (5\^{}n) for n in range(num\_tile)]}
        \item \texttt{gamma\_vec = [0.5\^{}(num\_tile - 1 - i) for i in range(num\_tile)]}
    \end{itemize}
    \item For Fashion-MNIST:
    \begin{itemize}
        \item \texttt{transfer\_every\_vec = [2 * (5\^{}n) for n in range(num\_tile)]}
        \item \texttt{gamma\_vec = [0.2\^{}(num\_tile - 1 - i) for i in range(num\_tile)]}
    \end{itemize}
\end{itemize}
These vectors control the per-tile transfer schedule, readout scaling, and learning rate, respectively. The number of tiles \texttt{num\_tile} is a configurable parameter that we vary in experiments (e.g., Table~\ref{tab:num_states_comparison}).
It is worth noting that in the actual implementation, tile index \texttt{0} serves as the fixed gradient accumulation tile and plays the role of \( W^{(N)} \) in the main text. The remaining tiles at indices \( 1, 2, \ldots, \texttt{num\_tile} - 1 \) correspond to \( W^{(N-1)}, W^{(N-2)}, \ldots, W^{(0)} \), respectively. 
While this index order is opposite to the mathematical notation used in the main text, the transfer logic and learning behavior are equivalent. The index inversion only affects naming, not the functional correctness or conclusions of the training algorithm.

\textbf{I/O configuration.} As acknowledged in Section~\ref{sec:limitations}, this work assumes idealized I/O settings throughout all experiments. I/O behavior is configured as nearly ideal, with:
\begin{itemize}
    \item \textbf{Forward path:} The input vector \( x \) is injected into the crossbar without finite resolution quantization, amplitude clipping, or additive noise. The resulting output current is integrated ideally, bypassing any ADC or nonlinear feedback models.    
    \item \textbf{Backward path:} The backpropagated vector \( \delta \) is encoded and applied in a similarly ideal manner, ignoring input resolution limits, digital-to-analog conversion noise, or output quantization during the gradient computation.   
    \item \textbf{Transfer path:} When using compound devices such as \texttt{TransferCompound}, the internal transfer of weights between tiles (e.g., during warm start or periodic updates) also involves analog readout and write operations. Setting \texttt{device.transfer\_forward.is\_perfect = True} disables all I/O imperfections during this internal read phase, ensuring clean accumulation and precise programming of weights across tiles.
\end{itemize}

By default (\texttt{inp\_noise=0.0}, \texttt{out\_noise=0.0}). Unless overridden, the defaults are \\ \texttt{io\_inp\_bits}=7, \texttt{io\_out\_bits}=9.

\textbf{Tile switching schedule.} To avoid early saturation of coarse tiles, the training monitors convergence plateaus via loss history.  Early epochs use an aggressive trigger if training loss does not drop sufficiently between epochs.
After four tile switches, a smoother criterion is used based on the recent 5-step moving window.
Upon plateau detection, a C++ flag \texttt{trigger\_tile\_switch} is activated via Python binding for each tile.

\textbf{Training and Evaluation.} The network is trained for up to 100 epochs. Classification loss is computed using \texttt{nn.NLLLoss()} applied to the log-softmax outputs. Evaluation is performed after each epoch using classification accuracy on the full MNIST test set.

\subsection{CIFAR-10 and CIFAR-100 training with ResNet}
\label{sec:Simulation_Details_CIFAR}
 \textbf{Dataset and augmentation.} The CIFAR-10 dataset is used for all ResNet  experiments. Following the default splits provided by \texttt{torchvision}, the dataset consists of 50{,}000 training samples and 10{,}000 test samples, selected via the \texttt{train=True/False} flag. For additional experiments on more fine-grained recognition, we also use the CIFAR-100 dataset, which has the same image size and train/test splits  but with 100 object categories grouped into 20 superclasses. We utilize the entire training set and set the batch size to 128.
All images are normalized to zero mean and unit variance per channel. During training, strong data augmentation is applied, including random cropping, horizontal flipping, Cutout regularization and AutoAugment using 25 CIFAR-10-specific sub-policies. No augmentation is applied to the test set beyond normalization.

\noindent \textbf{Model architecture.} In different experiments, we use ResNet-18 and ResNet-34 models, where \texttt{layer3}, \texttt{layer4}, and the final classifier are mapped to analog tiles, while the remaining layers remain digital. To further demonstrate the scalability of our method (see Section \ref{subsec:80-state ReRAM}), we also conduct experiments on a ResNet-34 variant in which \texttt{layer2} is additionally mapped to analog tiles. Batch normalization and residual shortcuts are preserved unless explicitly disabled.

\noindent \textbf{Tile parameter configuration.}
We configure the behavior of each analog tile through the following parameter vectors, all generated dynamically as a function of the total number of tiles \texttt{num\_tile}:
\begin{itemize}
    \item For 4-state experiments:
    \begin{itemize}
        \item \texttt{transfer\_every\_vec = [3 * (2\^{}n) for n in range(num\_tile)]}
        \item \texttt{gamma\_vec = [0.5\^{}(num\_tile - 1 - i) for i in range(num\_tile)]}
    \end{itemize}
    \item For 16-state experiments:
    \begin{itemize}
        \item \texttt{transfer\_every\_vec = [3 * (2\^{}n) for n in range(num\_tile)]}
        \item \texttt{gamma\_vec = [0.1\^{}(num\_tile - 1 - i) for i in range(num\_tile)]}
    \end{itemize}
    % \item \texttt{transfer\_lr\_vec = [0.1 if n == 0 else 0.1 * 1.2\^{}n for n in range(num\_tile - 1, -1, -1)]}
\end{itemize}

 \textbf{I/O configuration.} The I/O configuration is the same as in the MNIST experiments.

\textbf{Optimizer and Scheduler.} All analog parameters are trained using \texttt{AnalogSGD} with an initial learning rate of 0.1.  A \texttt{StepLR} scheduler reduces the learning rate by a factor of 0.1 every 100 epochs. In analog layers, for our algorithm, each tile $W^{(n)}$ is assigned a fixed internal learning rate \texttt{transfer\_lr\_vec[n]}$ = 0.3\cdot 1.2^n$. These internal learning rates remain constant throughout the training and are not affected by the global schedule, as \texttt{scale\_transfer\_lr=False}. For both TT-v1 and TT-v2, we set the auxiliary tile learning rate \texttt{fast\_lr} $\alpha$ to 0.1 and the transfer learning rate \texttt{transfer\_lr} to 1.  Additionally, we set \texttt{scale\_transfer\_lr=True} for TT-v1, TT-v2 and MP as default.

\textbf{Training and Evaluation.} The network is trained for 200 epochs for CIFAR-10 and 400 epochs for CIFAR-100 with a batch size of 128. Classification loss is computed using label-smoothed cross-entropy, implemented via \texttt{LabelSmoothingLoss} with a smoothing factor of 0.1.

\textbf{Tile switching.} The tile switching strategy is the same as in the MNIST experiments.

\subsection{Least square problem}
 
\noindent \textbf{Model architecture.}
We use a scalar analog layer $\mathbb{R}^{1}\!\to\!\mathbb{R}^{1}$: \texttt{AnalogLinear(1,1)}.

\noindent \textbf{Tile parameter configuration.}
We instantiate a \texttt{TransferCompound} device with \texttt{num\_tile} identical unit cells, each a \texttt{SoftBoundsDevice} with $w_{\min}=-1$, $w_{\max}=1$, and $\Delta w_{\min}=0.5$. Column transfers are enabled and multi-sample updates are treated as mini-batch units (\texttt{transfer\_columns=True}, \texttt{units\_in\_mbatch=True}). The two key parameter vectors are generated from \texttt{num\_tile}: \begin{itemize}
    \item \texttt{transfer\_every\_vec = [2 * (2\textasciicircum{}n) for n in range(num\_tile)]}
    \item \texttt{gamma\_vec = [0.1\textasciicircum{}(num\_tile - 1 - i) for i in range(num\_tile)]}
\end{itemize}
Tile-internal transfer learning rate is fixed to \texttt{transfer\_lr}=0.01 with \texttt{scale\_transfer\_lr=False}. 
We also set the pulse update scheme as \texttt{update\_bl\_management=False}, \texttt{update\_management=False}, as well as weight scaling scheme \texttt{digital\_bias=False}, \texttt{learn\_out\_scaling=False}, \texttt{weight\_scaling\_columnwise=False}, and \texttt{weight\_scaling\_omega}=0.0.

\noindent \textbf{I/O configuration.}
The I/O configuration is the same as in the MNIST experiments.

\noindent \textbf{Target generation.}
Batch size defaults to \texttt{batch\_size}=1. The regression target $b \in [-1,1]$ is sampled from a uniform 16-bit quantizer:
\[
b = -1 + k \cdot \frac{2}{2^{16}-1}, \qquad k \sim \text{Uniform}\{0,\ldots,2^{16}-1\}.
\]

\noindent \textbf{Optimizer.}
All analog parameters are trained using \texttt{AnalogSGD} with an initial learning rate of 0.001. In analog layers, for our algorithm, each tile $W^{(n)}$ is assigned a fixed internal learning rate \texttt{transfer\_lr}$ = 0.01$. These internal learning rates remain constant throughout the training and are not affected by the global schedule, as \texttt{scale\_transfer\_lr=False}. We set the auxiliary tile learning rate \texttt{fast\_lr} $\alpha$ to 0.01.

\end{document}